\RequirePackage{fix-cm}
\documentclass[smallextended]{svjour3}       
\smartqed  
\usepackage{graphicx}
\usepackage{subcaption}
\usepackage{url}
\usepackage{booktabs}
\usepackage{lipsum, color}
\usepackage{comment}
\usepackage{all_math_commands}
\usepackage{soul}

\usepackage{amsmath,amssymb,amsfonts}
\newtheorem{assumption}[theorem]{Assumption}%

\DeclareMathOperator{\mean}{mean}
\DeclareMathOperator{\var}{var}

\DeclareMathOperator{\ConvOneD}{Conv1d}
\DeclareMathOperator{\EqConv}{EqConv}

\newcommand{\TimeStep}{{\Delta t}}

\newcommand{\RadiusMin}{{R_{\min}}}

\begin{document}
\newcommand{\gExpertName}[1][\FlockAgentIndex]{{\ensuremath{\va_{#1}}}}
\newcommand{\gExpertInput}[1][]{{\FlockPosMatrix#1, \FlockVelMatrix#1}}
\newcommand{\gGeneralizationGap}{\cR_{\GeneralizationSampleName,\LossFunctionName}}
\newcommand{\gConvInput}[1][\PupilLayerIndex]{{\PupilHistoryMatrix_{\FlockAgentIndex,#1}}}
\newcommand{\gConvInputAsMLP}[1][\PupilLayerIndex]{{\PupilHistoryMatrix_{\FlockAgentIndex,#1}^{\mathrm{MLP}}}}
\newcommand{\gConvInputWithBias}[1][\PupilLayerIndex]{{\widetilde{\gConvInput[#1]}}}
\newcommand{\gDAggerDatasetDatum}[1][]{{\gExpertInput[#1], \set{\PupilHistoryMatrix_\FlockAgentIndex#1}_{\FlockAgentIndex = 1}^\FlockAgentCount}}
\newcommand{\gConvChannelsIn}[1][\PupilLayerIndex]{{C_{#1}}}
\newcommand{\gConvFeaturesIn}[1][\PupilLayerIndex]{{F_{#1}}}

\title{
Learning Decentralized Swarms Using\\Rotation Equivariant Graph Neural Networks
}


\author{
  Taos Transue
  \and
  Bao Wang
}


\institute{
    Taos Transue \and Bao Wang
    \at
    Department of Mathematics, University of Utah, 155 South 1400 East, Salt Lake City, Utah 84112, USA.
    \and
    Bao Wang
    \at
    Scientific Computing and Imaging Institute, University of Utah, 72 S Central Campus Drive, Salt Lake City, Utah 84112, USA.\\
    Corresponding author \email{taos.j.transue@gmail.com} \\
    \email{wangbaonj@gmail.com}
}

\date{Received: date / Accepted: date}

\maketitle

\begin{abstract}
The orchestration of agents to optimize a collective objective without centralized control is challenging yet crucial for applications such as controlling autonomous fleets, and surveillance and reconnaissance using sensor networks. Decentralized controller design has been inspired by self-organization found in nature, with a prominent source of inspiration being flocking; however, decentralized controllers struggle to maintain flock cohesion. The graph neural network (GNN) architecture has emerged as an indispensable machine learning tool for developing decentralized controllers capable of maintaining flock cohesion, but they fail to exploit the symmetries present in flocking dynamics, hindering their generalizability. We enforce rotation equivariance and translation invariance symmetries in decentralized flocking GNN controllers and achieve comparable flocking control with 70\% less training data and 75\% fewer trainable weights than existing GNN controllers without these symmetries enforced. We also show that our symmetry-aware controller generalizes better than existing GNN controllers. Code and animations are available at \url{github.com/Utah-Math-Data-Science/Equivariant-Decentralized-Controllers}. 
\end{abstract}

\section{Introduction}
Orchestrating a system of agents to optimize a collective objective is crucial in autonomous vehicle control \cite{perea2009extension}, surveillance and reconnaissance with sensor networks \cite{akyildiz2002survey}, and beyond \cite{hua2024towards}. Designing controllers for these systems is often inspired by self-organization in nature. A prominent source of inspiration is flocking: myxobacteria travel in clusters to increase the concentration of an enzyme they use to stun their prey, making predation more effective for each individual \cite{camazineSelforganizationBiologicalSystems2003}; pelicans follow a leader pelican in a vee formation, taking advantage of special aerodynamics to reduce the energy expenditure of flight \cite{weimerskirchEnergySavingFlight2001}; and schools of fish coordinate their movement in several formations, minimizing each individual's risk of predation \cite{parrishSelfOrganizedFishSchools2002}.

When designing a controller to optimize a collective objective, there are three pertinent questions: (1) What objective should an agent optimize individually? (2) What information about other agents or the environment does an agent need to act? (3) How do the actions of each agent culminate in optimizing the collective objective?
The first two questions were studied for flocking controllers in \cite{reynoldsFlocksHerdsSchools1987}, where Reynolds sought to reproduce bird flocking and created one of the earliest simulations of flock-like motion. His simulation had each agent (so-called ``bird-oid'' or ``boid'') balance the following three objectives:
\begin{itemize}
    \item \textbf{Alignment:} Each agent should steer toward its neighbors' average direction.
    \item \textbf{Cohesion:} Each agent should move toward the average position of its neighbors.
    \item \textbf{Separation:} Each agent should move away from its neighbors if they are too close.
\end{itemize}
To optimize these objectives, each agent used the position and velocity of all other agents. With qualitative answers to questions (1) and (2), Reynolds programmed agents that ``participate in an acceptable approximation of flock-like motion.''
Since flock-like motion was not yet quantified, question (3) remained unanswered.
Later, researchers quantified flock-like motion with
\textit{asymptotic flocking} and \textit{separation} (Definitions~\ref{def:asymptotic-flocking} and \ref{def:agent-separation}).
Now, question (3) is answered by arguing that a given controller induces asymptotic flocking.

Notable flocking controllers include those designed in \cite{tannerStableFlockingMobile2003a,cuckerEmergentBehaviorFlocks2007}.
In \cite{tannerStableFlockingMobile2003a}, Tanner et al. set the agents' accelerations equal to a term that encourages alignment plus the negative gradient of a potential function for cohesion and separation. If the graph representing what agents can exchange information is fully connected, asymptotic flocking and separation are guaranteed. 
In \cite{cuckerEmergentBehaviorFlocks2007}, Cucker and Smale define the agents' accelerations in a way that aligns their velocities while preserving the momentum of the flock. They achieve this with an ``influence function,'' which gauges an agent's influence on another as a function of their distance. The benefit of momentum conservation is that the limiting velocity of the flock can be computed from the flock's initial conditions. Cucker and Smale's controller guarantees asymptotic flocking, and guarantees separation with singular influence functions \cite{cuckerAvoidingCollisionsFlocks2010,minakowskiSingularCuckerSmaleDynamics2019}.
%
%
Further developments in flocking controllers can be viewed as adding constraints when optimizing a collective flocking objective. 
Developments that add constraints to Tanner et al.'s controller include leader following -- requiring the flock to follow a leader agent \cite{guLeaderFollowerFlocking2009}, and segregation -- flocking while having agents cluster into preassigned groups \cite{santosSegregationMultipleHeterogeneous2014}. Developments that add constraints to the Cucker-Smale controller include encouraging the agents to be a fixed distance $R>0$ from each other using a PID controller \cite{parkCuckerSmaleFlockingInterParticle2010}, and pattern formation via control terms added to the agents' acceleration \cite{ohFlockingBehaviorStochastic2021,choiCollisionlessSingularCuckerSmale2019,choiControlledPatternFormation2022}.

All these controllers require agents to exchange information with all others agents, making them {\bf centralized controllers}. 
As the number of agents 
grows, an increasingly prominent alternative is {\bf decentralized controllers} where agent communication is represented by a sparse graph.
In a flock of $\FlockAgentCount$ agents, a centralized controller requires each agent to 
communicate with 
$\FlockAgentCount-1$ agents, resulting in quadratic ($\cO\pn{\FlockAgentCount^2}$) communication overhead growth. 
Moreover, a centralized controller imposes a bound on $\FlockAgentCount$ when separation is required, and each agent has a maximum communication distance. Maximum communication distance constraints arise in wireless communications where the power required to send information over a distance $r$ is proportional to $r^p$ for some $p \in \hlpn{2, 4}$ \cite{akyildiz2002survey}. A maximum communication distance $\FlockAgentCommRadius$ and a minimum separation of $\FlockAgentMinDist$ bound $\FlockAgentCount$ because only a finite number of disjoint balls with diameter $2\FlockAgentMinDist$ can be packed into a ball whose diameter is the maximum diameter of the flock ($2\FlockAgentCommRadius$).
{\bf Decentralization is crucial for building 
scalable flocking controllers}, but compared to centralization, decentralization requires additional assumptions to guarantee asymptotic flocking and may not be compatible with other guarantees (e.g., separation).

Examples of decentralized controllers include those proposed by Tanner et al. \cite{tannerStableFlockingMobile2003a,tannerStableFlockingMobile2003}.
The controller proposed in \cite{tannerStableFlockingMobile2003a} uses a fixed connected communication graph to guarantee asymptotic flocking,\footnote{In practice, the communication graph can be constucted using the idea of virtual nodes; see e.g., \cite{9850408,doi:10.1137/21M1465081}.} but it does not guarantee separation because nonadjacent agents cannot communicate. Though flock cohesion is guaranteed, adjacent agents may need to communicate over arbitrarily large distances while the controller is active. Instead of a fixed agent neighborhood, the neighborhood of an agent in \cite{tannerStableFlockingMobile2003} changes over time by only containing other agents currently within the agent's communication radius $\FlockAgentCommRadius$. In this case, the authors guarantee asymptotic flocking and separation assuming that the communication graph is connected for all time; \figref{fig:tanner:decentralized:dynamic:failure} shows that this assumption is invalid for some initial conditions. In \cite{guLeaderFollowerFlocking2009}, the authors propose a decentralized flocking controller with time-dependent agent neighbors and a leader following constraint, but they only test the controller with a fully connected communication graph.
Without a fixed communication graph, decentralized flocking controllers struggle to maintain communication graph connectivity and flock cohesion -- a requirement for effectively optimizing the collective objective.

As suggested by question (2), a plausible approach to improve decentralized controllers is providing more information to each agent.
All the controllers we described, centralized or decentralized, only use information from the current time. 
Tolstaya et al. \cite{tolstayaLearningDecentralizedControllers2017b} apply machine learning (ML) to leverage current and past information.
\figref{fig:tanner:decentralized:dynamic:failure} shows that their ML controller can achieve asymptotic flocking for flock initial conditions where the controller in \cite{tannerStableFlockingMobile2003} fails (see \figref{fig:tanner:decentralized:dynamic:failure}).
The ML controller works by having each agent retain a summary of the information it receives during the current time step and, during the next time step, send that summary to its neighbors.
In addition, each agent retains a summary of the summaries it receives to send to its neighbors later, and this recursion continues up to $\PupilHistoryCount - 1$ time steps in the past. This process gives each agent access to information $\PupilHistoryCount$ hops away in the communication graph.
In training, the ML controller learns to utilize patterns in how the information exchanged in the network changes over time.
Though Tolstaya et al. \cite{tolstayaLearningDecentralizedControllers2017b} have found a way to provide additional information to each agent, their ML controller processes the information suboptimally. The flocking controllers in \cite{tannerStableFlockingMobile2003a,cuckerEmergentBehaviorFlocks2007} exhibit crucial symmetries, e.g., rotation equivariance, that have been ignored. 
Symmetry-aware ML models have demonstrated remarkable advantages in physical congruence and data efficiency 
\cite{cohen2017steerable,fuchs2020se,wang2024rethinking,baker2024an}. 
We aim to bridge this gap in this paper. 

\section{Our contributions}

We improve decentralized ML flocking controllers by enforcing rotation equivariance into time-delayed aggregation GNNs (\PupilTDAGNN{}s) \cite{tolstayaLearningDecentralizedControllers2017b}, restoring the symmetries 
of non-ML flocking controllers. 
We summarize our key contributions as follows:
\begin{itemize}
\item We present a simple yet efficient rotation equivariant convolutional neural network (CNN) and integrate it into \PupilTDAGNN{} for learning decentralized flocking.

\item We justify the theoretical advantages of rotation equivariance in learning decentralized flocking by demonstrating better generalization compared to non-equivariant controllers.

\item We demonstrate the advantages of our new ML controller in decentralized flocking, leader following, and obstacle avoidance. 
\end{itemize}

\subsection{Organization}

We organize this paper as follows. In \secref{sec:bg}, we review some related works on flocking controllers. In \secref{sec:model-proposed}, we present our rotation equivariant GNN for decentralized flocking. In \secref{sec:main:theory} and \secref{sec:experiments}, we analyze the generalization advantages of the rotation equivariant GNN for learning decentralized flocking and verify its effectiveness. Technical proofs and additional details are provided in the appendix. 

\section{Background}\label{sec:bg}
\subsection{Flock representation and definitions}
\label{sec:flocking-control-description}
Let $\FlockPosMatrix=[\FlockAgentPos_1,\;\ldots,\;\FlockAgentPos_\FlockAgentCount]\in\bR^{2\times \FlockAgentCount}$ be the position vectors of \FlockAgentCount{} agents, $\FlockVelMatrix$ be their velocities, and $\FlockAgentRelativePos = \FlockAgentPos_\FlockAgentIndex - \FlockAgentPos_\FlockAgentIndexTwo$ be the relative position. 
We formulate the alignment and cohesion rules for asymptotic flocking as follows:

\begin{definition}[Asymptotic flocking \cite{choiEmergentDynamicsCuckerSmale2016}]
\label{def:asymptotic-flocking}
A flock of \FlockAgentCount{} agents flocks asymptotically if and only if the following two conditions are satisfied: (1)  {\bf Alignment:} $\lim\limits_{t \to \infty} \max\limits_{\FlockAgentIndex,\FlockAgentIndexTwo} \norm{\FlockAgentRelativeVel\pn{t}} = 0$, and (2) {\bf Cohesion:} $\sup\limits_{0 \leq t < \infty,\FlockAgentIndex,\FlockAgentIndexTwo} \norm{\FlockAgentRelativePos\pn{t}} < \infty$.
\end{definition}

Asymptotic flocking is usually the primary objective of flocking controllers.
Controllers can also make guarantees concerning how their objectives are met, such as agent separation, which is defined as follows:

\begin{definition}[Agent $\FlockAgentMinDist$-separation]\label{def:agent-separation}
Let $\FlockAgentMinDist \geq 0$. Agents are separated if for all time $t$, $\norm{\FlockAgentRelativePos\pn{t}} \geq \FlockAgentMinDist$.
\end{definition}

Other objectives or guarantees beyond those mentioned include segregation \cite{santosSegregationMultipleHeterogeneous2014}, where agents cluster into preassigned groups, and pattern formation \cite{ohFlockingBehaviorStochastic2021,choiCollisionlessSingularCuckerSmale2019,choiControlledPatternFormation2022}.

The acceleration of the $\FlockAgentIndex$-th agent is $\FlockAgentAccel_\FlockAgentIndex = \gExpertName\pn{\gExpertInput[\pn{t}]}$ where $\gExpertName$ is the flocking controller.
The design of $\gExpertName$ ensures the flocking controller's objectives are met while maintaining its guarantees.
The two main classes of flocking controllers -- centralized and decentralized -- are separated by whether they use the entire flock state $\pn{\gExpertInput}$ or only a subset of the columns of $\FlockPosMatrix$ and $\FlockVelMatrix$ based on what agents can communicate.

\subsection{Centralized flocking}\label{sec:bg:centralized:math}
We recap on the centralized flocking controller from \cite{tannerStableFlockingMobile2003a}, which defines an agent's acceleration as
\begin{equation}\label{eq:tanner:centralized}
\begin{aligned}
\FlockAgentAccel_\FlockAgentIndex = \gExpertName\pn{\gExpertInput} := -\sum_{\FlockAgentIndexTwo = 1}^\FlockAgentCount \FlockAgentRelativeVel - \sum_{\FlockAgentIndexTwo = 1}^\FlockAgentCount \nabla \TannerEtAlPotential\pn{\norm{\FlockAgentRelativePos}},
\end{aligned}
\end{equation}
where $\TannerEtAlPotential:\pn{0, \infty} \to \bR$ is a potential function such that $\lim_{r \to \infty} \TannerEtAlPotential\pn{r} = \infty$ and $r_* > 0$ is $\TannerEtAlPotential$'s unique minimizer representing the desired distance between two agents.
The first summation aligns the agents' velocities, and the second summation uses $\TannerEtAlPotential$ to control flock cohesion and agent separation.
The potential function for $\TannerEtAlPotential$ used in \cite{tannerStableFlockingMobile2003} is
$\TannerEtAlPotential\pn{r} = 1/r^2 + \ln\pn{r^2}$,
which encourages agents to be a distance $r_* = 1$  from each other for cohesion and is singular at $r = 0$ for separation. The function $\TannerEtAlPotential$ can be changed to induce a variety of flock behaviors. For example, organizing leader-following \cite{guLeaderFollowerFlocking2009} and segregating the agents into clusters \cite{santosSegregationMultipleHeterogeneous2014}.

\subsection{Decentralized flocking}\label{sec:bg:decentralized:math}
We briefly review the decentralized flocking controllers 
from \cite{tannerStableFlockingMobile2003a,tannerStableFlockingMobile2003}.
The first decentralized controller \cite{tannerStableFlockingMobile2003a} operates on a fixed sparse communication graph and defines an agent's acceleration as
$$
\begin{aligned}
\FlockAgentAccel_\FlockAgentIndex = \gExpertName\pn{\gExpertInput} := -\sum_{\FlockAgentIndexTwo \in \FlockAgentNeighborhood_\FlockAgentIndex} \FlockAgentRelativeVel - \sum_{\FlockAgentIndexTwo \in \FlockAgentNeighborhood_\FlockAgentIndex} \nabla \TannerEtAlPotential\pn{\norm{\FlockAgentRelativePos}},
\end{aligned}
$$
where $\FlockAgentNeighborhood_\FlockAgentIndex$ is the neighborhood of agent $\FlockAgentIndex$. This controller guarantees asymptotic flocking but does not guarantee separation because nonadjacent agents cannot communicate to avoid collision. Another drawback of this controller is that any adjacent agents must be able to communicate regardless of their spatial distance.

The second decentralized controller \cite{tannerStableFlockingMobile2003} defines the neighborhood of an agent to only include agents within 
radius $\FlockAgentCommRadius$, i.e., $\FlockAgentNeighborhood_\FlockAgentIndex\pn{t} = \set{\FlockAgentIndexTwo : \FlockAgentIndexTwo \neq \FlockAgentIndex,\ \norm{\FlockAgentRelativePos} \leq \FlockAgentCommRadius}$. The agent's acceleration is defined as
$$
\begin{aligned}
\FlockAgentAccel_\FlockAgentIndex = \gExpertName\pn{\gExpertInput} := -\sum_{\FlockAgentIndexTwo \in \FlockAgentNeighborhood_\FlockAgentIndex\pn{t}} \FlockAgentRelativeVel - \sum_{\FlockAgentIndexTwo \in \FlockAgentNeighborhood_\FlockAgentIndex\pn{t}} \nabla \TannerEtAlPotential\pn{\norm{\FlockAgentRelativePos}}.
\end{aligned}
$$
This controller guarantees asymptotic flocking and separation.
However, it assumes the communication graph is connected for all time, but \Figref{fig:tanner:decentralized:dynamic:failure}(a,b,c) show that this assumption may be invalid.

\begin{figure}[!ht]
\centering
\begin{subfigure}[t]{.22\textwidth}
\centering
\includegraphics[width=\textwidth]{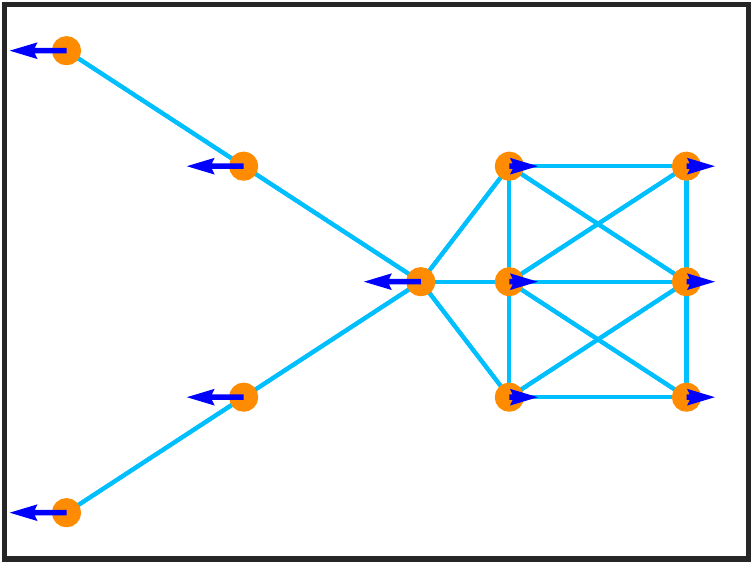}
\caption{\footnotesize $t_0$}
\label{fig:tanner:decentralized:dynamic:failure:time0}
\end{subfigure}
\hspace{1em}
\begin{subfigure}[t]{.22\textwidth}
\centering
\includegraphics[width=\textwidth]{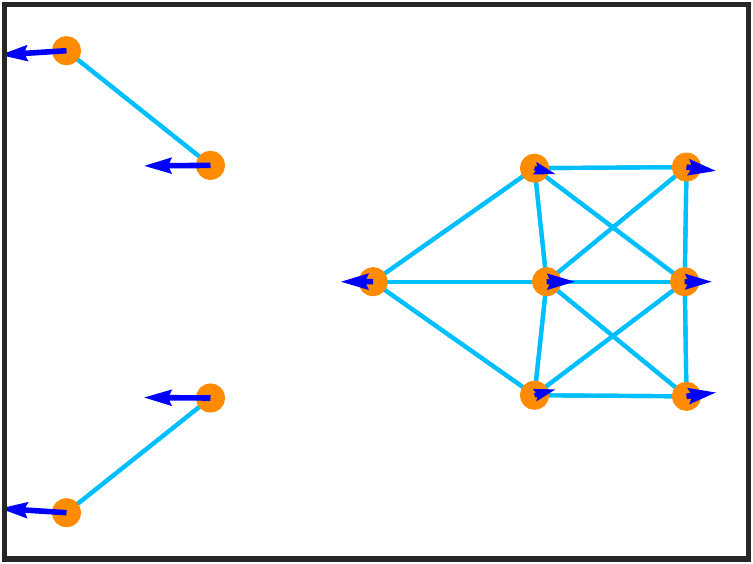}
\caption{
\footnotesize $t_{27}$
}
\label{fig:tanner:decentralized:dynamic:failure:time27}
\end{subfigure}
\hspace{1em}
\begin{subfigure}[t]{.22\textwidth}
\centering
\includegraphics[width=\textwidth]{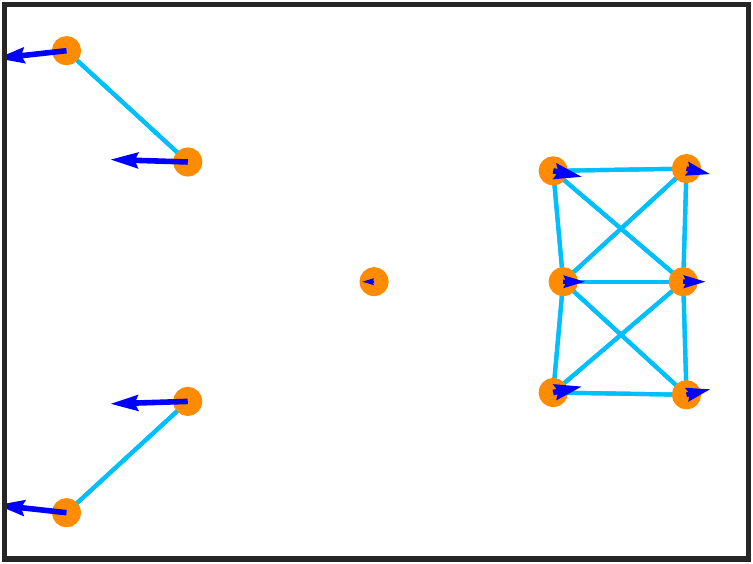}
\caption{\footnotesize $t_{53}$}
\label{fig:tanner:decentralized:dynamic:failure:time53}
\end{subfigure} \\
\begin{subfigure}[t]{.22\textwidth}
\centering
\includegraphics[width=\textwidth]{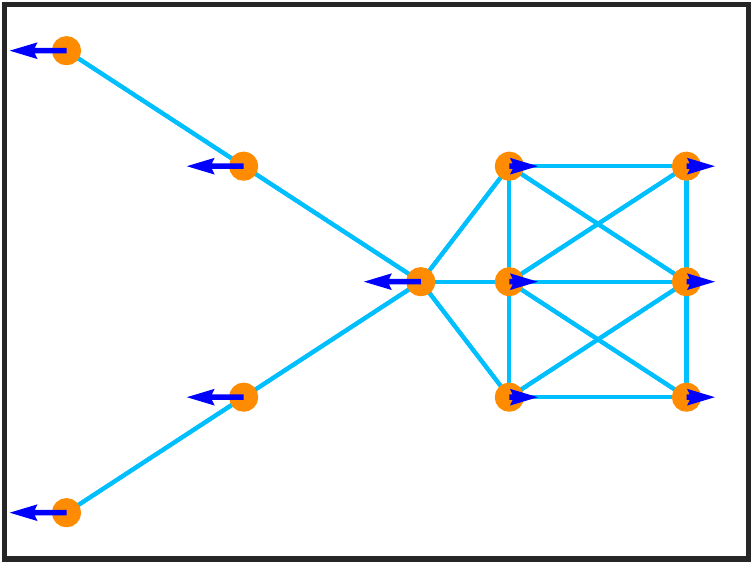}
\caption{\footnotesize $t_0$}\label{fig:tolstaya:dynamic:success:time0}
\end{subfigure}
\hspace{1em}
\begin{subfigure}[t]{.22\textwidth}
\centering
\includegraphics[width=\textwidth]{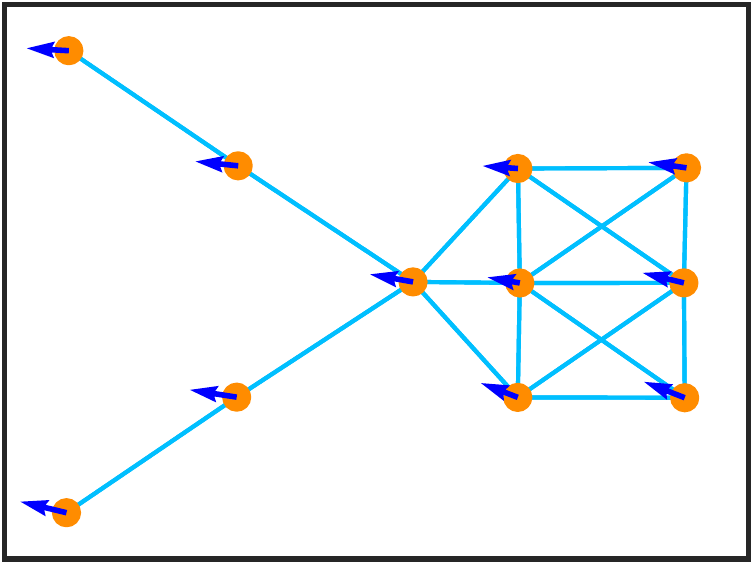}
\caption{\footnotesize $t_{27}$}
\label{fig:tolstaya:dynamic:success:time27}
\end{subfigure}
\hspace{1em}
\begin{subfigure}[t]{.22\textwidth}
\centering
\includegraphics[width=\textwidth]{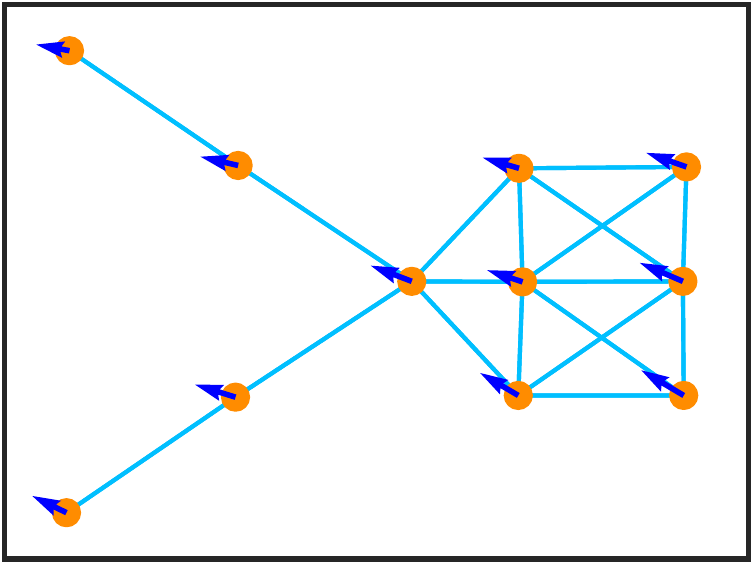}
\caption{\footnotesize $t_{53}$}
\label{fig:tolstaya:dynamic:success:time53}
\end{subfigure}
\caption{
Snapshots at times $t_\DAggerSimulationTimeStepIndex$ of a flock of agents (orange dots) with velocities (blue arrows) and the flock's communication graph (light blue edges). When controlled by the time-dependent neighborhood decentralized flocking controller from Tanner et al. \cite{tannerStableFlockingMobile2003} (top row), the communication graph loses connectivity from time $t_{27}$ onward. In contrast, the ML-based decentralized flocking controller from Tolstaya et al. \cite{tolstayaLearningDecentralizedControllers2017b} (bottom row), not trained on these agent initial conditions, successfully maintains communication graph connectivity and achieves asymptotic flocking.
}
\label{fig:tanner:decentralized:dynamic:failure}
\end{figure}

\subsection{ML-based decentralized flocking}\label{sec:bg:decentralized:ml}
\subsubsection{Time-delayed aggregation graph neural networks}
We review \PupilTDAGNN{} from \cite{tolstayaLearningDecentralizedControllers2017b}. 
Let $\FlockAgentNeighborhood_\FlockAgentIndex\pn{t} = \set{\FlockAgentIndexTwo: \FlockAgentIndexTwo \neq \FlockAgentIndex,\ \norm{\FlockAgentRelativePos} \leq \FlockAgentCommRadius}$ be the neighborhood of agent $\FlockAgentIndex$ at time $t$.
Define a function $\PupilFeatureFunc: \bR^2 \times \bR^2 \to \bR^{\gConvChannelsIn[1]}$ and the recurrence relation at time $t_\DAggerSimulationTimeStepIndex$ as
\begin{equation}
\begin{aligned}
 \PupilHistoryVector_\FlockAgentIndex^1\pn{t_\DAggerSimulationTimeStepIndex} & = \sum_{\FlockAgentIndexTwo \in \FlockAgentNeighborhood_\FlockAgentIndex\pn{t_\DAggerSimulationTimeStepIndex}} \psi\pn{\FlockAgentRelativePos\pn{t_\DAggerSimulationTimeStepIndex}, \FlockAgentRelativeVel\pn{t_\DAggerSimulationTimeStepIndex}}, 
\end{aligned}
\label{eq:tolstaya:node-features:one-hop}
\end{equation}
\begin{equation}
\begin{aligned}
\PupilHistoryVector_\FlockAgentIndex^\PupilHistoryIndex\pn{t_\DAggerSimulationTimeStepIndex} & = \frac{1}{\abs{\FlockAgentNeighborhood_\FlockAgentIndex\pn{t_\DAggerSimulationTimeStepIndex}}}\sum_{\FlockAgentIndexTwo \in \FlockAgentNeighborhood_\FlockAgentIndex\pn{t_\DAggerSimulationTimeStepIndex}} \PupilHistoryVector_\FlockAgentIndexTwo^{\PupilHistoryIndex - 1}\pn{t_{\DAggerSimulationTimeStepIndex - 1}}.
\end{aligned}
\label{eq:tolstaya:node-features:k-hop}
\end{equation}
With $\PupilFeatureFunc\pn{\FlockAgentRelativePos\pn{t_\DAggerSimulationTimeStepIndex}, \FlockAgentRelativeVel\pn{t_\DAggerSimulationTimeStepIndex}}$ representing the message from agent $\FlockAgentIndexTwo$ to agent $\FlockAgentIndex$ at time $t_\DAggerSimulationTimeStepIndex$, $\PupilHistoryVector_\FlockAgentIndex^1\pn{t_\DAggerSimulationTimeStepIndex}$ summarizes the messages within a one-hop neighborhood of agent $\FlockAgentIndex$, and $\PupilHistoryVector_\FlockAgentIndex^2\pn{t_\DAggerSimulationTimeStepIndex}$ summarizes the one-hop summaries from the previous time step within a one-hop neighborhood. Let $\PupilHistoryMatrix_\FlockAgentIndex\pn{t_\DAggerSimulationTimeStepIndex} = [ \PupilHistoryVector_\FlockAgentIndex^1\pn{t_\DAggerSimulationTimeStepIndex},\ldots,\PupilHistoryVector_\FlockAgentIndex^\PupilHistoryCount\pn{t_\DAggerSimulationTimeStepIndex}]$.
Finally, \PupilTDAGNN{} outputs the vector $\PupilName\pn{\PupilHistoryMatrix_\FlockAgentIndex\pn{t_\DAggerSimulationTimeStepIndex}}$ for agent $\FlockAgentIndex$'s acceleration,
where $\PupilName:\bR^{\gConvChannelsIn[1] \times \PupilHistoryCount} \to \bR^{\gConvChannelsIn[\PupilLayerCount + 1]}$ is a 1D CNN with element-wise $\tanh$ activation after each convolutional layer except the last layer $\PupilLayerCount$.
For agents in $\bR^2$, $\gConvChannelsIn[\PupilLayerCount + 1] = 2$.
$\PupilHistoryCount$ is the depth that the recurrence is computed to, and the larger $\PupilHistoryCount$ is, the more message summaries each agent needs to store; i.e., each agent $\FlockAgentIndex$ stores $\PupilHistoryMatrix_{\FlockAgentIndex} \pn{t_{\DAggerSimulationTimeStepIndex - 1}}$ to be passed on at time $t_{\DAggerSimulationTimeStepIndex}$.

\PupilTDAGNN{} is trained using imitation learning (IL) reviewed in \secref{sec:bg:imitation-learning}.
In our context of flocking control, the expert controller for IL is the centralized controller in \eqref{eq:tanner:centralized} with the potential function
\begin{equation}
\label{eq:tanner:centralized:tolstaya}
\begin{aligned}
    U\pn{r} = \begin{cases}
        1/r^2 + \ln\pn{r^2} & r \leq \FlockAgentCommRadius, \\
        U\pn{\FlockAgentCommRadius} & \text{else}.
    \end{cases}
\end{aligned}
\end{equation}
The output of the expert controller $\gExpertName$ is clamped so that $\norm{\gExpertName\pn{\gExpertInput}}_\infty \leq \ExpertAccelerationClampConstant$.
In \cite{tolstayaLearningDecentralizedControllers2017b}, $\PupilFeatureFunc: \bR^2 \times \bR^2 \to \bR^6$ computes the terms that enter linearly into the centralized controller's acceleration:
\begin{equation}
\label{eq:tolstaya:node-features}
\begin{aligned}
    \PupilFeatureFunc\pn{\FlockAgentRelativePos, \FlockAgentRelativeVel} = \spn{
        \FlockAgentRelativeVel[_{\FlockAgentIndex\FlockAgentIndexTwo}^\top] ,\; \norm{\FlockAgentRelativePos}^{-4}\FlockAgentRelativePos[_{\FlockAgentIndex\FlockAgentIndexTwo}^\top] ,\; \norm{\FlockAgentRelativePos}^{-2}\FlockAgentRelativePos[_{\FlockAgentIndex\FlockAgentIndexTwo}^\top]
    }^\top.
\end{aligned}
\end{equation}
\PupilTDAGNN{} trained using this $\TannerEtAlPotential$ and $\PupilFeatureFunc$ learns flocking control.
Using the $\TannerEtAlPotential$ from \cite{santosSegregationMultipleHeterogeneous2014} and the $\PupilFeatureFunc$ from \cite{omotuyiLearningDecentralizedControllers2022}, \PupilTDAGNN{} learns flocking control with segregation.

\subsubsection{Imitaiton learning}\label{sec:bg:imitation-learning}
Flocking controllers are evaluated by whether they achieve their collective objective 
and by what intermediate steps they take to do so. 
For example, controllers that guarantee separation are devalued if they take intermediate steps that cause agent collisions. While non-ML controllers can be designed with these evaluation criteria in mind, ML controllers must be trained to comply with the constraints imposed by these criteria. IL is a framework that trains ML controllers to produce the same output as a compliant ``expert'' controller. Let $\gExpertName$ be the expert controller in \eqref{eq:tanner:centralized:tolstaya} and \PupilName{} be \PupilTDAGNN{} -- a pupil controller.
The IL loss function used in \cite{tolstayaLearningDecentralizedControllers2017b} is
\begin{equation}
\label{eq:imitation-learning:loss-function}
\begin{aligned}
    & \LossFunctionName\pn{\gDAggerDatasetDatum} = \frac{1}{\FlockAgentCount}\sum_{\FlockAgentIndex = 1}^\FlockAgentCount \ell\pn{\gExpertName\pn{\gExpertInput}, \PupilName\pn{\PupilHistoryMatrix_\FlockAgentIndex}},
\end{aligned}
\end{equation}
where $\ell$ is another loss function (e.g., the squared error loss), and $\PupilHistoryMatrix_\FlockAgentIndex$ is the input to \PupilTDAGNN{} for agent $\FlockAgentIndex$. 
Ideally, we would compute $\LossFunctionName$ for every tuple $\pn{\gDAggerDatasetDatum}$ 
but this is intractable. Instead, we sample a training set from the space of these tuples using Dataset Aggregation (DAgger) reinforcement learning \cite{rossReductionImitationLearning2011}. 
DAgger is employed in \cite{tolstayaLearningDecentralizedControllers2017b} to train \PupilTDAGNN{} as follows: 
Let $E$ and $T$ be the number of training epochs and time steps in a flocking simulation, respectively.
At the beginning of each epoch $\DAggerEpochIndex \in \set{0, \dots, \DAggerEpochCount - 1}$, an initial flock state $\pn{\FlockPosMatrix\pn{t_0}, \FlockVelMatrix\pn{t_0}}$ is chosen from a dataset of initial conditions. Next, flocking is simulated for $T$ time steps. For each $\DAggerSimulationTimeStepIndex \in \set{0,\dots,\DAggerSimulationTimeStepsCount-1}$, $\gExpertName\pn{\gExpertInput[\pn{t_\DAggerSimulationTimeStepIndex}]}$ and $\PupilHistoryMatrix_\FlockAgentIndex\pn{t_\DAggerSimulationTimeStepIndex}$ are computed. With probability $\DAggerActionPropability_\DAggerEpochIndex$, the acceleration $\gExpertName\pn{\gExpertInput[\pn{t_\DAggerSimulationTimeStepIndex}]}$ is used to update the flock state, computing $\FlockPosMatrix\pn{t_{\DAggerSimulationTimeStepIndex + 1}}$ and $\FlockVelMatrix\pn{t_{\DAggerSimulationTimeStepIndex + 1}}$; otherwise, the acceleration $\PupilName\pn{\PupilHistoryMatrix_\FlockAgentIndex\pn{t_\DAggerSimulationTimeStepIndex}}$ is used.
Finally, the tuple $\pn{\gDAggerDatasetDatum[\pn{t_\DAggerSimulationTimeStepIndex}]}$ is added to the training set. Once $T$ time steps of flocking are complete, a batch $\GeneralizationSampleName$ of tuples is uniformly randomly sampled from the training set.
The average value of $\LossFunctionName$ over the batch is computed and \PupilTDAGNN{}'s weights are updated.
We can update the weights more than once per epoch by sampling more batches.

\subsection{Rotation equivariance and translation invariance of flocking controllers
}
\label{sec:flocking-symmetry}
Symmetry is a fundamental inductive bias for designing reliable and efficient neural networks \cite{pmlr-v80-kondor18a,pmlr-v162-lawrence22a,van2020mdp,yarotsky2022universal,wang2024rethinking}.
The symmetry of a 
function $f: X \to Y$ is often described by equivariance: $f$ is equivariant if $f(T_g(x)) = T'_g(f(x))$, where $T_g, T'_g$ are transformations representing the group element $g$ on $X$ and $Y$, respectively.
If $T_g'$ is the identity, $f$ is invariant.
A crucial property of GNNs is that their node feature aggregations are permutation invariant \cite{garg2020generalization}, and GNNs have also been extended to be roto-translation equivariant \cite{satorras2021n,brandstetter2022geometric}.
For flocking, prominent controllers $\gExpertName$ (e.g., \cite{tannerStableFlockingMobile2003a,cuckerEmergentBehaviorFlocks2007}) have the following symmetries:
\begin{itemize}
\item \textbf{Translation invariance:} For any $\vt_1, \vt_2 \in \bR^2$,  
considering the controller in \eqref{eq:tanner:centralized}, we have 
$$
\begin{aligned}
&\ \ \ \ \ \gExpertName\pn{\FlockPosMatrix + \vt_1\vone_\FlockAgentCount^\top, \FlockVelMatrix + \vt_2\vone_\FlockAgentCount^\top}\\ & = -\sum_{\FlockAgentIndexTwo = 1}^{\FlockAgentCount} \pn{\FlockAgentVel_\FlockAgentIndex + \vt_2} - \pn{\FlockAgentVel_\FlockAgentIndexTwo + \vt_2} - \sum_{\FlockAgentIndexTwo = 1}^{\FlockAgentCount} \nabla\TannerEtAlPotential\pn{\norm{\pn{\FlockAgentPos_\FlockAgentIndex + \vt_1} - \pn{\FlockAgentPos_\FlockAgentIndexTwo + \vt_1}}} \\
& = -\sum_{\FlockAgentIndexTwo = 1}^{\FlockAgentCount} \FlockAgentRelativeVel - \sum_{\FlockAgentIndexTwo = 1}^{\FlockAgentCount} \nabla\TannerEtAlPotential\pn{\norm{\FlockAgentRelativePos}} \\
& = \gExpertName\pn{\gExpertInput}
\end{aligned}
$$

\item \textbf{Rotation and reflection equivariance:} For any orthogonal matrix $\UtilOrthogonalMatrix \in \OrthogonalGroup{2}$, we have
$$
\begin{aligned}
\gExpertName\pn{\UtilOrthogonalMatrix\FlockPosMatrix, \UtilOrthogonalMatrix\FlockVelMatrix} & = -\sum_{\FlockAgentIndexTwo = 1}^{\FlockAgentCount} \UtilOrthogonalMatrix\FlockAgentVel_\FlockAgentIndex - \UtilOrthogonalMatrix\FlockAgentVel_\FlockAgentIndexTwo - \sum_{\FlockAgentIndexTwo = 1}^{\FlockAgentCount} \nabla\TannerEtAlPotential\pn{\norm{\UtilOrthogonalMatrix\FlockAgentPos_\FlockAgentIndex - \UtilOrthogonalMatrix\FlockAgentPos_\FlockAgentIndexTwo}} \\
& = -\sum_{\FlockAgentIndexTwo = 1}^{\FlockAgentCount} \UtilOrthogonalMatrix\FlockAgentRelativeVel - \sum_{\FlockAgentIndexTwo = 1}^{\FlockAgentCount} \TannerEtAlPotential'\pn{\norm{\UtilOrthogonalMatrix\FlockAgentRelativePos}}\frac{\UtilOrthogonalMatrix\FlockAgentRelativePos}{\norm{\UtilOrthogonalMatrix\FlockAgentRelativePos}} \\
& = \UtilOrthogonalMatrix\pn*{-\sum_{\FlockAgentIndexTwo = 1}^{\FlockAgentCount} \FlockAgentRelativeVel - \sum_{\FlockAgentIndexTwo = 1}^{\FlockAgentCount} \nabla\TannerEtAlPotential\pn{\norm{\FlockAgentRelativePos}}} \\
& = \UtilOrthogonalMatrix\gExpertName\pn{\gExpertInput}.
\end{aligned}
$$
\end{itemize}
Existing ML models, like TDAGNN, do not satisfy these symmetries.

\section{Equivariant controllers for learning decentralized flocking}\label{sec:model-proposed}
In this section, we present a rotation equivariant ML controller for decentralized flocking. Our approach replaces the CNN in \PupilTDAGNN{} with an \OrthogonalGroup{2} equivariant CNN, ensuring rotation and reflection equivariance of the resulting ML controller.


\subsection{
Rotation equivariant convolution layers
}
\label{sec:methods:equiv-model}
\PupilTDAGNN{} is translation invariant, but not rotation equivariant because its CNN component \PupilName{} is not. 
There has been significant effort put toward developing roto-translation equivariant CNNs (e.g., \SpecialOrthogonalGroup{2}-steerable CNNs \cite{weiler2019general}); however, these cannot be directly integrated into \PupilTDAGNN{}. As such, we aim to replace \PupilName{} with an \OrthogonalGroup{2} equivariant CNN \PupilEquivariantName{} equipped with rotation equivariant convolutional layers and activations. To ease our presentation, we make the following two assumptions:


\begin{assumption}\label{assumption:methods:no-padding}
The 1D convolutional layers 
$$\set{\ConvOneD_\PupilLayerIndex: \bR^{\gConvChannelsIn \times \gConvFeaturesIn} \to \bR^{\gConvChannelsIn[\PupilLayerIndex+1] \times \gConvFeaturesIn[\PupilLayerIndex+1]}}_{\PupilLayerIndex = 1}^{\PupilLayerCount}
$$ 
of \PupilName{} have no padding.
\end{assumption}

\begin{assumption}\label{assumption:methods:same-size-input}
The input of \PupilName{} always has the same size.
\end{assumption}

We represent the input $\gConvInput$ of $\ConvOneD_\PupilLayerIndex$ as a block matrix composing vectors in $\bR^2$ that are \OrthogonalGroup{2} equivariant with respect to the agents' positions and velocities:
\begin{align*}
\gConvInput = \spn{
\EqConvChannel_{\ConvChannelInIndex,\ConvFeatureIndex}
}_{\ConvChannelInIndex=1,\ConvFeatureIndex=1}^{\pn{\gConvChannelsIn/2},\gConvFeaturesIn} \subset \bR^{\gConvChannelsIn \times \gConvFeaturesIn}
\end{align*}
If $\PupilLayerIndex = 1$ and we 
use $\PupilFeatureFunc$ from \eqref{eq:tolstaya:node-features}, then $\gConvInput = \PupilHistoryMatrix_\FlockAgentIndex$ from section~\ref{sec:bg:decentralized:ml}, 
$\gConvChannelsIn = 6$, and $\gConvFeaturesIn = \PupilHistoryCount$.
Moreover, the last convolutional layer 
in $\PupilName$ outputs acceleration so $\gConvChannelsIn[\PupilLayerCount+1] = 2$ and $\gConvFeaturesIn[\PupilLayerCount+1] = 1$.
By Assumption~\ref{assumption:methods:no-padding}, $\ConvOneD_\PupilLayerIndex$ can be represented by a collection of Toeplitz matrices $\PupilWeightMatrix: \set{1,\dots,\gConvChannelsIn} \times \set{1,\dots,\gConvChannelsIn[\PupilLayerIndex+1]} \to \bR^{\gConvFeaturesIn \times \gConvFeaturesIn[\PupilLayerIndex+1]}$ and bias terms $\ConvBias: \set{1,\dots,\gConvChannelsIn[\PupilLayerIndex+1]} \to \bR$.
For $\ConvChannelOutIndex \in \set{1,\dots,\gConvChannelsIn[\PupilLayerIndex+1]}$, the output channel $\ConvOneD_\PupilLayerIndex\pn{\gConvInput}[\ConvChannelOutIndex]$ can be represented as
$$
\begin{aligned}
    &\ \ \ConvOneD_\PupilLayerIndex\pn{\gConvInput}[\ConvChannelOutIndex]\\ 
    & = \ConvBias\pn{\ConvChannelOutIndex}\vone_{\gConvFeaturesIn[\PupilLayerIndex+1]}^\top + \sum_{\ConvChannelInIndex = 1}^{\gConvChannelsIn} \spn{
        \pn{\gConvInput}_{\ConvChannelInIndex,1}
        ,\; \dots
        ,\; \pn{\gConvInput}_{\ConvChannelInIndex,\gConvFeaturesIn}
    }\PupilWeightMatrix\pn{\ConvChannelInIndex,\ConvChannelOutIndex} \\
    & = \ConvBias\pn{\ConvChannelOutIndex}\vone_{\gConvFeaturesIn[\PupilLayerIndex+1]}^\top + \sum_{\ConvChannelInIndex = 1}^{\gConvChannelsIn/2} \sum_{c = 1}^2 \spn{
        \pn{\EqConvChannel_{\ConvChannelInIndex, 1}}_c
        ,\; \dots
        ,\; \pn{\EqConvChannel_{\ConvChannelInIndex, \gConvFeaturesIn}}_c
    }\PupilWeightMatrix\pn{2\pn{\ConvChannelInIndex - 1} + c, \ConvChannelOutIndex} \\
    & = \ConvBias\pn{\ConvChannelOutIndex}\vone_{\gConvFeaturesIn[\PupilLayerIndex+1]}^\top + \sum_{\ConvChannelInIndex = 1}^{\gConvChannelsIn/2} \vone_2^\top\begin{bmatrix}
        \spn{
            \pn{\EqConvChannel_{\ConvChannelInIndex, 1}}_1
            ,\; \dots
            ,\; \pn{\EqConvChannel_{\ConvChannelInIndex, \gConvFeaturesIn}}_1
        }\PupilWeightMatrix\pn{2\pn{\ConvChannelInIndex - 1} + 1, \ConvChannelOutIndex} \\[.1em]
        \spn{
            \pn{\EqConvChannel_{\ConvChannelInIndex, 1}}_2
            ,\; \dots
            ,\; \pn{\EqConvChannel_{\ConvChannelInIndex, \gConvFeaturesIn}}_2
        }\PupilWeightMatrix\pn{2\pn{\ConvChannelInIndex - 1} + 2, \ConvChannelOutIndex}
    \end{bmatrix}.
\end{aligned}
$$
Applying an orthogonal matrix $\UtilOrthogonalMatrix \in \OrthogonalGroup{2}$ to the agents' position and velocity, we have
$$
\begin{aligned}
    &\ConvOneD_\PupilLayerIndex\pn*{\spn{
        \UtilOrthogonalMatrix\EqConvChannel_{\ConvChannelInIndex,\ConvFeatureIndex}
    }_{\ConvChannelInIndex=1,\ConvFeatureIndex=1}^{\pn{\gConvChannelsIn/2},\gConvFeaturesIn}}[\ConvChannelOutIndex]\\
     =\; & \ConvBias\pn{\ConvChannelOutIndex}\vone_{\gConvFeaturesIn[\PupilLayerIndex+1]}^\top 
     + \sum_{\ConvChannelInIndex = 1}^{\gConvChannelsIn/2} \vone_2^\top\begin{bmatrix}
        \spn{
            \pn{\UtilOrthogonalMatrix\EqConvChannel_{\ConvChannelInIndex, 1}}_1
            ,\; \dots
            ,\; \pn{\UtilOrthogonalMatrix\EqConvChannel_{\ConvChannelInIndex, \gConvFeaturesIn}}_1
        }\PupilWeightMatrix\pn{2\pn{\ConvChannelInIndex - 1} + 1, \ConvChannelOutIndex} \\[.1em]
        \spn{
            \pn{\UtilOrthogonalMatrix\EqConvChannel_{\ConvChannelInIndex, 1}}_2
            ,\; \dots
            ,\; \pn{\UtilOrthogonalMatrix\EqConvChannel_{\ConvChannelInIndex, \gConvFeaturesIn}}_2
        }\PupilWeightMatrix\pn{2\pn{\ConvChannelInIndex - 1} + 2, \ConvChannelOutIndex}
    \end{bmatrix} \\
    &
    \neq \UtilOrthogonalMatrix\ConvOneD_\PupilLayerIndex\pn*{\spn{
        \EqConvChannel_{\ConvChannelInIndex,\ConvFeatureIndex}
    }_{\ConvChannelInIndex=1,\ConvFeatureIndex=1}^{\pn{\gConvChannelsIn/2},\gConvFeaturesIn}}.
\end{aligned}
$$
From the above analysis, we see that {\bf $\ConvOneD_\PupilLayerIndex$ is not \OrthogonalGroup{2} equivariant} for three reasons: (1) the bias term, (2) vector element indexing is not \OrthogonalGroup{2} equivariant, and (3) $\UtilOrthogonalMatrix$ does not commute with $\vone_2^\top$. 
We need to address these issues to develop an \OrthogonalGroup{2} equivariant convolution. For the first issue, we eliminate the bias term.
For the second issue, we eliminate the vector element indexing through weight sharing.
By setting $\PupilWeightMatrix\pn{2\pn{\ConvChannelInIndex - 1} + 1, \ConvChannelOutIndex} = \PupilWeightMatrix\pn{2\pn{\ConvChannelInIndex - 1} + 2, \ConvChannelOutIndex}$ for all $\ConvChannelInIndex$, then for an output channel $\ConvChannelOutIndex$, we have
$$
\begin{aligned}
    &\ \  \sum_{\ConvChannelInIndex = 1}^{\gConvChannelsIn/2} \vone_2^\top\begin{bmatrix}
        \spn{
            \pn{\EqConvChannel_{\ConvChannelInIndex, 1}}_1
            ,\; \dots
            ,\; \pn{\EqConvChannel_{\ConvChannelInIndex, \gConvFeaturesIn}}_1
        }\PupilWeightMatrix\pn{2\pn{\ConvChannelInIndex - 1} + 1, \ConvChannelOutIndex} \\[.1em]
        \spn{
            \pn{\EqConvChannel_{\ConvChannelInIndex, 1}}_2
            ,\; \dots
            ,\; \pn{\EqConvChannel_{\ConvChannelInIndex, \gConvFeaturesIn}}_2
        }\PupilWeightMatrix\pn{2\pn{\ConvChannelInIndex - 1} + 2, \ConvChannelOutIndex}
    \end{bmatrix} \\
 &    =  \sum_{\ConvChannelInIndex = 1}^{\gConvChannelsIn/2} \vone_2^\top\spn{
        \EqConvChannel_{\ConvChannelInIndex, 1}
        ,\; \dots
        ,\; \EqConvChannel_{\ConvChannelInIndex, \gConvFeaturesIn}
    }\PupilWeightMatrix\pn{2\pn{\ConvChannelInIndex - 1} + 1, \ConvChannelOutIndex}.
\end{aligned}
$$
For the third issue, we drop $\vone_2^\top$, which changes the number of rows per output channel from one to two. To retain the original number of rows in the output of $\ConvOneD$, we halve the number of output channels but set each output channel to have two rows (or ``subchannels'').
With these changes, we define a new \OrthogonalGroup{2} equivariant convolution $\EqConv_\PupilLayerIndex$: For output channels $\ConvChannelOutIndex \in \set{1,\dots,\gConvChannelsIn[\PupilLayerIndex+1]/2}$,
$$
\begin{aligned}
    \EqConv_\PupilLayerIndex\pn{\gConvInput}[\ConvChannelOutIndex] & = \sum_{\ConvChannelInIndex = 1}^{\gConvChannelsIn/2} \spn{
        \EqConvChannel_{\ConvChannelInIndex, 1}
        ,\; \dots
        ,\; \EqConvChannel_{\ConvChannelInIndex, \gConvFeaturesIn}
    }\PupilWeightMatrix\pn{2\pn{\ConvChannelInIndex - 1} + 1, \ConvChannelOutIndex}.
\end{aligned}
$$
$\EqConv_\PupilLayerIndex$ is \OrthogonalGroup{2} equivariant since
$$
\begin{aligned}
  &\ \   \EqConv_\PupilLayerIndex\pn*{\spn{
        \UtilOrthogonalMatrix\EqConvChannel_{\ConvChannelInIndex,\ConvFeatureIndex}
    }_{\ConvChannelInIndex=1,\ConvFeatureIndex=1}^{\pn{\gConvChannelsIn/2},\gConvFeaturesIn}}[\ConvChannelOutIndex] \\
    & = \sum_{\ConvChannelInIndex = 1}^{\gConvChannelsIn/2} \spn{
        \UtilOrthogonalMatrix\EqConvChannel_{\ConvChannelInIndex, 1}
        ,\; \dots
        ,\; \UtilOrthogonalMatrix\EqConvChannel_{\ConvChannelInIndex, \gConvFeaturesIn}
    }\PupilWeightMatrix\pn{2\pn{\ConvChannelInIndex - 1} + 1, \ConvChannelOutIndex} \\
    & = \UtilOrthogonalMatrix\EqConv_\PupilLayerIndex\pn*{\spn{
        \EqConvChannel_{\ConvChannelInIndex,\ConvFeatureIndex}
    }_{\ConvChannelInIndex=1,\ConvFeatureIndex=1}^{\pn{\gConvChannelsIn/2},\gConvFeaturesIn}}[\ConvChannelOutIndex].
\end{aligned}
$$
$\EqConv_\PupilLayerIndex$ can be implemented using a standard 2D convolutional layer whose kernels have one row and a vertical stride of one.
%

\subsection{
Rotation equivariant activations}
\label{sec:methods:equivariant-activations}
To build an \OrthogonalGroup{2} equivariant CNN, we need \OrthogonalGroup{2} equivariant activations.
Notice that the output of $\EqConv_\PupilLayerIndex$ is also a block matrix composing 2D vectors that are \OrthogonalGroup{2} equivariant:
$$
\begin{aligned}
    \gConvInput[\PupilLayerIndex+1] = \EqConv_\PupilLayerIndex\pn{\gConvInput} = \spn{
        \EqConvChannel_{\ConvChannelOutIndex,\ConvFeatureIndex}
    }_{\ConvChannelOutIndex=1,\ConvFeatureIndex=1}^{\pn{\gConvChannelsIn[\PupilLayerIndex+1]/2},\gConvFeaturesIn[\PupilLayerIndex+1]} \subset \bR^{\gConvChannelsIn[\PupilLayerIndex+1] \times \gConvFeaturesIn[\PupilLayerIndex+1]}.
\end{aligned}
$$
We construct an \OrthogonalGroup{2} equivariant activation for $\gConvInput[\PupilLayerIndex+1]$ by applying an \OrthogonalGroup{2} equivariant activation to each $\EqConvChannel_{\ConvChannelOutIndex, \ConvFeatureIndex}$ separately. 
By Proposition~1 of \cite{maWhySelfattentionNatural2022}, any O(2) equivariant function  can be expressed as $\PupilActivationEquivariantName\pn{\FlockAgentPos} = \FlockAgentPos\PupilActivationInvariantName\pn{\norm{\FlockAgentPos}}$ for some function $\PupilActivationInvariantName: \bR \to \bR$.
Therefore, the output channel $\ConvChannelOutIndex \in \set{1,\dots,\gConvChannelsIn[\PupilLayerIndex+1]/2}$ with an equivariant activation applied is given by
$$
\begin{aligned}
    \PupilActivationEquivariantName\pn{\gConvInput[\PupilLayerIndex+1]} [\ConvChannelOutIndex] & = \spn{
        \PupilActivationEquivariantName\pn{\EqConvChannel_{\ConvChannelOutIndex, 1}}
        ,\; \dots
        ,\; \PupilActivationEquivariantName\pn{\EqConvChannel_{\ConvChannelOutIndex, \gConvFeaturesIn[\PupilLayerIndex+1]}}
    }\\
    &= \gConvInput[\PupilLayerIndex+1] [\ConvChannelOutIndex] \odot \vone_2 \spn{
        \PupilActivationInvariantName\pn{\norm{\EqConvChannel_{\ConvChannelOutIndex, 1}}}
        ,\; \dots
        ,\; \PupilActivationInvariantName\pn{\norm{\EqConvChannel_{\ConvChannelOutIndex, \gConvFeaturesIn[\PupilLayerIndex+1]}}}
    },
\end{aligned}
$$
where $\odot$ denotes the Hadamard product. The \OrthogonalGroup{2} equivariant convolution layer and \OrthogonalGroup{2} equivariant activations are combined to create an \OrthogonalGroup{2} equivariant convolution neural network \PupilEquivariantName{}.

We construct the following tailored activations for our application:
$$
\begin{aligned}
\PupilETDAGNNActivationScaleLogInvariant\pn{x} = \begin{cases}
        1 & x = 0 \\
        \frac{\ln\pn{1 + x}}{x} & x \neq 0
    \end{cases}, \quad \text{and } \PupilETDAGNNActivationScaleTanhInvariant\pn{x} = \tanh\pn{x},
\end{aligned}
$$
$$
\begin{aligned}
    \, & \PupilETDAGNNActivationScaleLog\pn{\gConvInput{}} [\ConvChannelOutIndex] = \gConvInput{} [\ConvChannelOutIndex] \odot \vone_2 \spn{
        \PupilETDAGNNActivationScaleLogInvariant\pn{\norm{\EqConvChannel_{\ConvChannelOutIndex, 1}}}
        ,\; \dots
        ,\; \PupilETDAGNNActivationScaleLogInvariant\pn{\norm{\EqConvChannel_{\ConvChannelOutIndex, \gConvFeaturesIn[\PupilLayerIndex+1]}}}
    }, \\
    \, & \PupilETDAGNNActivationScaleTanh\pn{\gConvInput{}} [\ConvChannelOutIndex] = \gConvInput{} [\ConvChannelOutIndex] \odot \vone_2 \spn{
        \PupilETDAGNNActivationScaleTanhInvariant\pn{\norm{\EqConvChannel_{\ConvChannelOutIndex, 1}}}
        ,\; \dots
        ,\; \PupilETDAGNNActivationScaleTanhInvariant\pn{\norm{\EqConvChannel_{\ConvChannelOutIndex, \gConvFeaturesIn[\PupilLayerIndex+1]}}}
    }.
\end{aligned}
$$
The activation \PupilETDAGNNActivationScaleLog{} provides an important normalization effect similar to $\tanh$ in the non-equivariant ML controller \PupilTDAGNN{}.
Large convolutional layer inputs can occur because the input features in \eqref{eq:tolstaya:node-features} are unbounded as $\norm{\FlockAgentRelativePos} \to 0$, potentially causing instabilities during training. Applying $\tanh$ elementwise bounds the features passed inside the CNN to $\pn{-1, 1}$. However, in the non-equivariant controllers, $\tanh$ is not applied before the first convolutional layer because it reduces all large-magnitude features to approximately the same magnitude (i.e., it is saturated); doing so would remove important information about the proximity of neighboring agents. Unfortunately, the first convolutional layer is still subject to large inputs. In contrast, for \PupilEquivariantName{} we can reduce the size of the inputs to the first convolutional layer and avoid saturation by using $\PupilETDAGNNActivationScaleLog$.
Next, the activation \PupilETDAGNNActivationScaleTanh{} is simply a nonlinear scaling of the feature vectors. 

\subsection{
Further improving \PupilTDAGNN{}}\label{sec:methods:improving-tdagnn}
We present two strategies to improve both equivariant and non-equivariant ML controllers further. 

{\bf Activate once:} The $\tanh$ activation of the CNN \PupilName{}  normalizes the output of each convolutional layer 
into $\pn{-1, 1}$.
Since the activation is applied after all but the last convolutional layer, the last convolutional layer must have large weights to ensure the controller can output a large acceleration vector. A large acceleration is needed whenever the controller needs to react quickly to maintain a connected communication graph or avoid collision. However, large weights concentrated in the last convolutional layer may cause the controller to output large accelerations too often, leading to over-corrections of the agents' velocities. When training \PupilTDAGNN{}, the CNN \PupilName{} must balance the need for large weights in the last convolutional layer with their risk. This challenge is lessened by applying the activation only after the first convolutional layer, allowing the controller to learn large weights across its convolutional layers.

{\bf Mean one-hop aggregation:} 
In \eqref{eq:tolstaya:node-features:one-hop}, \PupilTDAGNN{} uses a sum aggregation of the messages in the one-hop neighborhood of agent $\FlockAgentIndex$. Sum aggregation can cause over-corrections when agent $\FlockAgentIndex$'s neighbors send similar messages. Consider the scenario where all neighbors of agent $\FlockAgentIndex$ have the same velocity $\vv_{\FlockAgentNeighborhood_\FlockAgentIndex}$ and agent $\FlockAgentIndex$ has velocity $\FlockAgentVel_\FlockAgentIndex \neq \vv_{\FlockAgentNeighborhood_\FlockAgentIndex}$. Each neighbor sends a message to agent $\FlockAgentIndex$ requesting that it corrects its velocity by $-\pn{\FlockAgentVel_\FlockAgentIndex - \vv_{\FlockAgentNeighborhood_\FlockAgentIndex}}$. With sum aggregation, agent $\FlockAgentIndex$ overcorrects its velocity by $-\abs{\FlockAgentNeighborhood_\FlockAgentIndex}\pn{\FlockAgentVel_\FlockAgentIndex - \vv_{\FlockAgentNeighborhood_\FlockAgentIndex}}$, but with mean aggregation, agent $\FlockAgentIndex$ corrects its velocity by $-\pn{\FlockAgentVel_\FlockAgentIndex - \vv_{\FlockAgentNeighborhood_\FlockAgentIndex}}$, exactly what was requested by all its neighbors.

\section{Generalization analysis 
}\label{sec:main:theory}
In this section, we 
analyze the generalizability
of \PupilTDAGNN{}-related models when trained for flocking following the generalization analysis 
framework in \cite{pmlr-v235-karczewski24a}. 

\subsection{Generalization gap}


\begin{definition}[Generalization gap]\label{def:generalization-gap}
Let $f: \GeneralizationDomain \to \GeneralizationRange$ and define the loss function $\LossFunctionName: \GeneralizationRange \times \GeneralizationRange \to \hlpn{0, \infty}$. Let $\GeneralizationSampleName = \set{\pn{x_\GeneralizationSampleIndex, y_\GeneralizationSampleIndex}}_{\GeneralizationSampleIndex = 1}^\GeneralizationSampleSize$ be a set of i.i.d. samples from a probability distribution $\GeneralizationDatasetDistribution$ over $\GeneralizationDomain \times \GeneralizationRange$. The \textit{generalization gap} $\gGeneralizationGap$ of $f$ is the difference between the expected risk and empirical risk, i.e.,
$$
\begin{aligned}
\gGeneralizationGap\pn{f} = \underbrace{\bE_{\pn{x, y} \sim \GeneralizationDatasetDistribution}\spn{\LossFunctionName\pn{f\pn{x}, y}}}_{\mathrm{expected\ risk}} - \underbrace{\frac{1}{\GeneralizationSampleSize}\sum_{\GeneralizationSampleIndex = 1}^\GeneralizationSampleSize \LossFunctionName\pn{f\pn{x_\GeneralizationSampleIndex}, y_\GeneralizationSampleIndex}}_{\mathrm{empirical\ risk}}.
\end{aligned}
$$
\end{definition}

The generalization gap can be bounded by the expressiveness of the class of ML models. 
Empirical Rademacher complexity (ERC) 
measures how well the family of functions can fit random noise.

\begin{definition}[Empirical Rademacher complexity]\label{def:empirical-rademacher-complexity}
Let $\GeneralizationFunctionsLearners = \set{f: \GeneralizationDomain \to \bR}$ be a family of bounded functions and $\GeneralizationSampleName = \set{x_\GeneralizationSampleIndex}_{\GeneralizationSampleIndex = 1}^\GeneralizationSampleSize \subset \GeneralizationDomain$. The \textit{empirical Rademacher complexity} (ERC) of $\GeneralizationFunctionsLearners$ is
$$
\begin{aligned}
\GeneralizationEmpiricalRademacherComplexity\pn{\GeneralizationFunctionsLearners} = \bE_\GeneralizationRademacherRandomVector\spn*{\sup_{f \in \GeneralizationFunctionsLearners} \frac{1}{\GeneralizationSampleSize}\sum_{\GeneralizationSampleIndex = 1}^\GeneralizationSampleSize \pn{\GeneralizationRademacherRandomVector}_\GeneralizationSampleIndex f\pn{x_\GeneralizationSampleIndex}},
\end{aligned}
$$
where the entries of $\vsigma \in \set{-1, 1}^\GeneralizationSampleSize$ are distributed such that $P\pn{\pn{\vsigma}_\GeneralizationSampleIndex = -1} = 1/2$ and $P\pn{\pn{\vsigma}_\GeneralizationSampleIndex = 1} = 1/2$.
\end{definition}

Using the ERC, a probabilistic bound for the generalization gap can be derived.

\begin{theorem}[ERC bounds generalization gap \cite{mohriFoundationsMachineLearning2012}]
\label{theorem:erc-bounds-generalization-gap}
Define the loss funciton $\LossFunctionName: \GeneralizationRange \times \GeneralizationRange \to \spn{0, 1}$, let $\GeneralizationFunctionsLearners = \set{f: \GeneralizationDomain \to \GeneralizationRange}$, and let $\GeneralizationDatasetDistribution$ be a probability distribution over $\GeneralizationDomain \times \GeneralizationRange$. Let $\GeneralizationSampleName = \set{\pn{x_\GeneralizationSampleIndex, y_\GeneralizationSampleIndex}}_{\GeneralizationSampleIndex = 1}^\GeneralizationSampleSize$ be a set of i.i.d. samples from $\GeneralizationDatasetDistribution$. For any $\GeneralizationGapBoundFailProbability > 0$, for all $f \in \GeneralizationFunctionsLearners$, the following bound holds with probability $1 - \GeneralizationGapBoundFailProbability$:
$$
\begin{aligned}
\gGeneralizationGap\pn{f} \leq 2\GeneralizationEmpiricalRademacherComplexity\pn{\GeneralizationFunctionsDatumToLoss} + 3\sqrt{\frac{\ln\pn*{\frac{2}{\GeneralizationGapBoundFailProbability}}}{2\GeneralizationSampleSize}}, 
\end{aligned}
$$
where
$\GeneralizationFunctionsDatumToLoss = \set{\pn{x, y} \mapsto \LossFunctionName\pn{f\pn{x}, y}: f \in \GeneralizationFunctionsLearners}$.
\end{theorem}

Therefore, we only need to bound ERC. 
The covering number is a celebrated tool to bound ERC. 

{
\newcommand{\zSetOfInterest}{{\cZ}}
\newcommand{\zSetOfInterestItem}{{z}}
\newcommand{\zCoveringSet}{{\zSetOfInterest'}}
\newcommand{\zCoveringSetItem}{{\zSetOfInterestItem'}}
\newcommand{\zCoveringDistance}{{r}}
\begin{definition}[Covering number]
    \label{def:covering-number}
    The covering number $\GeneralizationCoveringNumber\pn{\zSetOfInterest, \zCoveringDistance, \UtilNormWithPlaceholder}$ of a set $\zSetOfInterest$ with respect to some norm $\UtilNormWithPlaceholder$ is the minimum cardinality of a set $\zCoveringSet$ such that, for any element of $\zSetOfInterest$, there is an element in $\zCoveringSet$ that is within a distance $\zCoveringDistance$ of it, i.e., $\GeneralizationCoveringNumber\pn{\zSetOfInterest, \zCoveringDistance, \UtilNormWithPlaceholder} = \min\set{\abs{\zCoveringSet}: \forall \zSetOfInterestItem \in \zSetOfInterest,\ \exists \zCoveringSetItem \in \zCoveringSet,\ \norm{\zSetOfInterestItem - \zCoveringSetItem} \leq \zCoveringDistance}$. 
\end{definition}
}

Now, we state an established bound of the ERC in terms of the covering number.

\begin{lemma}[Bounding ERC \cite{bartlettSpectrallynormalizedMarginBounds2017}]\label{lemma:bounding-erc}
Let $\GeneralizationFunctionsLearners = \set{f: \GeneralizationDomain \to \spn{-\GeneralizationBartlettBound,\GeneralizationBartlettBound}}$, and assume that there exists a function $\GeneralizationFunctionDatumToLossZero \in \GeneralizationFunctionsLearners$ such that $\GeneralizationFunctionDatumToLossZero\pn{x} = 0$ for all $x \in \GeneralizationDomain$. With $\norm{f}_\infty = \sup_{x \in \GeneralizationDomain} \abs{f\pn{x}}$, for any $\GeneralizationSampleName = \set{x_i}_{i = 1}^\GeneralizationSampleSize \subset \GeneralizationDomain$,
$$
\begin{aligned}
\GeneralizationEmpiricalRademacherComplexity\pn{\GeneralizationFunctionsLearners} \leq \inf_{\alpha > 0} \pn*{\frac{4\alpha}{\sqrt{\GeneralizationSampleSize}} + \frac{12}{\GeneralizationSampleSize}\int_{\alpha}^{2\GeneralizationBartlettBound\sqrt{\GeneralizationSampleSize}} \sqrt{\ln\pn{\GeneralizationCoveringNumber\pn{\GeneralizationFunctionsLearners, r, \UtilNormWithPlaceholder_\infty}}} \wrt{r}}
\end{aligned}.
$$
\end{lemma}


\subsection{Behavior cloning with fast-forwarding}
The space of training data for \PupilTDAGNN{} -- using IL with DAgger -- is the set of tuples $\pn{\gDAggerDatasetDatum}$, and the probability distribution on that space is induced by how DAgger generates these tuples. However, Definition~\ref{def:generalization-gap} requires that the 
training samples are
independent, but the tuples generated by DAgger are not for two reasons.
%
%
First, each tuple is computed using the previously generated tuple. At the beginning of each epoch, flocking is simulated for $\DAggerSimulationTimeStepsCount$ time steps where $\pn{\gExpertInput[\pn{t_{\DAggerSimulationTimeStepIndex + 1}}]}$ is computed by applying an acceleration to $\pn{\gExpertInput[\pn{t_\DAggerSimulationTimeStepIndex}]}$. In addition, $\PupilHistoryMatrix_\FlockAgentIndex\pn{t_\DAggerSimulationTimeStepIndex}$ is computed using $\set{\pn{\gExpertInput[\pn{t_{\DAggerSimulationTimeStepIndex - \PupilHistoryIndex + 1}}]}}_{\PupilHistoryIndex = 1}^\PupilHistoryCount$. The second reason is that the ML controller is trained on previously generated tuples, and it can influence what tuples are generated next by applying its acceleration.

The second reason cannot be addressed without fundamentally changing DAgger, and therefore, the probability distribution it induces on the space of training data. A primary motivation for the DAgger algorithm is allowing the ML controller to influence what training examples (e.g., tuples) are generated, and since the ML controller's weights depend on previously generated training examples, the training examples generated next cannot be independent. Therefore, we cannot analyze the generalization gap of ML controllers trained using IL with DAgger. To move forward with our analysis, we substitute DAgger with a variation of behavior cloning we call {\bf fast-forward behavior cloning} (\GeneralizationFastForwardBC{}).

\GeneralizationFastForwardBC{} addresses the second reason by only using the expert controller's acceleration during the flocking simulations run before each epoch.
The first reason is addressed by ensuring every tuple saved is derived from a newly sampled flock initial condition using fast-forwarding.
Instead of sampling an initial condition once at the beginning of each simulation, we sample an initial condition at each time step of the simulation. At time $t_\DAggerSimulationTimeStepIndex$, we sample $\pn{\gExpertInput[\pn{0}]}$, and then use the expert controller to advance it to $\pn{\gExpertInput[\pn{t_\DAggerSimulationTimeStepIndex}]}$ by sequentially applying the accelerations $\set{\gExpertName\pn{\gExpertInput[\pn{t_{\DAggerSimulationTimeStepIndex'}}]}}_{\DAggerSimulationTimeStepIndex' = 0}^{\DAggerSimulationTimeStepIndex - 1}$. Once the last acceleration is applied, we have also computed $\set{\PupilHistoryMatrix_\FlockAgentIndex\pn{t_\DAggerSimulationTimeStepIndex}}_{\FlockAgentIndex = 1}^\FlockAgentCount$. Finally, we save the tuple $\pn{\gDAggerDatasetDatum[\pn{t_\DAggerSimulationTimeStepIndex}]}$ to the training set.

\GeneralizationFastForwardBC{} is independent of the ML controller's weights, so the dataset of tuples may be computed before training by running $\DAggerEpochCount$ flocking simulations. Note that the dataset is dependent on whether the ML controller uses sum aggregation or mean aggregation (see \secref{sec:methods:improving-tdagnn}) because the tuples contain $\PupilHistoryMatrix_\FlockAgentIndex$. The datasets are split into training and test sets, and each ML controller is trained on its respective training set. We train the ML controller on the entire training set each epoch so that the training set size does not vary as the ML controller trains, allowing us to compare generalization gaps at different epochs.

\subsection{\PupilTDAGNN{} reinterpretations and loss function}
The bound proved in \cite{pmlr-v235-karczewski24a} is for EGNN \cite{satorras2021n} 
applied to graph-level tasks (e.g., graph classification). For these tasks, all \textit{invariant} node features after the last EGNN layer are aggregated, 
and the aggregated feature is fed to a final scoring model (e.g., an MLP) to produce a label for the graph. In flocking, each agent attempts to compute an acceleration for itself to match that of an expert controller, 
so the ML controller is performing a node-level task. 
We need to convert the node-level task into a graph-level task to adapt the analysis in \cite{pmlr-v235-karczewski24a}. 
We normally train ML controllers using squared error (SE) between their acceleration and the expert's acceleration averaged over all agents, resulting in a mean SE (MSE) loss. To convert to a graph-level task, we modify the output of each agent to be the 
SE instead of its acceleration, then define the scoring model $\PupilGraphScorer$ for the flock as
\begin{equation}
\label{eq:generalization-gap:scoring-model}
\begin{aligned}
    \PupilGraphScorer\pn{\mathrm{MSE}} = \PupilGraphScorerBias^2 + \mathrm{MSE} 
\end{aligned},
\end{equation}
where $\mathrm{MSE} = \frac{1}{\FlockAgentCount} \sum_{\FlockAgentIndex = 1}^\FlockAgentCount \mathrm{SE}_\FlockAgentIndex$, $\mathrm{SE}_\FlockAgentIndex =  \norm{\gExpertName\pn{\gExpertInput} - \PupilName\pn{\PupilHistoryMatrix_\FlockAgentIndex}}^2$,
and $\PupilGraphScorerBias$ is a trainable parameter in the scoring model. Note that $\mathrm{SE}_\FlockAgentIndex$ is \OrthogonalGroup{2} invariant, so MSE is an aggregation of invariant node features. The scoring model now computes the original loss given by \eqref{eq:results:loss-function}, so we can simply choose the identity function $\LossFunctionName\pn{y} = y$ as the loss function. However, to comply with the assumptions of Theorem~\ref{theorem:erc-bounds-generalization-gap}, we instead choose the loss function $\LossFunctionName\pn{y} = \min\set{1, y/\LossFunctionMSENormalizationConstant}$ for some $\LossFunctionMSENormalizationConstant > 0$.

\subsection{Bounding the generalization gap}
Paper \cite{pmlr-v235-karczewski24a} uses MLPs to update the node features of EGNN, so our adaptation of their proof considers the MLP representations of $\PupilName$ given by \Lemmaref{lemma:results:conv1d-as-linear} or \ref{lemma:results:eqconv-as-linear}. The authors also assume that the input to EGNN is bounded, so we also assume a bound on the input of the MLP representation in Assumption~\ref{assumption:results:dagger-dataset-non-equivariant-bounded}.

\begin{definition}[Input of \PupilTDAGNN{} as an MLP]\label{def:results:input-for-mlp-non-equivariant}
The first convolutional layer $\ConvOneD_1: \bR^{\gConvChannelsIn[1] \times \gConvFeaturesIn[1]} \to \bR^{\gConvChannelsIn[2] \times \gConvFeaturesIn[2]}$ of $\PupilName$ has $\gConvChannelsIn[1]$ input channels that are row vectors of the form $\gConvInput[1] [\ConvChannelInIndex] = \spn{
\pn{\PupilHistoryMatrix_\FlockAgentIndex}_{\ConvChannelInIndex,1},\; \dots,\; \pn{\PupilHistoryMatrix_\FlockAgentIndex}_{\ConvChannelInIndex,\gConvFeaturesIn[1]}
}$. 
    The input $\gConvInputAsMLP[1]$ of $\PupilName$ as an MLP is the row-wise concatenation of the input channels 
    and $\vone_{\gConvFeaturesIn[2]}^\top$ to account for the bias term of $\ConvOneD_1$: $\gConvInputAsMLP[1]=\spn{\gConvInput[1] [1],\; 
            \dots,\; 
            \gConvInput[1] [\gConvChannelsIn],\; 
            \vone_{\gConvFeaturesIn[2]}^\top}$.
\end{definition}

\begin{definition}[Input of \OrthogonalGroup{2} equivariant \PupilTDAGNN{} as an MLP]\label{def:results:input-for-mlp-equivariant}
$\EqConv_1: \bR^{\gConvChannelsIn[1] \times \gConvFeaturesIn[1]} \to \bR^{\gConvChannelsIn[2] \times \gConvFeaturesIn[2]}$ of $\PupilEquivariantName$ has $\gConvChannelsIn[1]/2$ input channels that are two-row matrices of the form $\gConvInput[1] [\ConvChannelInIndex] = \spn{
            \EqConvChannel_{\ConvChannelInIndex,1}
            ,\; \dots
            ,\; \EqConvChannel_{\ConvChannelInIndex,\gConvFeaturesIn[1]}
}$.
The input $\gConvInputAsMLP[1]$ of $\PupilName$ as an MLP is the row-wise concatenation of the input channels:
$
\gConvInputAsMLP[1]  = \spn{
    \gConvInput[1] [1]
    ,\; \dots
    ,\; \gConvInput[1] [\gConvChannelsIn[1]/2]
}
$.
\end{definition}

\begin{assumption}[\GeneralizationFastForwardBC{} datasets are bounded] 
\label{assumption:results:dagger-dataset-non-equivariant-bounded}
There exists $\DAggerDatumBound \geq 1$ such that for all tuples $\pn{\gDAggerDatasetDatum}$ in the \GeneralizationFastForwardBC{} dataset,
$$
\begin{aligned}
        \max_\FlockAgentIndex\set{
            \norm{\gExpertName\pn{\gExpertInput}},
            \norm{
                \gConvInputAsMLP[1]
            }_F
        } \leq \DAggerDatumBound,
\end{aligned}
$$
where \gExpertName{} is the expert controller and $\gConvInputAsMLP[1]$ is defined in either Definition~\ref{def:results:input-for-mlp-non-equivariant} or Definition~\ref{def:results:input-for-mlp-equivariant}.
\end{assumption}

Now we state the generalization gap bound for the ML controllers with proof in the appendix.

\begin{proposition}[Generalization bound of TDAGNN]
Let $P$ be the probability distribution over tuples $\pn{\gDAggerDatasetDatum}$ induced by \GeneralizationFastForwardBC{}. Let 
$\cL\pn{y} = \min\set{1, y/\LossFunctionMSENormalizationConstant}$ for $\LossFunctionMSENormalizationConstant > 0$ be the loss function. Let $\set{\PupilWeightMatrix_\PupilLayerIndex}_{\PupilLayerIndex = 1}^{\PupilLayerCount}$ be the weights of the MLP representation $\PupilAsMLPName$ of $\PupilName$ given by \Lemmaref{lemma:results:conv1d-as-linear} or \ref{lemma:results:eqconv-as-linear}, and let $\PupilGraphScorerBias$ of $\PupilGraphScorer$ in \eqref{eq:generalization-gap:scoring-model} be such that $\PupilGraphScorerBias \in \spn{0, \sqrt{\LossFunctionMSENormalizationConstant}}$. For any $\delta > 0$, with probability at least $1 - \delta$ over choosing a batch $\GeneralizationSampleName$ of $\GeneralizationSampleSize$ tuples sampled from $P$, the following bound holds:
\begin{equation}\label{eq:generalization-gap:bound}
\begin{aligned}
    &\ \gGeneralizationGap\pn{\PupilGraphScorer} \leq \frac{8}{\GeneralizationSampleSize} +\\
    &\ 
    \frac{48\PupilWeightMatrixLargestDim}{\sqrt{\GeneralizationSampleSize}}\sqrt{\pn{3\PupilLayerCount + 1}\ln\pn{10\PupilLayerCount\DAggerDatumBound\PupilActivationLipshitzConstant^{\PupilLayerCount}\sqrt{\PupilWeightMatrixLargestDim\GeneralizationSampleSize\LossFunctionMSENormalizationConstant}} + \pn{2\PupilLayerCount + 3}\sum_{\PupilLayerIndex = 1}^{\PupilLayerCount} \ln\pn{\max\set{1,\norm{\PupilWeightMatrix_\PupilLayerIndex}_F}}} + \\
    &\ 3\sqrt{\frac{\ln\pn{\frac{2}{\delta}}}{2\GeneralizationSampleSize}}.
\end{aligned}
\end{equation}
\end{proposition}

\section{Experiments}\label{sec:experiments}
{\bf Controllers and hyperparameters:} We compare four ML controllers and the expert controller used to train them. The expert controller is \ExpertTanner{} from \eqref{eq:tanner:centralized} with the potential function from \eqref{eq:tanner:centralized:tolstaya}. In \cite{tolstayaLearningDecentralizedControllers2017b}, \ExpertTanner{}'s acceleration $\gExpertName\pn{\gExpertInput}$ is clamped so that $\norm{\gExpertName\pn{\gExpertInput}}_\infty \leq \ExpertAccelerationClampConstant$, but this breaks \ExpertTanner{}'s \OrthogonalGroup{2} equivariance. To retain \OrthogonalGroup{2} equivariance, we enforce $\norm{\gExpertName\pn{\gExpertInput}}_2 \leq \ExpertAccelerationClampConstant$.

\begin{table}[!ht]
    \centering\small
    \begin{tabular}{lrrr}
        \toprule
         \textbf{Controller} & \textbf{
         Aggregation} & \textbf{Activations (after conv. $\PupilLayerIndex$)} & \textbf{Conv. layers} \\
         \midrule
         \PupilTDAGNN{} & Sum & $\PupilActivationName_\PupilLayerIndex = \tanh$, $\PupilLayerIndex \in \set{1,\dots,\PupilLayerCount - 1}$ & $\ConvOneD$ \\
         \midrule
         \PupilTDAGNNTF{} & Sum & $\PupilActivationName_1 = \tanh$ & $\ConvOneD$ \\
         \midrule
         \PupilTDAGNNTFMu{} & Mean & $\PupilActivationName_1 = \tanh$ & $\ConvOneD$ \\
         \midrule
         \PupilETDAGNN{} & Mean & $\PupilActivationEquivariantName_0 = \PupilETDAGNNActivationScaleLog$, $\PupilActivationEquivariantName_1 = \PupilETDAGNNActivationScaleTanh$ & $\EqConv$ \\
         \bottomrule
    \end{tabular}
\caption{
Summary of the architectural differences between the compared ML controllers for flocking.
The subscripts $\PupilLayerIndex$ of the activations in the Activations column indicate the convolutional layer that the activation is applied after. If the subscript $\PupilLayerIndex$ is 0, then the activation is applied before the first convolutional layer. If a subscript value is not listed, then no (or the identity) activation is applied after the corresponding convolutional layer.}
    \label{tab:results:models}
\end{table}

The first controller is \PupilTDAGNN{} from \cite{tolstayaLearningDecentralizedControllers2017b} and serves as a baseline. The second is \PupilTDAGNN{} with the  ``Activate once'' improvement described in \secref{sec:methods:improving-tdagnn} and is denoted by \PupilTDAGNNTF{} -- ``TF'' stands for ``tanh first.'' The third is \PupilTDAGNN{} with both the improvements described in \secref{sec:methods:improving-tdagnn} and is denoted by \PupilTDAGNNTFMu{}. The last is \PupilTDAGNNTFMu{} with the CNN \PupilEquivariantName{} described in \secref{sec:methods:equiv-model} and the equivariant activations proposed in \secref{sec:methods:equivariant-activations}, and it is denoted by \PupilETDAGNN{}. The architectural differences, including what activations are used after each convolutional layer, are summarized in Table~\ref{tab:results:models}.
All ML controllers use $\PupilLayerCount = 3$ convolutional layers mapping $\bR^{\gConvChannelsIn \times \gConvFeaturesIn}$ to $\bR^{\gConvChannelsIn[\PupilLayerIndex+1] \times \gConvFeaturesIn[\PupilLayerIndex+1]}$ for $\PupilLayerIndex \in \set{1,\dots,\PupilLayerCount}$. 
Explictly writing their domains and ranges, $\pn{\gConvChannelsIn[1], \gConvFeaturesIn[1]} = \pn{6, \PupilHistoryCount}$, $\pn{\gConvChannelsIn[2], \gConvFeaturesIn[2]} = \pn{32, 1}$, $\pn{\gConvChannelsIn[3], \gConvFeaturesIn[3]} = \pn{32, 1}$, and $\pn{\gConvChannelsIn[4], \gConvFeaturesIn[4]} = \pn{2, 1}$. 
With these hyperparameters, Table \ref{tab:results:num-model-weights} shows each ML controller's trainable parameter count.

\begin{table}[!ht]
    \centering\small
    \begin{tabular}{lrrrr}
        \toprule
        & \textbf{\textit{\PupilTDAGNN{}}} & \textbf{\textit{\PupilTDAGNNTF{}}} & \textbf{\textit{\PupilTDAGNNTFMu{}}} & \textbf{\textit{\PupilETDAGNN{}}} \\
        \midrule
        \textbf{\#Weights} & 1,730 & 1,730 & 1,730 & 416 \\
        \bottomrule
    \end{tabular}
\caption{
Number of trainable weights of the ML controllers for decentralized flocking.}
\label{tab:results:num-model-weights}
\end{table}

\noindent{\bf Training:} All ML controllers are trained for flocking control using IL with DAgger for $\DAggerEpochCount = 400$ epochs and then tested in flocking, leader following, and obstacle avoidance scenarios. 
We initialize $\DAggerActionPropability_0 = \DAggerActionPropabilityInit=0.993$, and set $\DAggerActionPropability_\DAggerEpochIndex = \max\set{\DAggerActionPropability_{\DAggerEpochIndex - 1}\DAggerActionPropabilityInit, 0.5}$.
Tuples are generated and added to the training set by running flocking simulations for $\DAggerSimulationTimeStepsCount = 2/\TimeStep$ time steps with $\TimeStep = 10^{-2}$. The initial condition for each simulation is sampled from a dataset of initial conditions. The training set of tuples is capped at 10,000 examples, and the oldest examples are discarded first after reaching the cap. After the flocking simulation of each epoch, we sample 200 batches of 20 tuples with replacement from the training set. For each batch, we compute the loss averaged over the batch and update the ML controller's weights.
The loss function for one tuple is
\begin{equation}
\label{eq:results:loss-function} 
\begin{aligned}
\LossFunctionName\pn{\gDAggerDatasetDatum} = \frac{1}{\FlockAgentCount} \sum_{\FlockAgentIndex = 1}^\FlockAgentCount \norm{\gExpertName\pn{\gExpertInput} - \PupilName\pn{\PupilHistoryMatrix_\FlockAgentIndex}}^2
\end{aligned}.
\end{equation}
The trainable parameters are initialized using Xavier uniform initialization with gain 1, and optimized using the Adam optimizer with learning rate $5 \times 10^{-5}$, $\beta_1 = 0.9$, and $\beta_2 = 0.999$.

\noindent{\bf Initial conditions:}
The dataset of flock initial conditions is randomly generated following the procedure described in \cite{tolstayaLearningDecentralizedControllers2017b}. We refer to this dataset as the \textit{RandomDisk} dataset. Each initial condition composes $\FlockAgentCount$ agents whose positions are distributed in a 2D disk of radius $\sqrt{\FlockAgentCount}$. Having radius $\sqrt{\FlockAgentCount}$ implies that the ratio of the number of agents to the disk's area is the constant $\pi$. The agents are placed in the disk uniformly randomly such that three conditions are met: the agents are not too close (for $\FlockAgentIndexTwo \in \FlockAgentNeighborhood_\FlockAgentIndex\pn{0}$,  $\FlockAgentMinDist \leq \norm{\FlockAgentRelativePos} \leq \FlockAgentCommRadius$); 
the agents have enough neighbors ($\abs{\FlockAgentNeighborhood_\FlockAgentIndex\pn{0}} \geq \FlockCommunicationGraphMinDegree \geq 0$); and, the flock's communication graph is connected. The agents' velocities are initialized to $\FlockVelMatrix\pn{0} = \FlockAgentVelocityInitial + \FlockAgentVelocityBias\vone_\FlockAgentCount^\top$ where the entries of $\FlockAgentVelocityInitial$ and $\FlockAgentVelocityBias$ are uniformly randomly sampled from $\spn{-\FlockAgentVelocityMaxMaxNorm, \FlockAgentVelocityMaxMaxNorm}$ for $\FlockAgentVelocityMaxMaxNorm \in \hlpn{0, \infty}$. \Figref{fig:dataset:randomdisk} shows example initial conditions of this dataset. Following \cite{tolstayaLearningDecentralizedControllers2017b}, 
we choose the RandomDisk dataset parameters as $\FlockAgentCount = 100$, $\FlockAgentMinDist = 0.1$, $\FlockAgentCommRadius = 1$, $\FlockCommunicationGraphMinDegree = 2$, and $\FlockAgentVelocityMaxMaxNorm = 3$.

\begin{figure}[!ht]
    \centering
    \begin{subfigure}[t]{.29\textwidth}
        \centering
        \includegraphics[width=\textwidth]{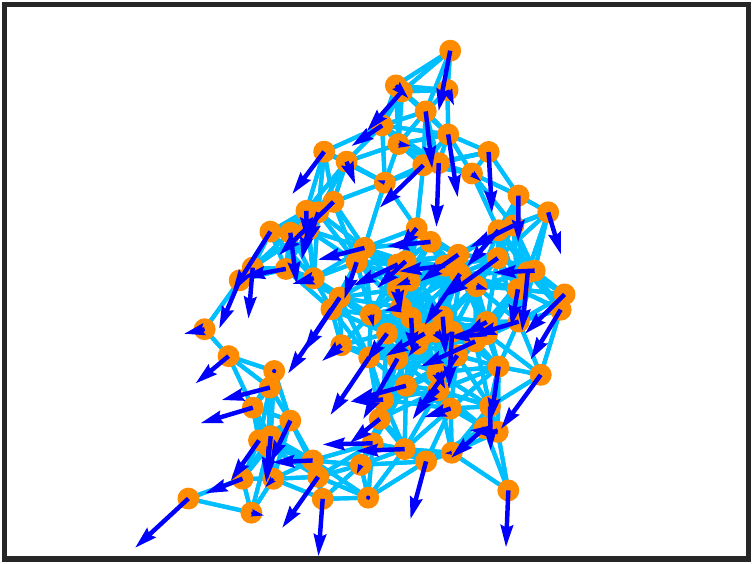}
        \label{fig:dataset:randomdisk:0}
    \end{subfigure}
    \hspace{1em}
    \begin{subfigure}[t]{.29\textwidth}
        \centering
        \includegraphics[width=\textwidth]{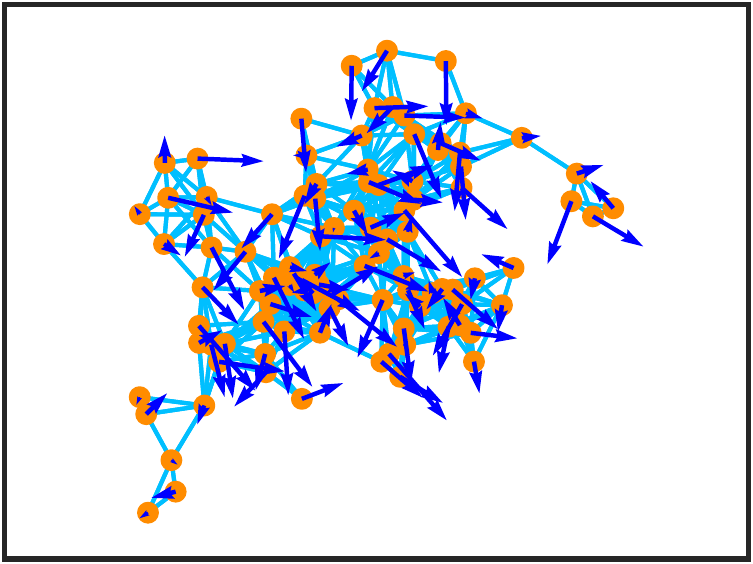}
        \label{fig:dataset:randomdisk:1}
    \end{subfigure}
    \hspace{1em}
    \begin{subfigure}[t]{.29\textwidth}
        \centering
        \includegraphics[width=\textwidth]{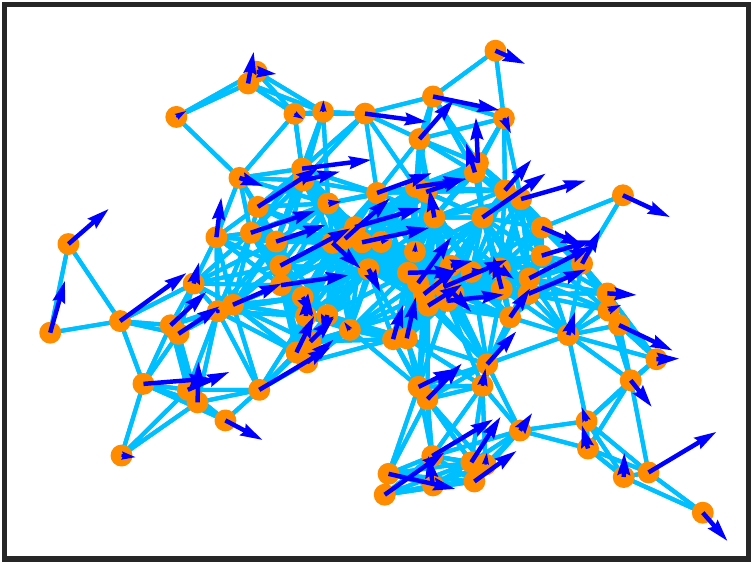}
        \label{fig:dataset:randomdisk:2}
    \end{subfigure}
\caption{
Examples from the RandomDisk dataset of flock initial conditions. There are $\FlockAgentCount = 100$ agents (orange dots) with at least $\deg_{\min} = 2$ neighbors (indicated by light blue edges connecting them). The distance between an agent and its neighbors is between $\FlockAgentMinDist = 0.1$ and $\FlockAgentCommRadius = 1$. The agents' velocities (dark blue arrows) have magnitudes no larger than $2\FlockAgentVelocityMaxMaxNorm = 6$. 
}
\label{fig:dataset:randomdisk}
\end{figure}

\noindent{\bf Metrics:} We quantify the performance of the ML controllers using the two metrics:
\begin{itemize}
\item \textbf{Velocity variance:} The variance of velocities is $\var\pn{\FlockAgentVel} = \frac{1}{\FlockAgentCount}\sum_{\FlockAgentIndex = 1}^\FlockAgentCount \norm{\FlockAgentVel_\FlockAgentIndex - \mean\pn{\FlockAgentVel}}^2$ where $\mean\pn{\FlockAgentVel} := \frac{1}{\FlockAgentCount}\sum_{\FlockAgentIndex = 1}^\FlockAgentCount \FlockAgentVel_\FlockAgentIndex$. 
Lower variance implies the flock is closer to a limiting velocity of asymptotic flocking.

\item \textbf{Mean acceleration norm:} The mean acceleration norm of the agents is $\frac{1}{\FlockAgentCount} \sum_{\FlockAgentIndex = 1}^\FlockAgentCount \norm{\FlockAgentAccel_\FlockAgentIndex\pn{t_\DAggerSimulationTimeStepIndex}}$. 
A lower mean acceleration norm means the controller uses a smaller control input to achieve its objective. This metric measures the controller's efficiency because less acceleration implies less energy expenditure for the flock.
\end{itemize}

\subsection{Flocking}\label{sec:results:flocking}
\newcommand{\RiemannSumVelVar}{IVV}
\newcommand{\RiemannSumMeanAccelNorm}{IMAN}

In this task, we observe significant performance gaps
between \ExpertTanner{}, the ML controllers using sum aggregation (\PupilTDAGNN{} and \PupilTDAGNNTFMu{}), and the ML controllers using mean aggregation (\PupilTDAGNNTFMu{} and \PupilETDAGNN{}).
When presenting these results, we will compare these groups, and then compare the controllers within these groups when necessary. Keep in mind that \PupilETDAGNN{} has 75\% fewer trainable weights than other controllers (see Table~\ref{tab:results:num-model-weights}).

For training, \figref{fig:results:flocking:randomdisk:by-epoch} shows the median Integral of the Velocity Variance (\RiemannSumVelVar{}) 
and the median Integral of the Mean Acceleration Norm (\RiemannSumMeanAccelNorm{}) on validation set of the RandomDisk dataset with $\FlockAgentCount = 100$ and $\TimeStep = 10^{-2}$. The mean-aggregation ML controllers achieve a lower median \RiemannSumVelVar{} and \RiemannSumMeanAccelNorm{} by epoch 80 than the sum-aggregation ML controllers do by epoch 400. Furthermore, at epoch 400, the \RiemannSumVelVar{} and \RiemannSumMeanAccelNorm{} IQRs 
of the best sum-aggregation ML controllers and worst mean-aggregation ML controllers do not overlap.

\begin{figure}[!ht]
    \centering
    \begin{subfigure}[T]{\textwidth}
        \centering
        \includegraphics[width=0.8\textwidth]{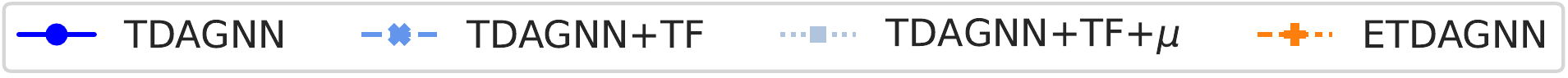}
    \end{subfigure} \\
    \begin{subfigure}[T]{.49\textwidth}
        \centering
        \includegraphics[width=.7\textwidth]{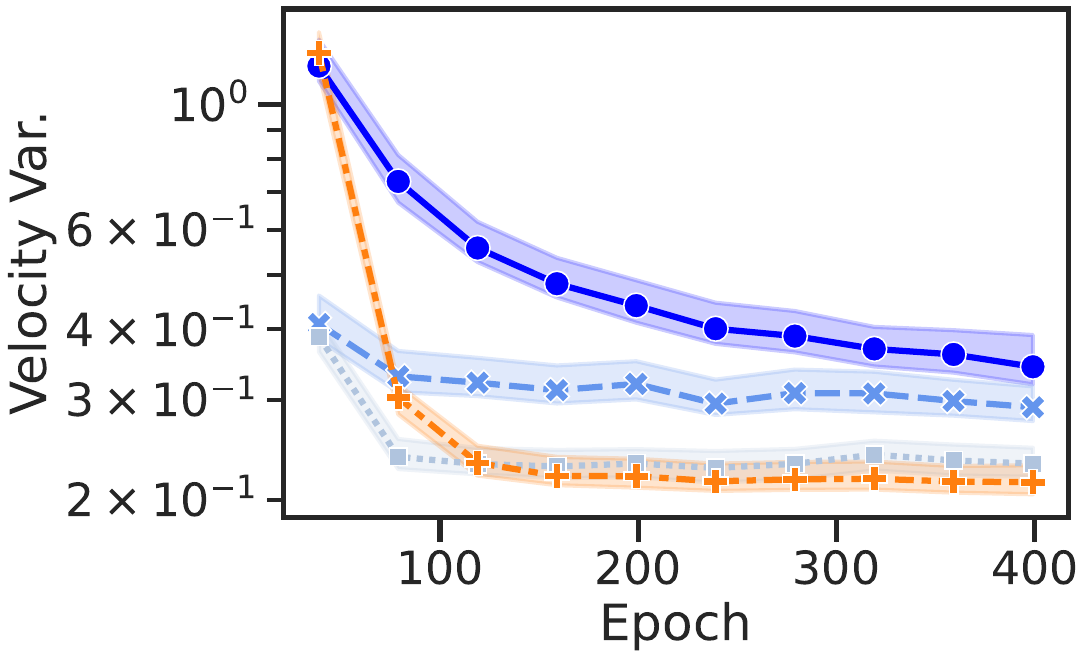}
    \end{subfigure}
    \begin{subfigure}[T]{.49\textwidth}
        \centering
        \includegraphics[width=.7\textwidth]{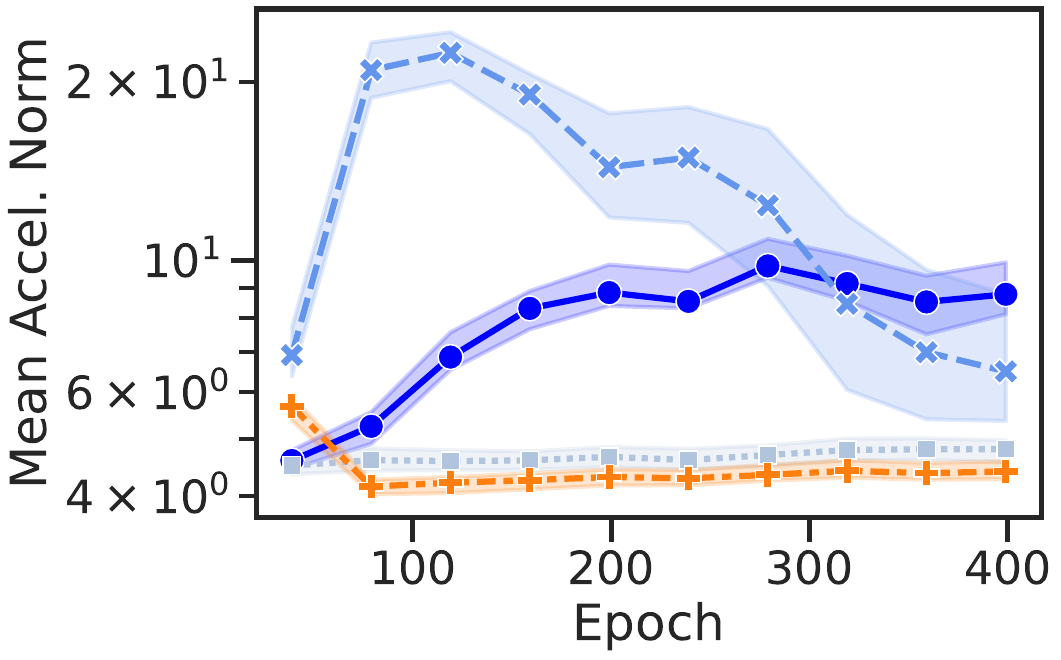}
    \end{subfigure}
\caption{
Performance of ML controllers in flocking as they train. They are evaluated on the RandomDisk validation set with $100$ agents. Each simulation is run for $\DAggerSimulationTimeStepsCount = 2/\TimeStep$ time steps where $\TimeStep = 10^{-2}$.
The curves show the median values of the respective metrics' integrals and the colored areas show their corresponding interquartile ranges.}
\label{fig:results:flocking:randomdisk:by-epoch}
\end{figure}

For the sum-aggregation ML controllers,
\PupilTDAGNNTF{} has a lower median \RiemannSumVelVar{} than \PupilTDAGNN{} for all epochs. 
\PupilTDAGNNTF{} has a higher \RiemannSumMeanAccelNorm{} than \PupilTDAGNN{} at epoch 40, and is over double compared to \PupilTDAGNN{} from epoch 80 to epoch 160. It takes until epoch 320 for it to become lower than \PupilTDAGNN{}.
For the mean-aggregation ML controllers, at epoch 40, \PupilTDAGNNTFMu{} has a median \RiemannSumVelVar{} less than half that of \PupilETDAGNN{}, but \PupilETDAGNN{} matches \PupilTDAGNNTFMu{} by epoch 120.
By epoch 400, the \RiemannSumVelVar{} of \PupilTDAGNNTFMu{} is only larger than \PupilETDAGNN{} by a few hundredths.
At epoch 40, \PupilTDAGNNTFMu{} also has a lower \RiemannSumMeanAccelNorm{} than \PupilETDAGNN{}, but by epoch 80 and for the rest of the epochs, the \RiemannSumVelVar{} of \PupilTDAGNNTFMu{} and \PupilETDAGNN{} remain within that gap, both increasing at the same rate.

After training, we test each ML controller's ability to achieve asymptotic flocking with separation on 50 RandomDisk initial conditions not used for training or hyperparameter tuning.
We use $\FlockAgentCount \in \set{50, 100, 200, 400}$ and $\TimeStep = 10^{-3}$.
\Figref{fig:results:flocking:randomdisk:vel-var} shows the median velocity variance over time.
For all $\FlockAgentCount$, the mean-aggregation controllers reduce the velocity variance from about 4 to below 0.2 faster than the sum-aggregation controllers.
When $\FlockAgentCount \leq 100$, the mean-aggregation ML controllers also reach a lower velocity variance at the last time step, but when $\FlockAgentCount \geq 200$, the sum-aggregation ML controllers reach a lower velocity variance.
Surprisingly, when $\FlockAgentCount = 50$, the velocity variances of \PupilTDAGNNTF{} and the mean-aggregation ML controllers are below that of \ExpertTanner{} for the majority of the simulation.
The exception is from times $t_\DAggerSimulationTimeStepIndex\TimeStep \in \spn{0.1, 0.3}$ where \PupilTDAGNNTF{} has a larger velocity variance than \ExpertTanner{}.
When $\FlockAgentCount = 100$, \PupilETDAGNN{} reaches a lower velocity variance than \ExpertTanner{} at the end of the simulation.
\Figref{fig:results:flocking:randomdisk:accel} shows the median mean acceleration norm over time.
The sum-aggregation ML controllers' median mean acceleration norm at the last time step matches or is a few hundredths smaller than the mean-aggregation ML controllers.

\begin{figure}[!ht]
    \centering
    \begin{subfigure}[T]{\textwidth}
        \centering
        \includegraphics[width=.95\textwidth]{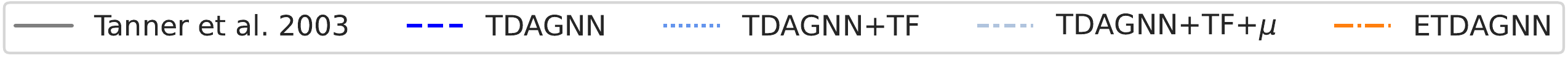}
    \end{subfigure} \\
    \begin{subfigure}[T]{.23\textwidth}
        \centering
        \hspace{-2.8em}
        \includegraphics[width=1.21\textwidth]{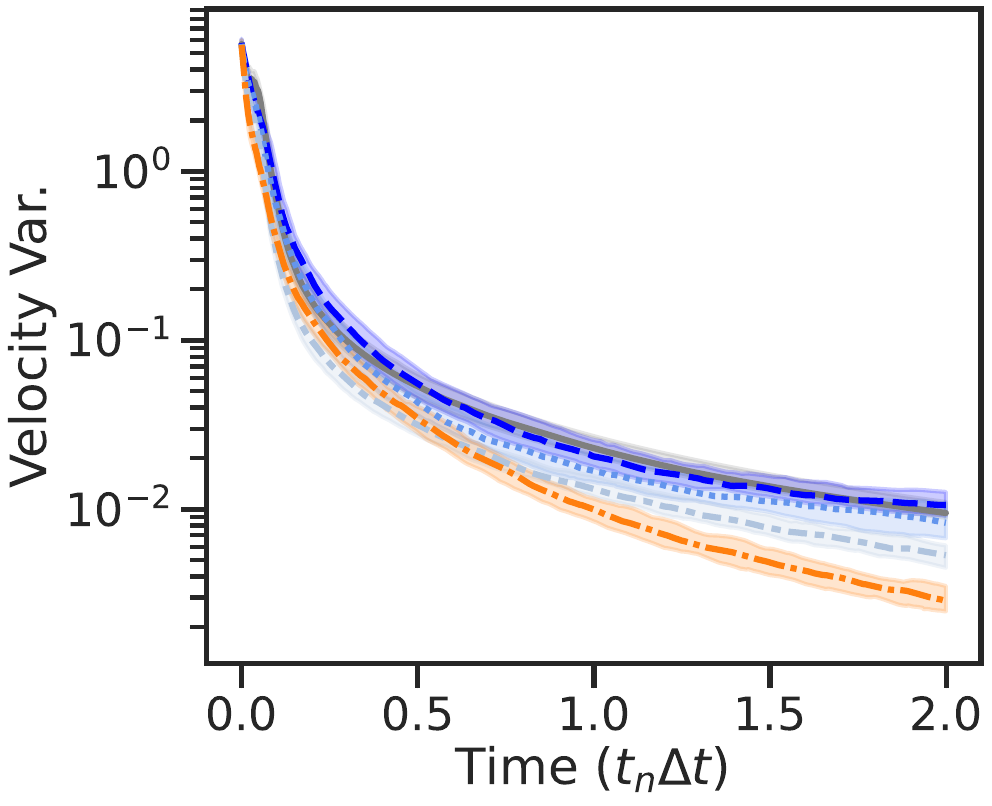}
    \end{subfigure}
    \begin{subfigure}[T]{.23\textwidth}
        \centering
        \includegraphics[width=\textwidth]{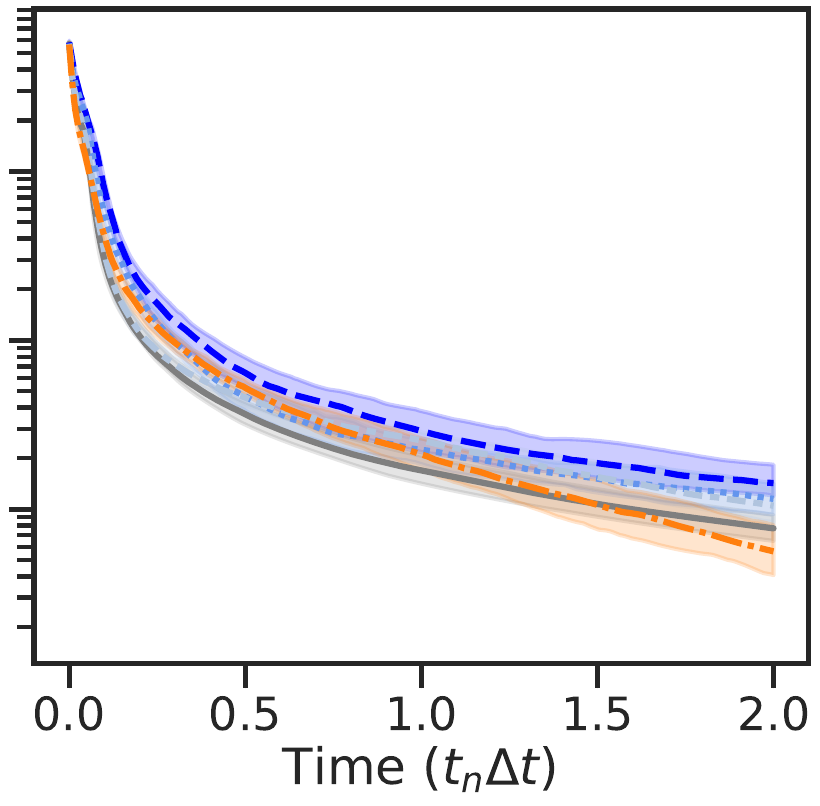}
    \end{subfigure}
    \begin{subfigure}[T]{.23\textwidth}
        \centering
        \includegraphics[width=\textwidth]{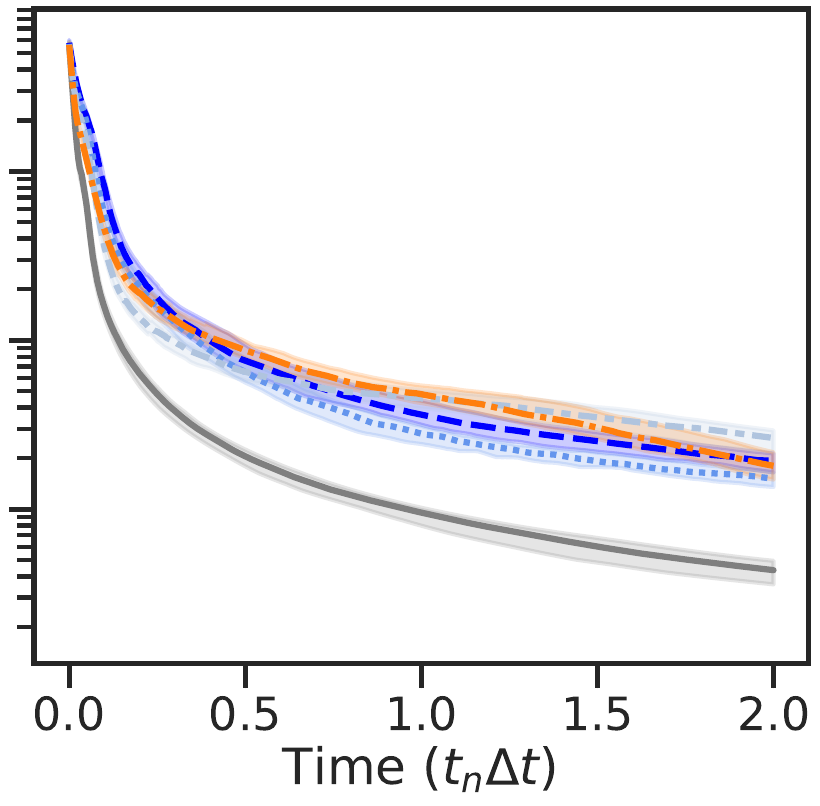}
    \end{subfigure}
    \begin{subfigure}[T]{.23\textwidth}
        \centering
        \includegraphics[width=\textwidth]{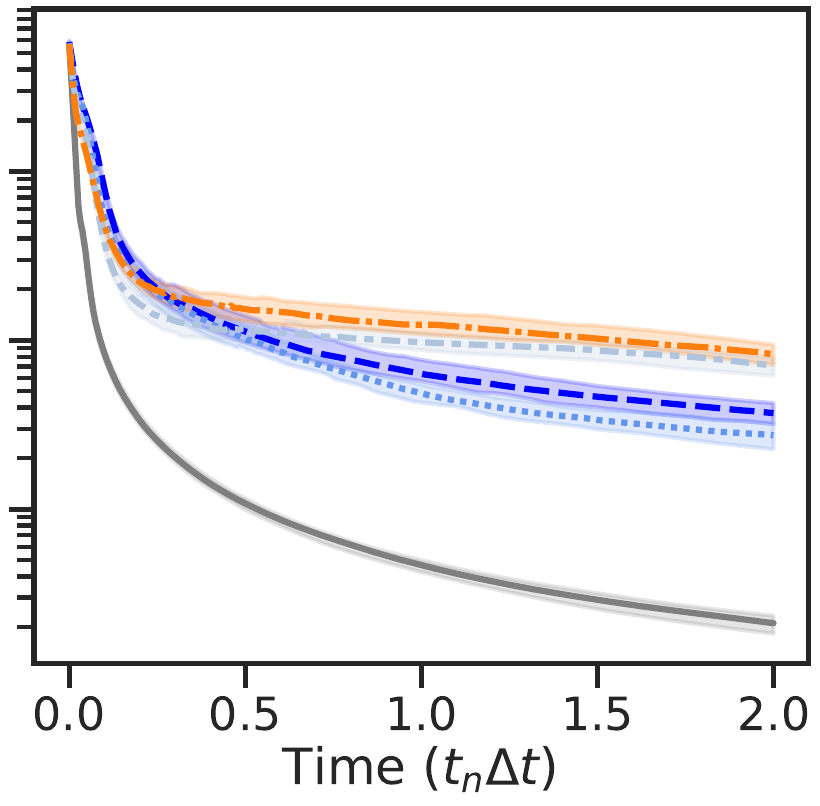}
    \end{subfigure}
\caption{
Velocity variance of the controllers in flocking over the simulation time with time step size $\TimeStep = 10^{-3}$.
They are evaluated on the RandomDisk test set with the number of agents $\FlockAgentCount \in \set{50, 100, 200, 400}$. The lines show the median metric values and the colored areas show the corresponding interquartile ranges.}
    \label{fig:results:flocking:randomdisk:vel-var}\vspace{-0.15cm}
\end{figure}

\begin{figure}[!ht]
    \centering
    \begin{subfigure}[T]{\textwidth}
        \centering
        \includegraphics[width=.95\textwidth]{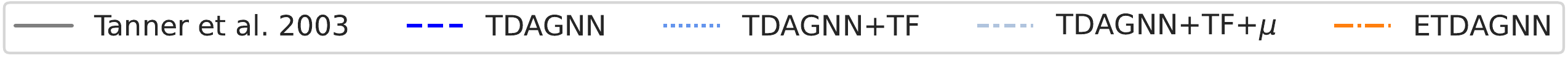}
    \end{subfigure} \\
    \begin{subfigure}[T]{.23\textwidth}
        \centering
        \hspace{-2.8em}
        \includegraphics[width=1.21\textwidth]{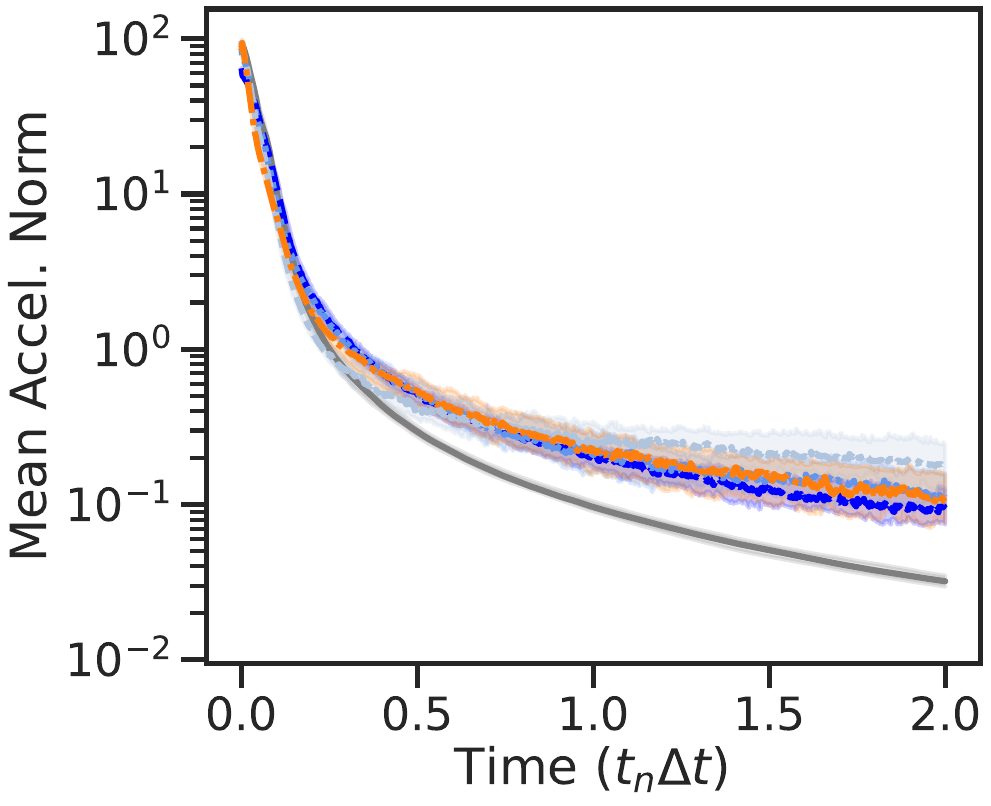}
    \end{subfigure}
    \begin{subfigure}[T]{.23\textwidth}
        \centering
        \includegraphics[width=\textwidth]{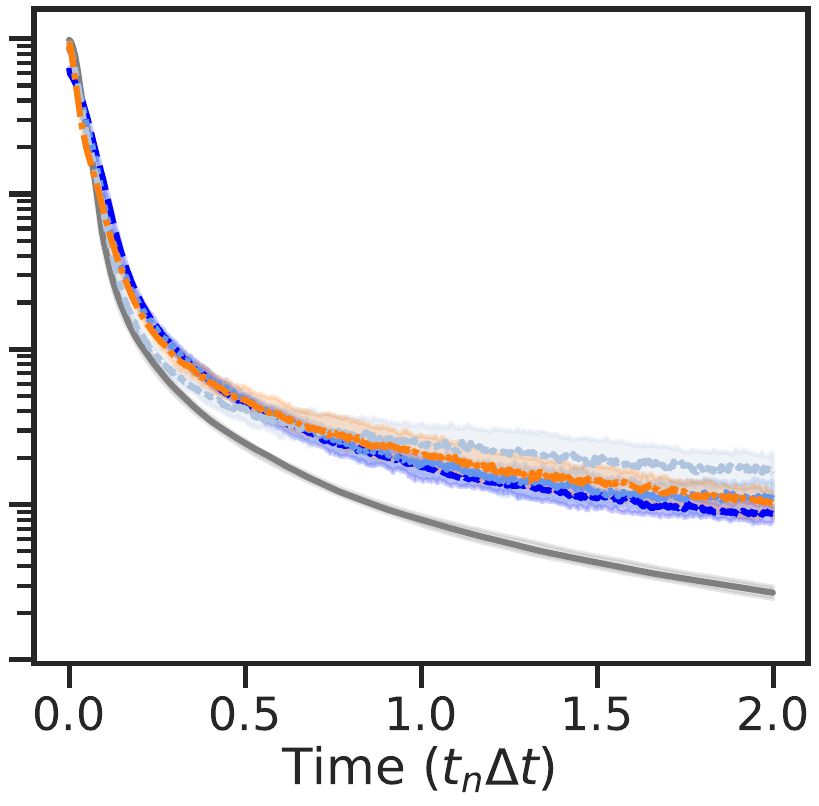}
    \end{subfigure}
    \begin{subfigure}[T]{.23\textwidth}
        \centering
        \includegraphics[width=\textwidth]{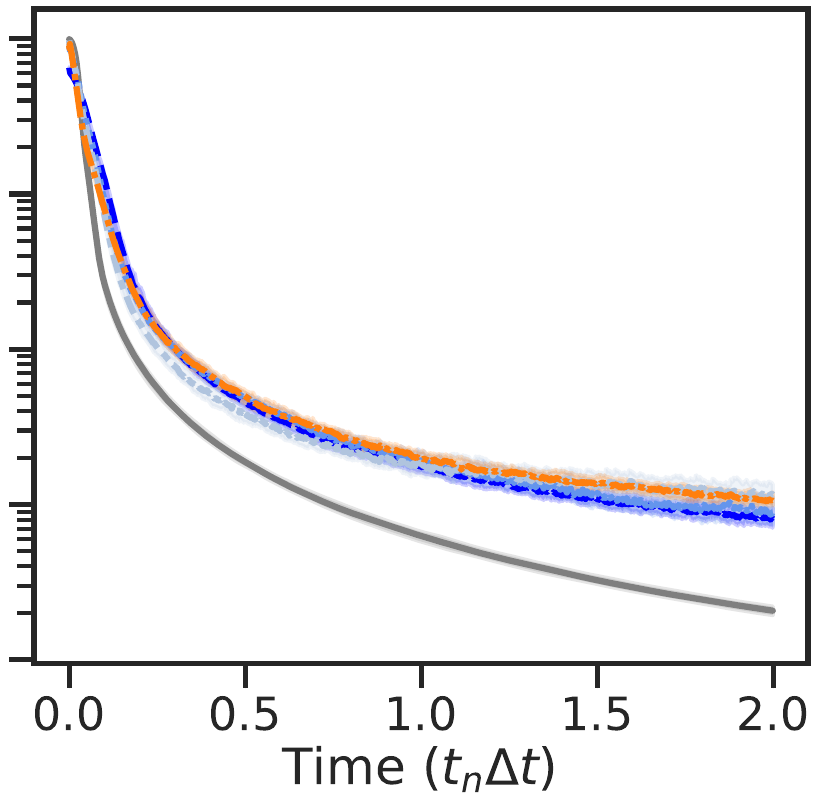}
    \end{subfigure}
    \begin{subfigure}[T]{.23\textwidth}
        \centering
        \includegraphics[width=\textwidth]{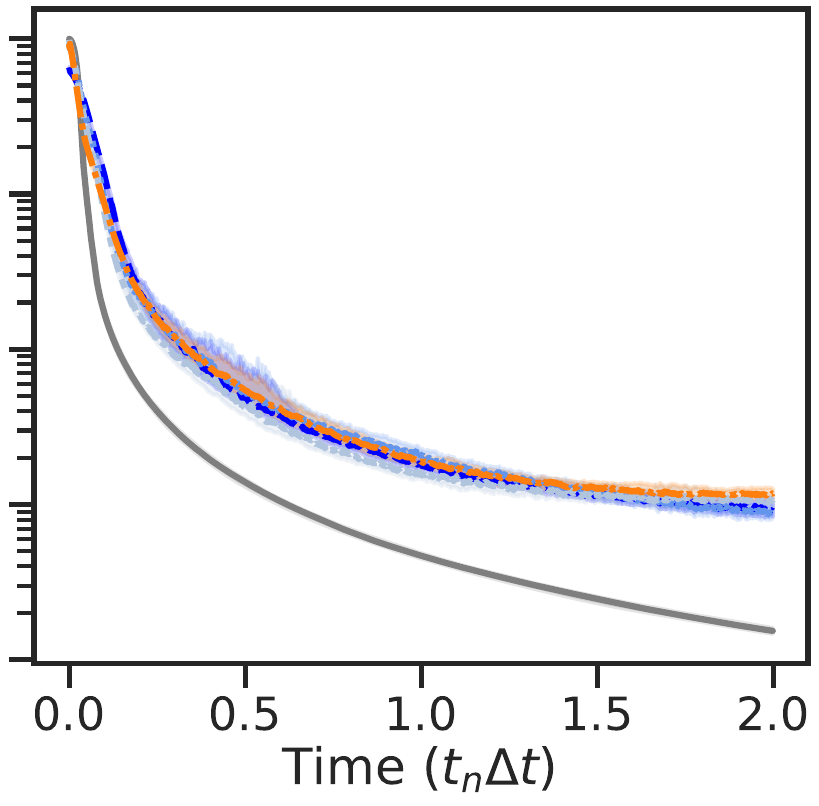}
    \end{subfigure}
\caption{
Mean acceleration norm of the controllers in flocking over the simulation time with time step size $\TimeStep$. They are evaluated on the RandomDisk test set with the number of agents $\FlockAgentCount \in \set{50, 100, 200, 400}$. The lines show the median metric values and the colored areas show the corresponding interquartile ranges.}
\label{fig:results:flocking:randomdisk:accel}\vspace{-0.15cm}
\end{figure}

In summary, there is no best-performing ML controller for all tested flock sizes. For all tests, the mean-aggregation ML controllers reduce the velocity variance of the flock faster than the sum-aggregation ones when the variance is large. 
The ML controller that reaches the smallest velocity variance by the last time step depends on $\FlockAgentCount$. When $\FlockAgentCount \leq 100$, \PupilETDAGNN{} reaches the lowest 
variance at the last time step 
than other ML controllers.
For larger $\FlockAgentCount$, \PupilTDAGNNTF{} reaches the smallest velocity variance. The ML controllers have approximately the same mean acceleration norm.
Animations of flocking are available on GitHub.\footnote{Flocking animations: \url{github.com/Utah-Math-Data-Science/Equivariant-Decentralized-Controllers/tree/main/misc/animations/flocking}}

\subsection{Leader following}\label{sec:results:leader-following}

{
\newcommand{\zAgentLeaderVelocity}{{\vv_{\mathrm{ldr}}}}

In leader following, the leader agents are selected from the flock and instructed to move along some predefined path (e.g., a line), and the 
other agents are followers. 
%
%
We test the ability of ML controllers, already trained for flocking, in conducting leader following using 50 RandomDisk initial conditions. 
Two agents from each initial condition are randomly selected to be leaders of the flock and the leaders' velocities are set equal. To prevent the followers from changing the leaders' trajectories, the leaders ignore all messages from the followers, i.e., the leaders only have directed edges from them to other agents. 
Consequently, the leaders do not pass on summaries of the messages they receive to their neighbors.
In addition, the ML controllers only control the followers, and since the ML controllers are trained to maintain communication graph connectivity, the controllers are compelled to have the followers match the velocity of the leaders.
The leader following simulations are run for $\DAggerSimulationTimeStepsCount = 3/\TimeStep$ time steps. In addition to the flocking validation metrics, leader following adds this validation metric:
\begin{itemize}
\item \textbf{Mean leader velocity distance (MLVD):} Let $\zAgentLeaderVelocity$ be the velocity of the leaders in the flock. 
We measure how close the flock is to the limiting velocity with $\frac{1}{\FlockAgentCount}\sum_{\FlockAgentIndex = 1}^\FlockAgentCount \norm{\FlockAgentVel_\FlockAgentIndex\pn{t_\DAggerSimulationTimeStepIndex} - \zAgentLeaderVelocity}$. 
\end{itemize}

\Figref{fig:results:leader-following:randomdisk:by-time} shows the median MLVD and median mean acceleration norm over time. By the last time step, \PupilTDAGNNTFMu{} has the lowest median MLVD. 
The medians of \PupilETDAGNN{} and \PupilTDAGNN{} are nearly double and triple that of \PupilTDAGNNTFMu{}.
The ML controllers have approximately the same mean acceleration norm.
%
%
Based on the performance of the ML controllers for the leader following task, we recommend \PupilTDAGNNTFMu{} for the best performance; however, \PupilETDAGNN{} provides comparable performance with 75\% fewer trainable parameters.
Animations of leader following are available on GitHub.\footnote{Leader following animations: \url{github.com/Utah-Math-Data-Science/Equivariant-Decentralized-Controllers/tree/main/misc/animations/leader_following}}

\begin{figure}[!ht]
    \centering
    \begin{subfigure}[T]{\textwidth}
        \centering
        \includegraphics[width=0.95\textwidth]{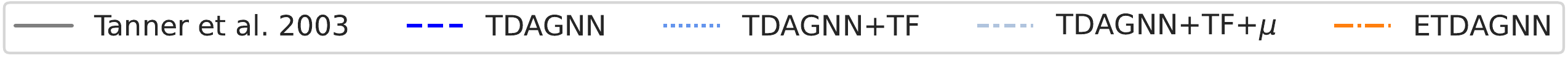}
    \end{subfigure} \\
    \begin{subfigure}[T]{.49\textwidth}
        \centering
        \includegraphics[width=.9\textwidth]{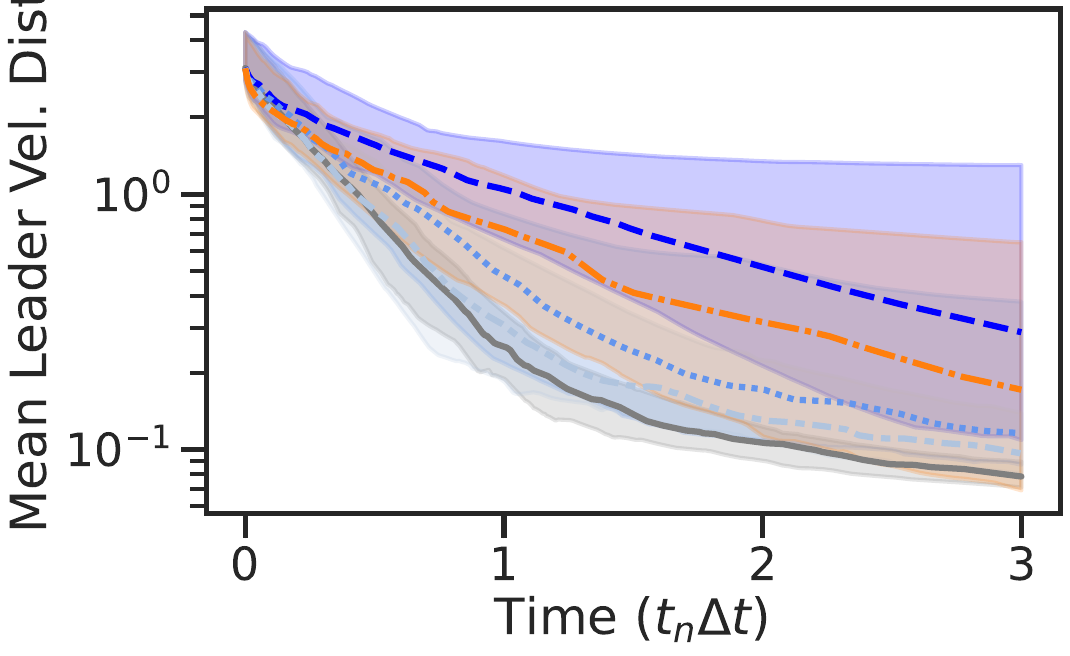}
    \end{subfigure}
    \begin{subfigure}[T]{.49\textwidth}
        \centering
        \includegraphics[width=.9\textwidth]{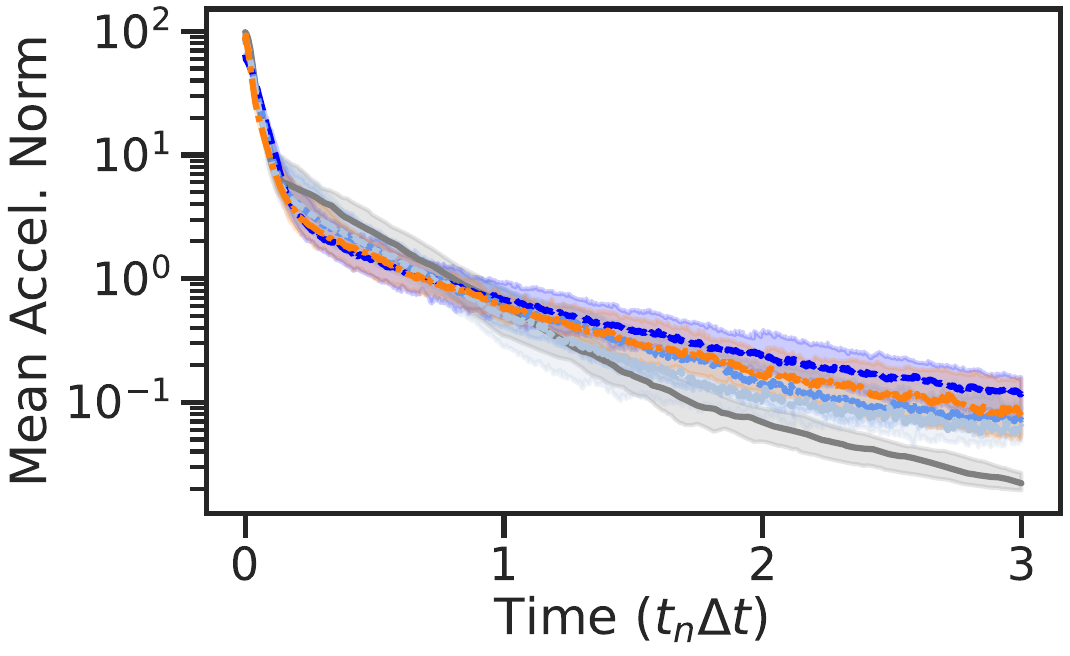}
    \end{subfigure}
\caption{
Performance of controllers in leader following at each time step of the simulation. They are evaluated on the RandomDisk test set with $\FlockAgentCount = 100$ agents where two agents are leaders. 
}
\label{fig:results:leader-following:randomdisk:by-time}\vspace{-0.15cm}
\end{figure}
}

\subsection{Obstacle avoidance}
When a flock is intercepting an obstacle, the ML controllers should have the flock circumnavigate it. For centralized flocking controllers, the primary requirement is that the agents do not collide with the obstacle. The centralized flocking controllers discussed can guarantee agent separation, which is readily extendable to obstacle collision avoidance if each agent can compute its distance from the obstacle. Decentralized flocking controllers, however, face the significant challenge of managing communication graph connectivity. An obstacle's diameter can be larger than the communication radius $\FlockAgentCommRadius$ of the agents, preventing agents on opposite sides of the obstacle from communicating. Decentralized flocking controllers have two options for successful circumnavigation -- ensure that all agents move around the obstacle in the same direction, or devise a scheme that will guarantee communication graph connectivity is restored if groups of agents go around in different directions. We present a technique for helping decentralized flocking controllers conduct obstacle avoidance inspired by Fig.~2 of \cite{reynoldsFlocksHerdsSchools1987}. The technique is tested using disk-shaped obstacles with diameters up to nearly half of the flock's diameter. 


{
\newcommand{\zFlockAgentExtremeIndex}{{\FlockAgentIndex_*}}
\newcommand{\zFlockAgentExtremeIndexTwo}{{\FlockAgentIndexTwo_*}}
\newcommand{\zObstacleSideCount}{{s}}

The obstacle avoidance dataset is built upon the RandomDisk dataset. For each initial condition, we construct a disk-shaped obstacle as follows. First, we place the center of a regular polygon with $s$ sides of length $\FlockAgentMinDist$ relative to the flock such that two conditions are met:
\begin{enumerate}
\item \textbf{Polygon is in the middle:} Let $\FlockAgentPos_\zFlockAgentExtremeIndex$ and $\FlockAgentPos_\zFlockAgentExtremeIndexTwo$ be agents whose distance is the diameter of the flock: $\norm{\vr_{\zFlockAgentExtremeIndex\zFlockAgentExtremeIndexTwo}} = \max_{\FlockAgentIndex,\FlockAgentIndexTwo} \norm{\FlockAgentRelativePos}$. The polygon's center is on the line passing through $\mean\set{\FlockAgentPos_\zFlockAgentExtremeIndex, \FlockAgentPos_\zFlockAgentExtremeIndexTwo}$ that is perpendicular to the line passing through $\FlockAgentPos_\zFlockAgentExtremeIndex$ and $\FlockAgentPos_\zFlockAgentExtremeIndexTwo$.

\item \textbf{Polygon is in front:} At time zero, the minimum distance between the polygon center and the flock agents is greater than the polygon's circumradius plus $\FlockAgentCommRadius$.
\end{enumerate}

The obstacle is the circumscribed circle of the polygon.
Next, the agents need a mechanism to compute their position relative to the obstacle. In our simulation, we utilize the existing communication mechanism
by placing additional \textit{obstacle agents} on the vertices of the polygon. The obstacle agents only send their position to the flock agents. Moreover, the obstacle agents do not receive any messages from the flock agents; that is, the obstacle agents only have directed edges from them to the flock agents in the communication graph.

The final step is to randomly select two agents in the flock to be leaders (see \secref{sec:results:leader-following}). Without leaders, the flock will not attempt to circumnavigate the obstacle; instead, the flock will turn around or simply halt in front of the obstacle. Leaders force the flock to continue past the obstacle. Leaders do not receive any messages from the obstacle agents. The leaders' velocity is fixed, so we select leaders that will always be at least $\RadiusMin$ away from the obstacle boundary. We set the initial velocity of leaders and followers to the unit vector pointing from $\mean\set{\FlockAgentPos_\zFlockAgentExtremeIndex, \FlockAgentPos_\zFlockAgentExtremeIndexTwo}$ to the center of the obstacle. Examples are shown in \Figref{fig:dataset:obstacle-avoidance:randomdisk}.
}

\begin{figure}[!ht]
    \centering
    \begin{subfigure}[T]{.3\textwidth}
        \centering
        \includegraphics[width=\textwidth]{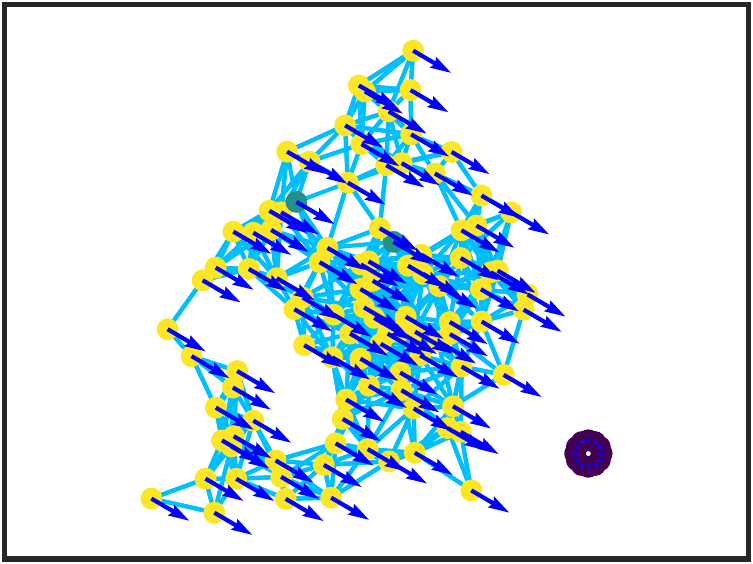}
        \label{fig:dataset:obstacle-avoidance:randomdisk:0}
    \end{subfigure}
    \hspace{1em}
    \begin{subfigure}[T]{.3\textwidth}
        \centering
        \includegraphics[width=\textwidth]{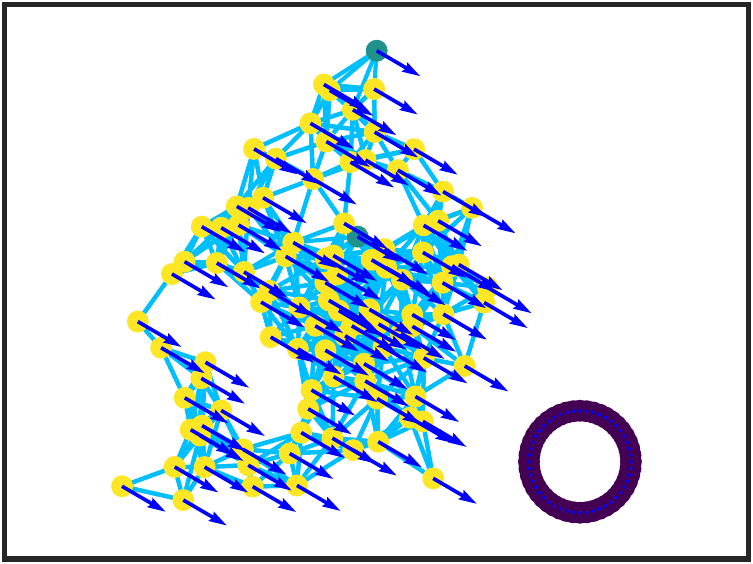}
        \label{fig:dataset:obstacle-avoidance:randomdisk:1}
    \end{subfigure}
    \hspace{1em}
    \begin{subfigure}[T]{.3\textwidth}
        \centering
        \includegraphics[width=\textwidth]{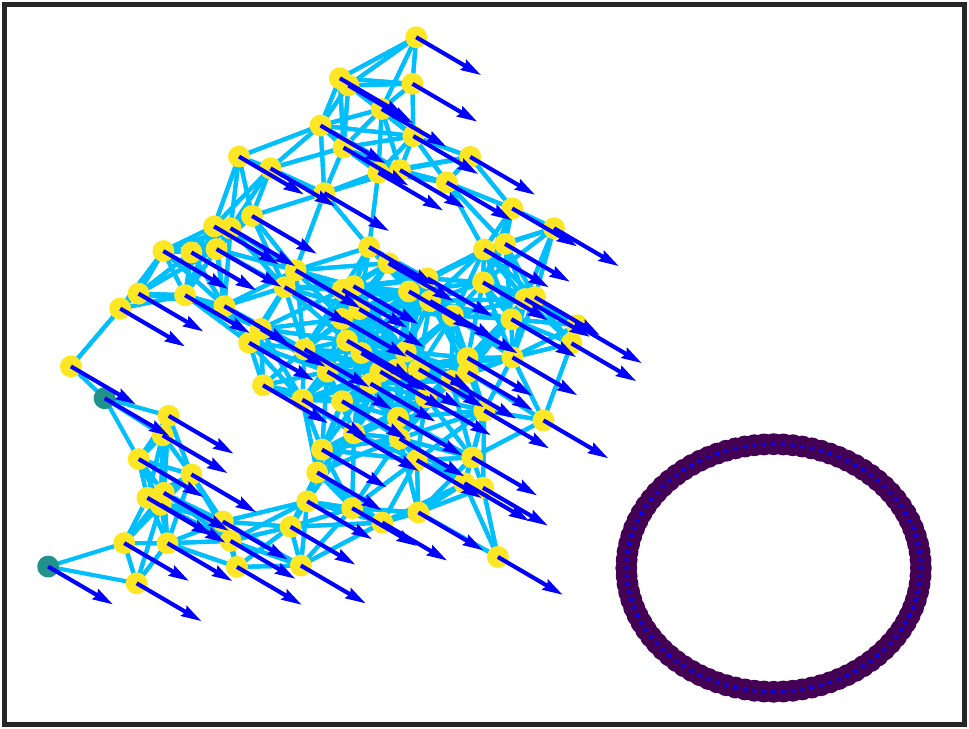}
        \label{fig:dataset:obstacle-avoidance:randomdisk:2}
    \end{subfigure}
\caption{
Examples from the Obstacle Avoidance RandomDisk dataset of initial conditions of agents with $\FlockCommunicationGraphMinDegree = 2$, $\FlockAgentMinDist = 0.1$, $\FlockAgentCommRadius = 1$, initial velocity (dark blue arrows) with magnitude 1, $98$ agent followers (yellow), two leaders (green), and obstacles (purple) of perimeter 12, 48, and 96 (all multiplied by $\FlockAgentMinDist$).}
    \label{fig:dataset:obstacle-avoidance:randomdisk}\vspace{-0.15cm}
\end{figure}

Instead of terminating the obstacle avoidance simulations after some fixed number of time steps $\DAggerSimulationTimeStepsCount$, the simulations are run until any of the following three termination conditions are met:
\begin{enumerate}
\item \textbf{Disconnected communication graph:}
First, check if the minimum distance between the flock 
and the obstacle is greater than $\FlockAgentCommRadius$, implying that the flock is unaware of the obstacle. If so, terminate the simulation if the communication graph containing leaders and followers is disconnected. Otherwise, terminate the simulation if the graph containing leaders, followers, and obstacle agents is disconnected.

\item \textbf{Collision:} Terminate the simulation if any flock agents are closer than $\RadiusMin$ to each other or any obstacle agents, or the distance of any flock agent to the center of the obstacle is less than or equal to the radius of the obstacle. 
    {
    \newcommand{\zTimeStepsAfterObstaclePassed}{{\DAggerSimulationTimeStepsCount_{\mathrm{passed}}}}
\item \textbf{Obstacle passed:} Terminate the simulation if the minimum distance between the flock 
and obstacle agents has been less than 
$\FlockAgentCommRadius$ for at least one 
step, and that distance has been larger than $\FlockAgentCommRadius$ for the past $\zTimeStepsAfterObstaclePassed \geq 1$ time steps. This termination condition indicates successful obstacle avoidance.
    }
\end{enumerate}
The obstacle avoidance 
validation metric is the fraction of terminations not due to passing the obstacle.

Now, we describe our obstacle avoidance technique.
At a high level, we have the followers orbit the obstacle, and then we use the leaders to draw the followers away from the obstacle and terminate the followers' orbits. The technique influences the acceleration of a follower $\FlockAgentIndex$ by formulating the relative velocity $-\FlockAgentRelativeVel = \FlockAgentRelativeVel[_{\FlockAgentIndexTwo\FlockAgentIndex}]$ between it and an obstacle agent $\FlockAgentIndexTwo$. Notice that we wrote the relative velocity vector so that it is rooted at the velocity of the follower $\FlockAgentIndex$.
This lets us work from the reference frame of the follower.

{
\newcommand{\zRotationMatrix}{{\UtilRotationMatrix}}
\newcommand{\zAngle}{{\theta}}
\newcommand{\zLinearDiscriminantName}{{\gamma}}
\newcommand{\zFlockAgentVelocity}{\vv}
\newcommand{\zVelocityNeighborAverage}{{\overline{\zFlockAgentVelocity}}}
\newcommand{\zRelativeVelocityMagnitude}{{\alpha_1}}
\newcommand{\zRelativeVelocityGravityBreak}{{\alpha_2}}
\newcommand{\zRelativeVelocityObstacleDodgeAngle}{{\alpha_\zAngle}}

The technique defines the interaction between a follower agent and a single obstacle agent.
As we will later show empirically, all these pairwise interactions culminate in the flock's obstacle-avoidance capability.
To implement our technique, we define a parametrized linear discriminant to make two classifications about a follower: (1) whether the follower is moving towards or away from an obstacle agent, 
and (2) whether the obstacle agent is on the left or right side of the follower's heading. The parameterized linear discriminant is
\begin{align}\label{eq:results:obstacle-avoidance-linear-discriminant}
\zLinearDiscriminantName\pn{\FlockAgentRelativePos[], \zFlockAgentVelocity, \zAngle} = \frac{\FlockAgentRelativePos[]^\top}{\norm{\FlockAgentRelativePos[]}}\zRotationMatrix\pn{\zAngle}\frac{\zFlockAgentVelocity}{\norm{\zFlockAgentVelocity}},
\end{align}
for 
$\FlockAgentRelativePos[],\zFlockAgentVelocity \neq \vzero$ and 
rotation matrix $\UtilRotationMatrix\pn{\cdot}$. 
Fixing $\FlockAgentVel_\FlockAgentIndex$, the linear discriminant $\FlockAgentRelativePos[_{\FlockAgentIndexTwo\FlockAgentIndex}] \mapsto \zLinearDiscriminantName\pn{\FlockAgentRelativePos[_{\FlockAgentIndexTwo\FlockAgentIndex}], \FlockAgentVel_\FlockAgentIndex, 0}$ classifies whether the follower is moving toward (positive value) or away (negative value) from the obstacle agent. Moreover, the linear discriminant $\FlockAgentRelativePos[_{\FlockAgentIndexTwo\FlockAgentIndex}] \mapsto \zLinearDiscriminantName\pn{\FlockAgentRelativePos[_{\FlockAgentIndexTwo\FlockAgentIndex}], \FlockAgentVel_\FlockAgentIndex, \frac{\pi}{2}}$ classifies whether the obstacle agent is to the left (positive value) or right (negative value) of the follower's heading.
These linear discriminants are assembled to determine how the follower should accelerate when receiving the position of an obstacle agent. When the follower accelerates, $\FlockAgentVel_\FlockAgentIndex$ changes, and so do the outputs of the linear discriminants. To reduce the linear discriminants sensitivity to $\FlockAgentVel_\FlockAgentIndex$, we replace $\FlockAgentVel_\FlockAgentIndex$ with the mean velocity of the follower and its neighbors: $\zVelocityNeighborAverage_\FlockAgentIndex\pn{t} = \mean\set{\FlockAgentVel_j\pn{t}: j \in \FlockAgentNeighborhood_\FlockAgentIndex\pn{t} \cup \set{i}}$.

Finally, our technique formulates the relative velocity as
\begin{align}
    \label{eq:results:obstacle-avoidance:relative-velocity}
    -\FlockAgentRelativeVel\pn{t} & = \begin{cases}
        \zRelativeVelocityMagnitude\pn{\norm{\FlockAgentRelativePos[_{\FlockAgentIndexTwo\FlockAgentIndex}]}, \norm{\zVelocityNeighborAverage_\FlockAgentIndex}}\frac{\FlockAgentRelativePos[_{\FlockAgentIndexTwo\FlockAgentIndex}]}{\norm{\FlockAgentRelativePos[_{\FlockAgentIndexTwo\FlockAgentIndex}]}} 
        \ \text{if } -\zRelativeVelocityGravityBreak \leq \zLinearDiscriminantName\pn{\FlockAgentRelativePos[_{\FlockAgentIndexTwo\FlockAgentIndex}],\zVelocityNeighborAverage_\FlockAgentIndex, 0} \leq 0, \\
    \zRelativeVelocityMagnitude\pn{\norm{\FlockAgentRelativePos[_{\FlockAgentIndexTwo\FlockAgentIndex}]},\norm{\zVelocityNeighborAverage_\FlockAgentIndex}}\pn{-\mathrm{sgn}\spn{\zLinearDiscriminantName\pn{\FlockAgentRelativePos[_{\FlockAgentIndexTwo\FlockAgentIndex}], \zVelocityNeighborAverage_\FlockAgentIndex, \frac{\pi}{2}}}\zRotationMatrix\pn{\alpha_\zAngle}}\frac{\FlockAgentRelativePos[_{\FlockAgentIndexTwo\FlockAgentIndex}]}{\norm{\FlockAgentRelativePos[_{\FlockAgentIndexTwo\FlockAgentIndex}]}} & \text{else},
    \end{cases}
\end{align}
where $\zRelativeVelocityMagnitude:\mathbb{R}_+^2\to\mathbb{R}_+$ is a rescaling function, 
$\zRelativeVelocityGravityBreak \in \spn{0, 1}$, and $\zRelativeVelocityObstacleDodgeAngle \in \hrpn{0, \frac{\pi}{2}}$. A follower may have multiple obstacle agent 
neighbors, but we limit the follower to only process the relative velocity computed from the position of the closest obstacle agent.

We explain the relative velocity formulation considering a follower $\FlockAgentIndex$, its flock agent neighbors, and the obstacle agent $\FlockAgentIndexTwo$ that it is closest to.
From the definition, when $\zLinearDiscriminantName\pn{\FlockAgentRelativePos[_{\FlockAgentIndexTwo\FlockAgentIndex}], \zVelocityNeighborAverage_\FlockAgentIndex, 0} > 0$, the follower is roughly heading toward the obstacle agent, 
so the relative velocity accelerates the follower to the left or right of the obstacle agent (depending on the sign of $\zLinearDiscriminantName\pn{\FlockAgentRelativePos[_{\FlockAgentIndexTwo\FlockAgentIndex}], \zVelocityNeighborAverage_\FlockAgentIndex, \frac{\pi}{2}}$).
When moving left or right, eventually $\zLinearDiscriminantName\pn{\FlockAgentRelativePos[_{\FlockAgentIndexTwo\FlockAgentIndex}], \zVelocityNeighborAverage_\FlockAgentIndex, 0} \in \spn{-\zRelativeVelocityGravityBreak, 0}$, meaning the follower either is heading tangent to a disk of radius $\norm{\FlockAgentRelativePos[_{\FlockAgentIndexTwo\FlockAgentIndex}]}$ centered at the obstacle agent or is heading away from the obstacle agent. 
In this case, the relative velocity accelerates the follower toward the obstacle agent, initiating an orbit about the obstacle agent. The orbit helps the follower reunite with the followers that moved around the obstacle agent the opposite way. When $\zRelativeVelocityGravityBreak = 1$, the relative velocity accelerates the follower toward the obstacle agent even when its flock agent neighbors are moving away from the obstacle agent. The follower can get stuck in the ``gravity'' of the obstacle agent, so setting $\zRelativeVelocityGravityBreak \in \hlpn{0, 1}$ can help the follower terminate its orbit. When $\zLinearDiscriminantName\pn{\FlockAgentRelativePos[_{\FlockAgentIndexTwo\FlockAgentIndex}], \zVelocityNeighborAverage_\FlockAgentIndex, 0} < -\zRelativeVelocityGravityBreak$ and $\zRelativeVelocityObstacleDodgeAngle = \pi/2$, the relative velocity accelerates the followers in the direction tangent to a disk of radius $\norm{\FlockAgentRelativePos[_{\FlockAgentIndexTwo\FlockAgentIndex}]}$ centered at the obstacle agent.

The followers are drawn away from the obstacle agents by the leaders because the leaders will always eventually head away from all of the obstacle agents. The followers try to align their velocity with the leaders' velocity, and if they can, $\zLinearDiscriminantName\pn{\FlockAgentRelativePos[_{\FlockAgentIndexTwo\FlockAgentIndex}], \zVelocityNeighborAverage_\FlockAgentIndex, 0}$ will become negative. How closely the followers need to align their velocity with the leaders in order to move away from the obstacle agents depends on $\zRelativeVelocityGravityBreak$.
In our experiment, we choose $\zRelativeVelocityGravityBreak = 0.5$, $\zRelativeVelocityObstacleDodgeAngle = \pi/2$, and $\zRelativeVelocityMagnitude$ as 
$\zRelativeVelocityMagnitude\pn{r, v} = e^{-r} + e^{-v}$.
We offer an intuition for our choice of $\zRelativeVelocityMagnitude$. The $e^{-\norm{\FlockAgentRelativePos[_{\FlockAgentIndexTwo\FlockAgentIndex}]}}$ term amplifies the acceleration 
supplied by the relative velocity when the follower and obstacle agent 
are closer. The $e^{-\norm{\zVelocityNeighborAverage_\FlockAgentIndex}}$ term strengthens the acceleration if the follower slows down to avoid colliding with the obstacle. This helps the follower maintain the magnitude of its velocity prior to detecting the obstacle, allowing the flock 
to circumnavigate the obstacle faster. 
Attenuating the signal when the follower's velocity is small would allow it to stall in front of the obstacle and fall behind the leaders.

}

\Figref{fig:results:obstacle-avoidance:randomdisk:by-perimeter:no-transfer-learning-algorithm} shows the failure rate of ML controllers on the test simulations when they do \textbf{\textit{not}} use our technique. Their failure rate is near 100\% with over 95\% of failures due to disconnected communication graphs. Using our obstacle avoidance technique makes obstacle avoidance possible for ML controllers, as shown in \figref{fig:results:obstacle-avoidance:randomdisk:by-perimeter}.
When the obstacle perimeter is $24\FlockAgentMinDist$ or smaller, the failure fraction decreases by at least 
80\% 
for all ML controllers. It decreases by 
55\%
for perimeter $48\FlockAgentMinDist$ and at least 
10\%
for perimeter $96\FlockAgentMinDist$.

\begin{figure}[!ht]
    \centering
    \begin{subfigure}[T]{\textwidth}
        \centering
        \includegraphics[width=0.95\textwidth]{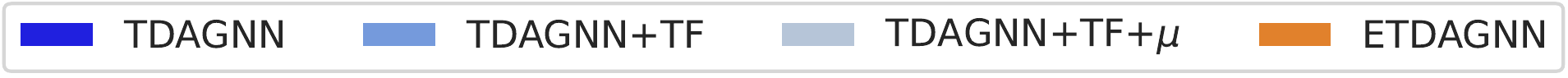}
    \end{subfigure} \\
    \begin{subfigure}[T]{.3\textwidth}
        \centering
        \hspace{-2.3em}
        \includegraphics[width=1.125\textwidth]{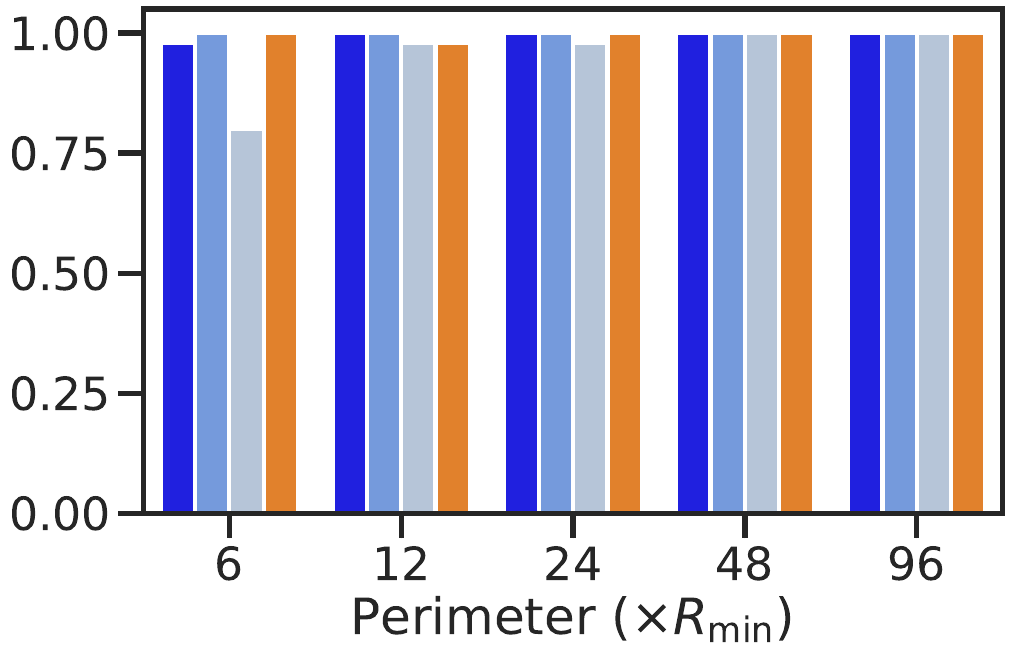}
        \caption{}
    \end{subfigure}
    \begin{subfigure}[T]{.3\textwidth}
        \centering
        \includegraphics[width=\textwidth]{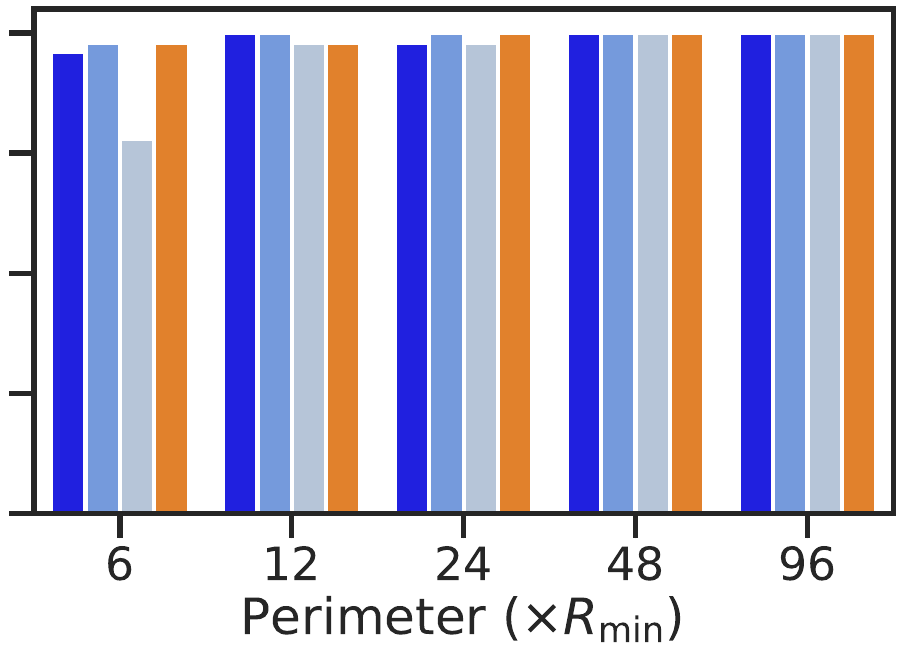}
        \caption{}
    \end{subfigure}
    \begin{subfigure}[T]{.3\textwidth}
        \centering
        \includegraphics[width=\textwidth]{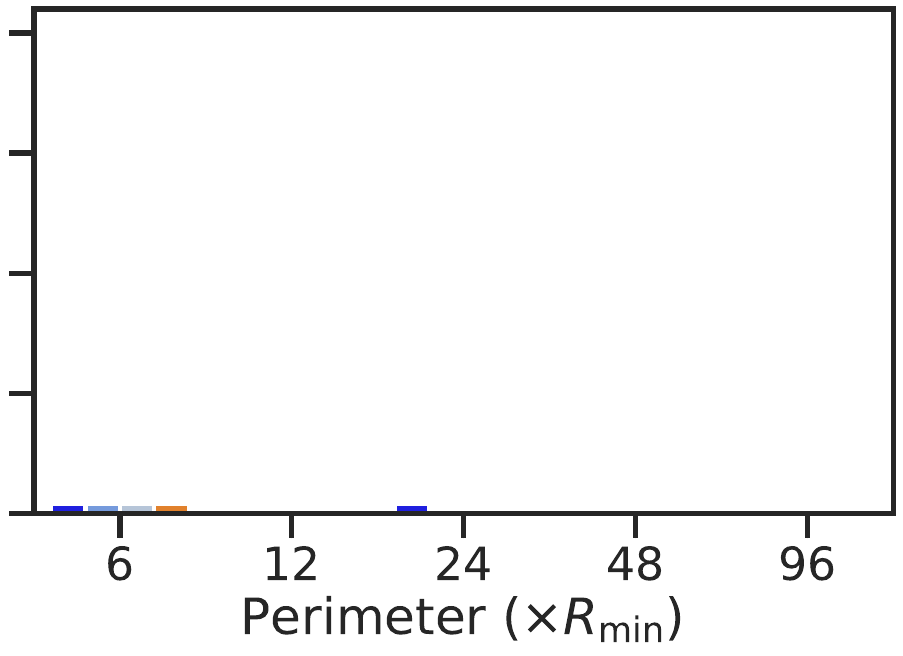}
        \caption{}
    \end{subfigure}
\caption{
Performance of ML controllers in obstacle avoidance varying the perimeter of the regular polygon inscribed in the obstacle when our obstacle avoidance technique is \textbf{\textit{not}} used. The obstacle is a disk, and the regular polygon has $s$ sides of length $\FlockAgentMinDist$ whose vertices are on the boundary of the disk. They are evaluated on the RandomDisk test set with $\FlockAgentCount = 100$ agents where two agents are leaders.
(a) Fraction of simulations that failed.
(b) Fraction of failures due to a disconnected communication graph.
(c) Fraction of failures due to a collision.
}
\label{fig:results:obstacle-avoidance:randomdisk:by-perimeter:no-transfer-learning-algorithm}
\end{figure}

\begin{figure}[!ht]
    \centering
    \begin{subfigure}[T]{\textwidth}
        \centering
        \includegraphics[width=0.95\textwidth]{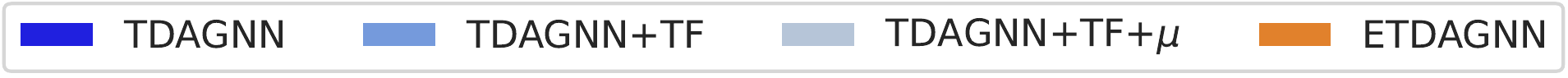}
    \end{subfigure} \\
    \begin{subfigure}[T]{.3\textwidth}
        \centering
        \hspace{-2.3em}
        \includegraphics[width=1.125\textwidth]{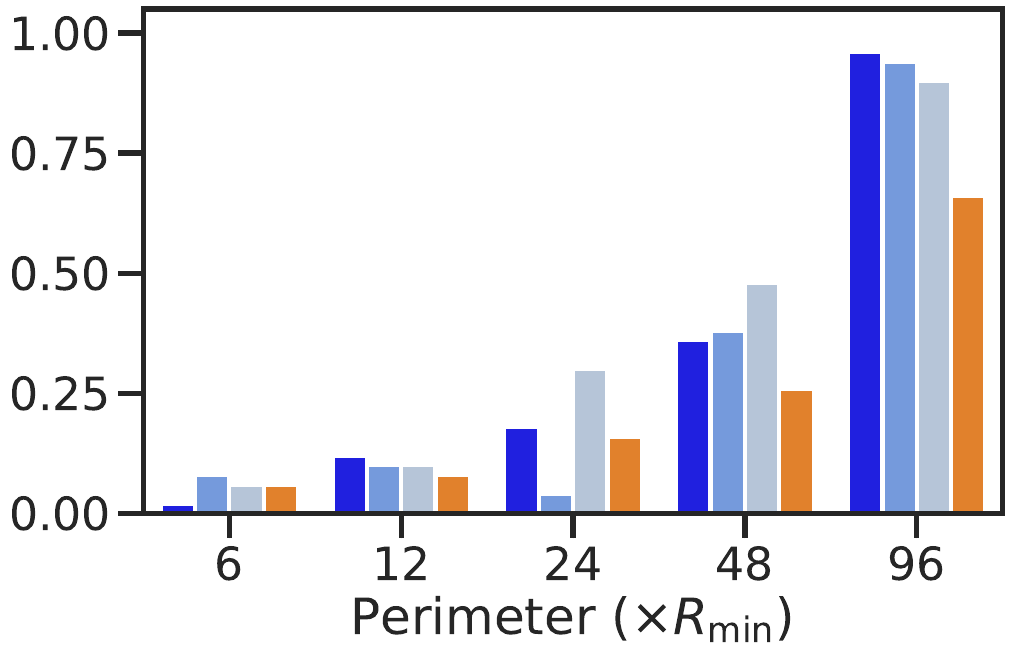}
        \caption{}
    \end{subfigure}
    \begin{subfigure}[T]{.3\textwidth}
        \centering
        \includegraphics[width=\textwidth]{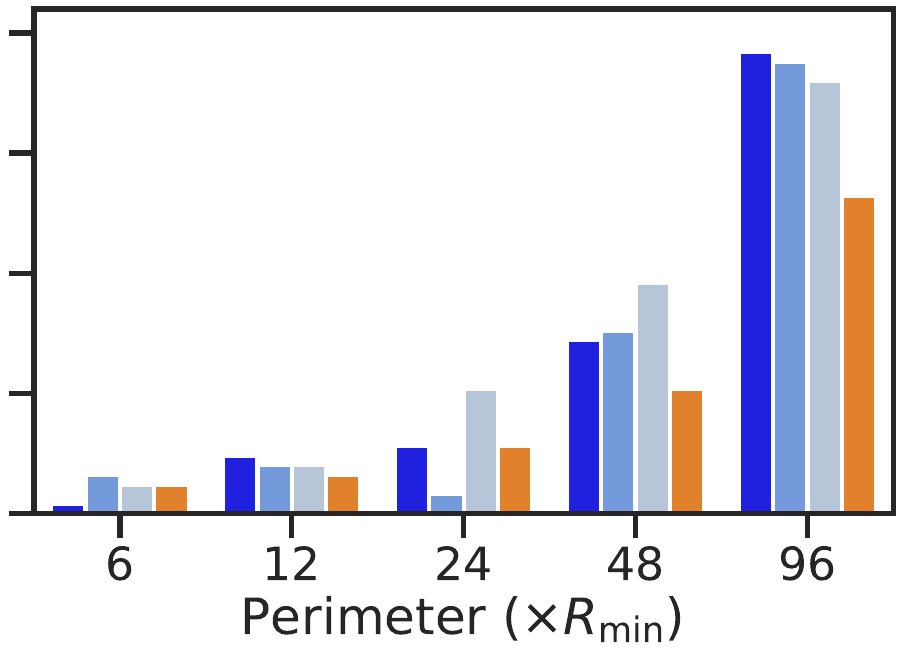}
        \caption{}
    \end{subfigure}
    \begin{subfigure}[T]{.3\textwidth}
        \centering
        \includegraphics[width=\textwidth]{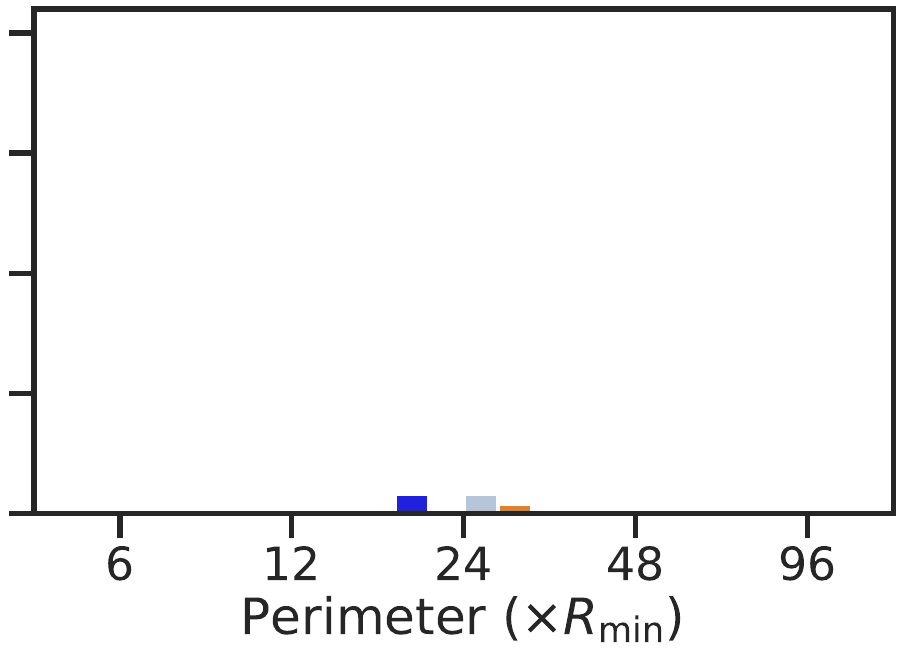}
        \caption{}
    \end{subfigure}
\caption{
Performance of ML controllers in obstacle avoidance using our obstacle avoidance technique varying the perimeter of the regular polygon inscribed in the obstacle. The obstacle is a disk, and the regular polygon has $s$ sides of length $\FlockAgentMinDist$ whose vertices are on the boundary of the disk. They are evaluated on the RandomDisk test set with $\FlockAgentCount = 100$ agents where two agents are leaders.
(a) Fraction of simulations that failed.
(b) Fraction of failures due to a disconnected communication graph.
(c) Fraction of failures due to a collision.
}
\label{fig:results:obstacle-avoidance:randomdisk:by-perimeter}
\end{figure}

Communication graph disconnection is still the dominant cause of failure. Failures due to collisions with the obstacle only occur when the obstacle perimeter is $24\FlockAgentMinDist$, representing less than 5\% of failures. \PupilETDAGNN{} has a failure rate at least 10\% lower than other controllers when the obstacle perimeter is $48\FlockAgentMinDist$, and 20\% lower for perimeter $96\FlockAgentMinDist$. For perimeter $48\FlockAgentMinDist$ and $96\FlockAgentMinDist$, the other ML controllers' failure rates are within 10\% of each other.
For obstacles a perimeter $6\FlockAgentMinDist$ or smaller, all ML controllers have a failure rate of less than 10\%.
For larger obstacles, we recommend \PupilETDAGNN{} since its failure fraction is up to 20\% smaller than other ML controllers. Rotation equivariance provides a significant performance advantage in this obstacle avoidance 
experiment.
Animations of obstacle avoidance are available on GitHub.\footnote{Obstacle avoidance animations: \url{github.com/Utah-Math-Data-Science/Equivariant-Decentralized-Controllers/tree/main/misc/animations/obstacle_avoidance}}

{
\newcommand{\zGeneralizationTestSampleName}{\GeneralizationSampleName^\mathrm{test}}
\newcommand{\zGeneralizationTestSampleSize}{\GeneralizationSampleSize^\mathrm{test}}
\newcommand{\zEmpiricalGeneralizationGap}{\hat{\cR}_{\zGeneralizationTestSampleName,\GeneralizationSampleName,\LossFunctionName}}
\subsection{Generalization gap}
\label{sec:results:generalization-gap}
We 
verify the generalization bound in \eqref{eq:generalization-gap:bound} for each ML controller with respect to its \GeneralizationFastForwardBC{} dataset. Each dataset has 80,400 tuples with 30,150 for training and 50,250 for testing.
We train the ML controllers on their training sets with $\LossFunctionMSENormalizationConstant = 2$, and evaluate them on their test sets. The constant $\LossFunctionMSENormalizationConstant$ helps avoid the loss gradients being 
zero due to the $\min\set{1, \cdot}$ function.
We compute the generalization bound in \eqref{eq:generalization-gap:bound} with $\GeneralizationGapBoundFailProbability = 10^{-3}$ and the empirical bound.
The empirical generalization bound is the difference of the empirical risk over the test set minus the empirical risk over the training set
$$
\begin{aligned}
\zEmpiricalGeneralizationGap\pn{f} = \frac{1}{\zGeneralizationTestSampleSize} \sum_{\GeneralizationSampleIndex = 1}^{\zGeneralizationTestSampleSize} \LossFunctionName\pn{f\pn{x_\GeneralizationSampleIndex^\mathrm{test}}, y^\mathrm{test}} - \frac{1}{\GeneralizationSampleSize} \sum_{\GeneralizationSampleIndex = 1}^{\GeneralizationSampleSize} \LossFunctionName\pn{f\pn{x_\GeneralizationSampleIndex}, y}.
\end{aligned}
$$
\Figref{fig:results:generalization-gap:losses} and \ref{fig:results:generalization-gap:bound} show that the empirical generalization gap is near zero for all controllers. 
Here, the training and the test sets are sampled from the same probability distribution induced by \GeneralizationFastForwardBC{}. 

\begin{figure}[!ht]
    \centering
    \begin{subfigure}[T]{\textwidth}
        \centering
        \includegraphics[width=.5\textwidth]{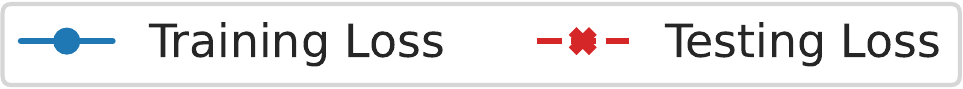}
    \end{subfigure} \\
    \begin{subfigure}[T]{.23\textwidth}
        \centering
        \hspace{-2.5em}
        \includegraphics[width=1.187\textwidth]{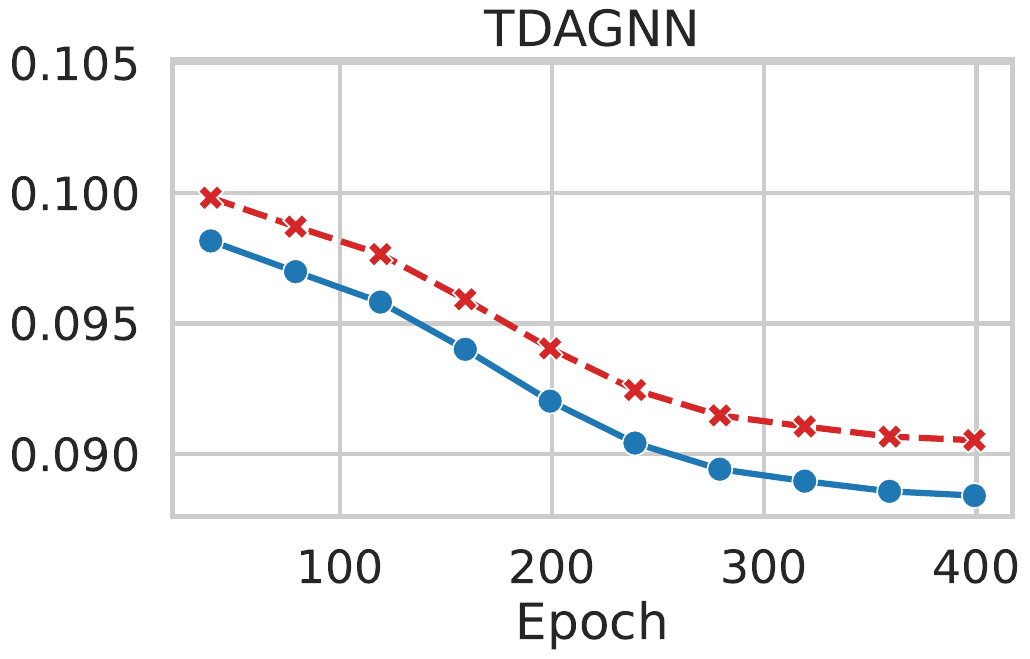}
    \end{subfigure}
    \begin{subfigure}[T]{.23\textwidth}
        \centering
        \includegraphics[width=\textwidth]{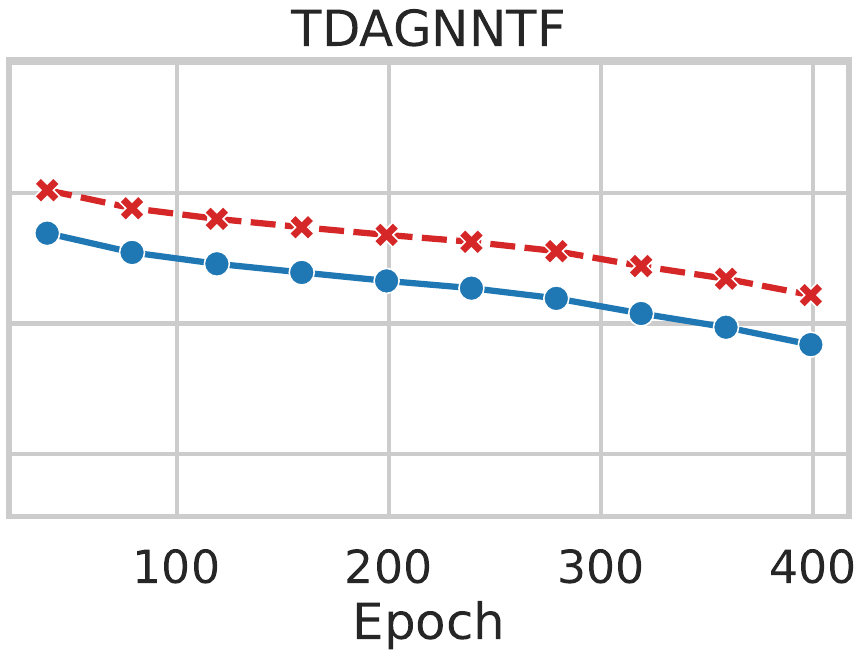}
    \end{subfigure}
    \begin{subfigure}[T]{.23\textwidth}
        \centering
        \includegraphics[width=\textwidth]{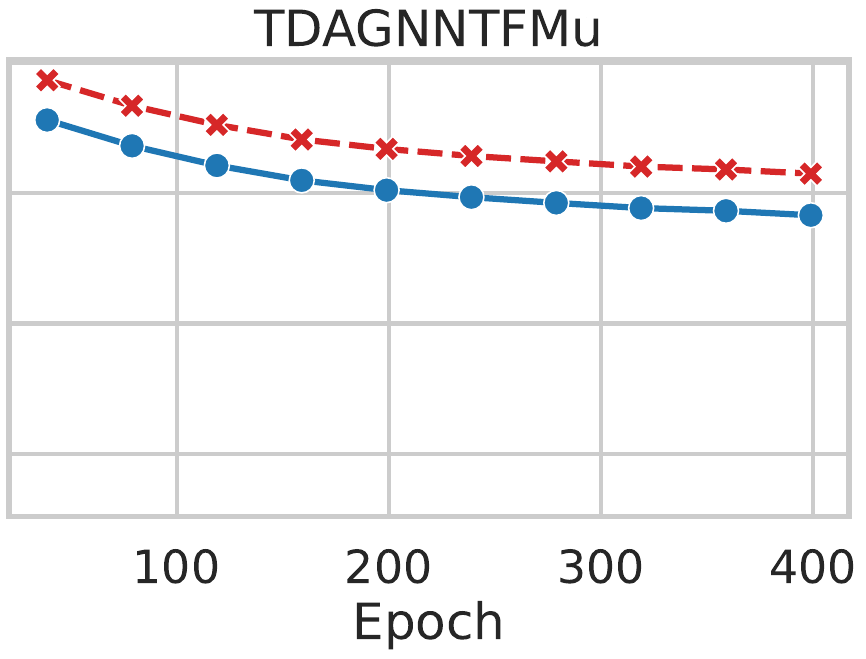}
    \end{subfigure}
    \begin{subfigure}[T]{.23\textwidth}
        \centering
        \includegraphics[width=\textwidth]{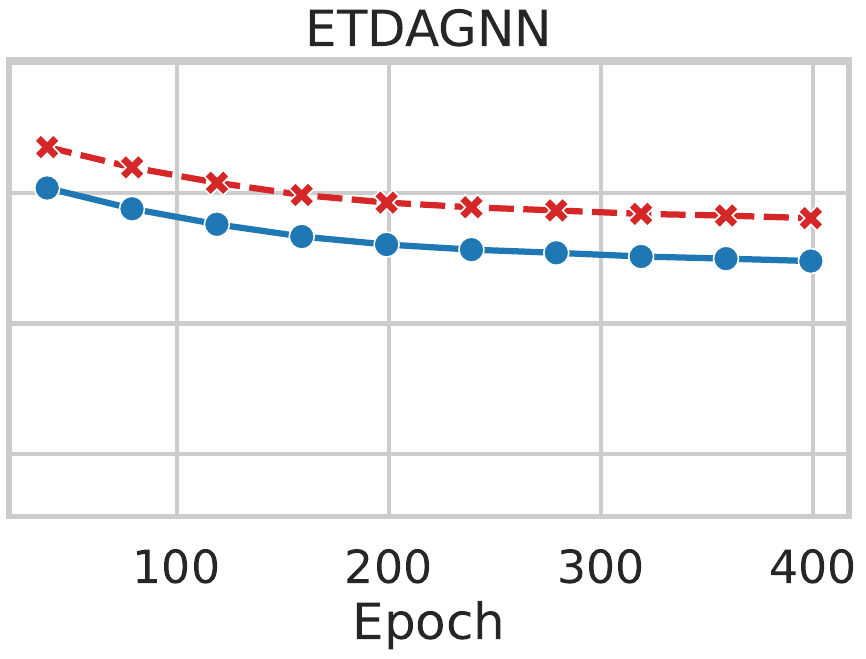}
    \end{subfigure}
    \caption{
        The losses on the behavior cloning training and test sets as the flocking ML controllers train.
    }
    \label{fig:results:generalization-gap:losses}
\end{figure}

\begin{figure}[!ht]
    \centering
    \begin{subfigure}[T]{\textwidth}
        \centering
        \includegraphics[width=.95\textwidth]{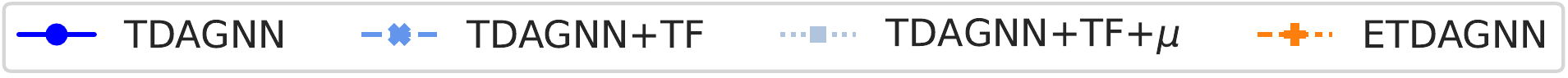}
    \end{subfigure} \\
    \begin{subfigure}[T]{.23\textwidth}
        \centering
        \hspace{-2.5em}
        \includegraphics[width=1.187\textwidth]{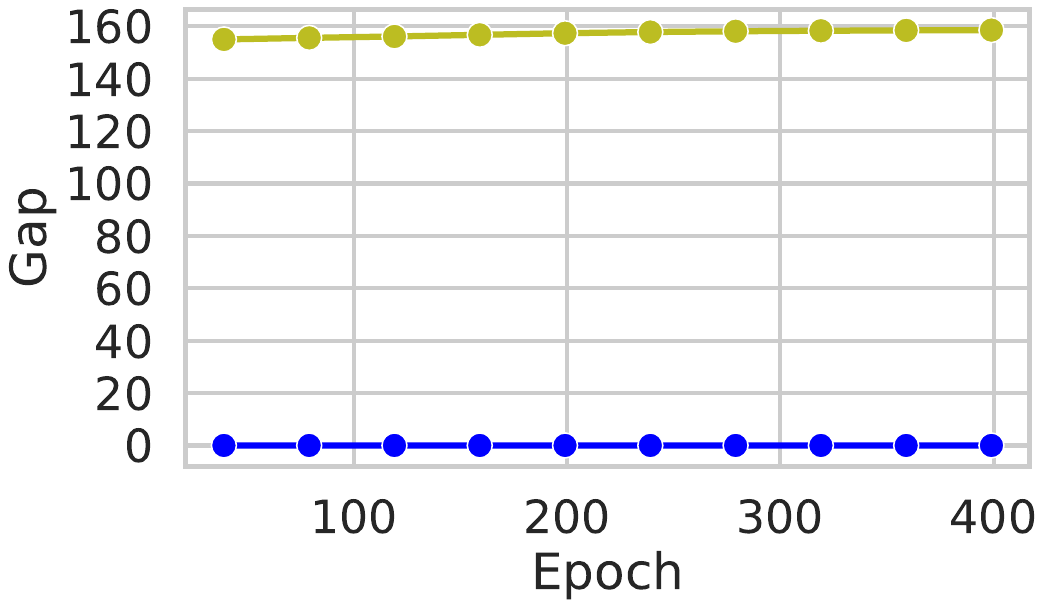}
    \end{subfigure}
    \begin{subfigure}[T]{.23\textwidth}
        \centering
        \includegraphics[width=\textwidth]{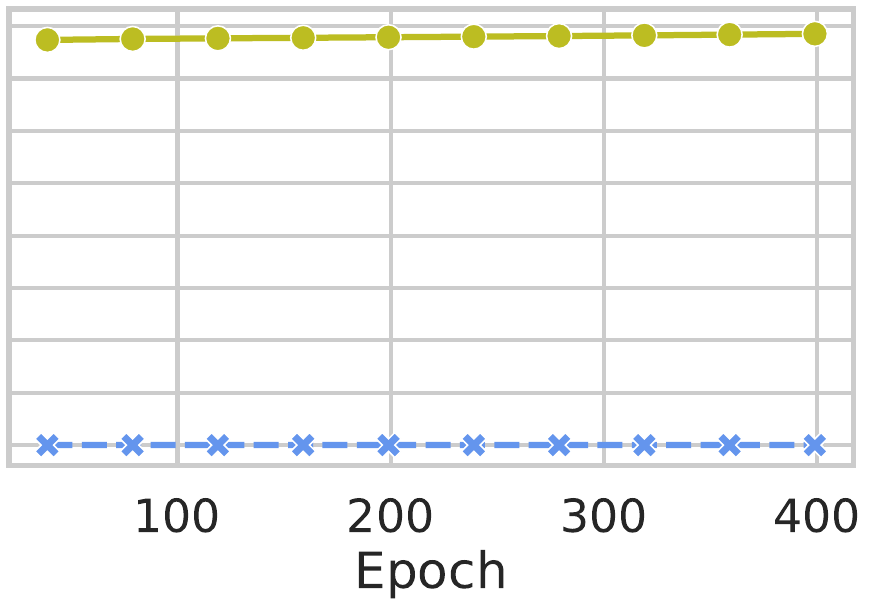}
    \end{subfigure}
    \begin{subfigure}[T]{.23\textwidth}
        \centering
        \includegraphics[width=\textwidth]{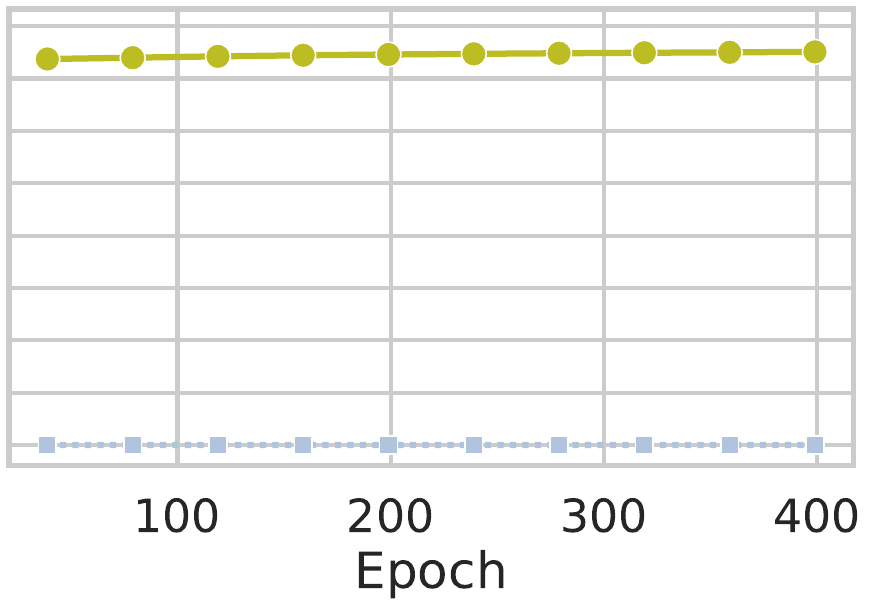}
    \end{subfigure}
    \begin{subfigure}[T]{.23\textwidth}
        \centering
        \includegraphics[width=\textwidth]{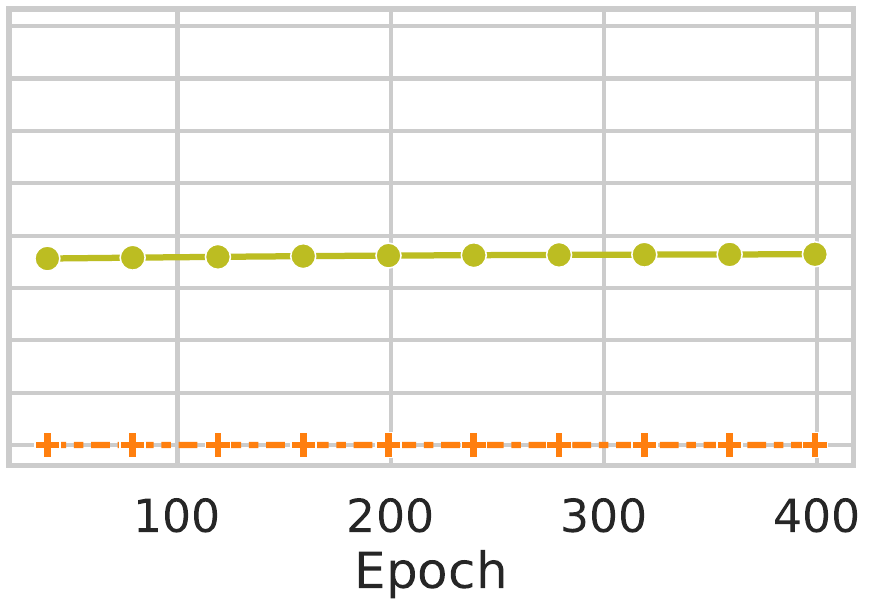}
    \end{subfigure}
    \caption{
        The generalization bound and empirical generalization gap over epochs.
        The generalization bound 
        improves with the ML controllers from left to right, and adding equivariance lowers the generalization gap the most.
    }
    \label{fig:results:generalization-gap:bound}
\end{figure}

The generalization bound reduces as we add the improvements (described in \secref{sec:methods:improving-tdagnn}) to \PupilTDAGNN{}. 
Enforcing equivariance significantly reduces the generalization gap, as \PupilETDAGNN{}'s bound is about half that of \PupilTDAGNNTFMu{}.
The term in the generalization bound that varies between ML controllers is $48\PupilWeightMatrixLargestDim/\sqrt{\GeneralizationSampleSize}$ multiplied by the square root term $\sqrt{\cdot}$. 
\Figref{fig:results:generalization-gap:variables} gives insight into how these terms influence the generalization bound.
The plot of $48\PupilWeightMatrixLargestDim/\sqrt{\GeneralizationSampleSize}$ explains why the generalization bound of \PupilETDAGNN{} is about half of the non-equivariant controllers' -- \PupilETDAGNN{} has $\PupilWeightMatrixLargestDim = 16$ whereas the non-equivariant ML controllers have $\PupilWeightMatrixLargestDim = 33$. Next, $\DAggerDatumBound$ is an order of magnitude smaller for mean-aggregation controllers compared to sum-aggregation; however, it only introduces a 10-point difference between the bounds of \PupilTDAGNNTF{} and \PupilTDAGNNTFMu{} seen in the square root term $\sqrt{\cdot}$. The summation depending on the Frobenius norm of the weight matrices is remarkably similar for the non-equivariant controllers and \PupilETDAGNN{}. \PupilETDAGNN{} has about 75\% fewer weights (see Table~\ref{tab:results:num-model-weights}), so its weights tend to be larger than that of the non-equivariant controllers. Though these bounds are much larger than the empirical generalization gap, \figref{fig:results:generalization-gap:corr} shows that they have a high correlation ($\rho \geq 0.95$) with the empirical generalization gap.

\begin{figure}[!ht]
    \centering
    \begin{subfigure}[T]{\textwidth}
        \centering
        \includegraphics[width=.95\textwidth]{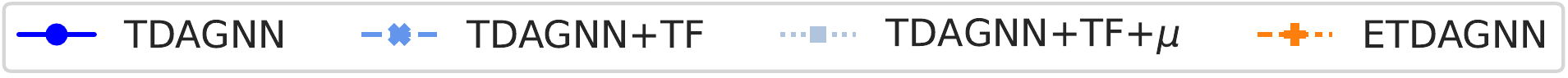}
    \end{subfigure} \\
    \begin{subfigure}[T]{.23\textwidth}
        \centering
        \hspace{-2em}
        \includegraphics[width=1.15\textwidth]{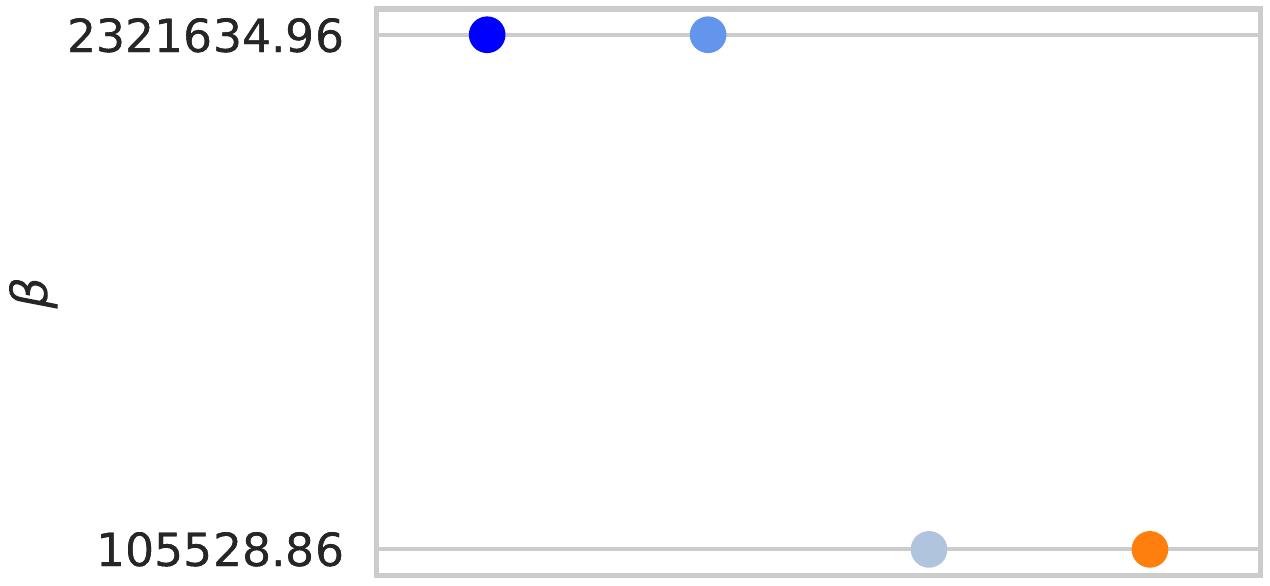}
        \caption{}
    \end{subfigure}
    \begin{subfigure}[T]{.23\textwidth}
        \centering
        \includegraphics[width=\textwidth]{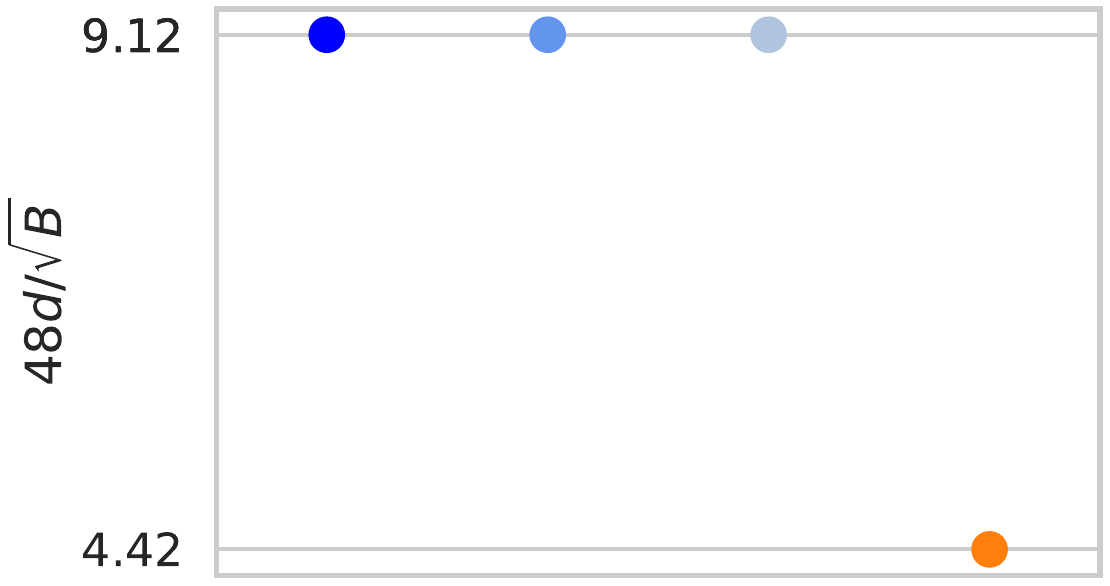}
        \caption{}
    \end{subfigure}
    \begin{subfigure}[T]{.23\textwidth}
        \centering
        \includegraphics[width=.87\textwidth]{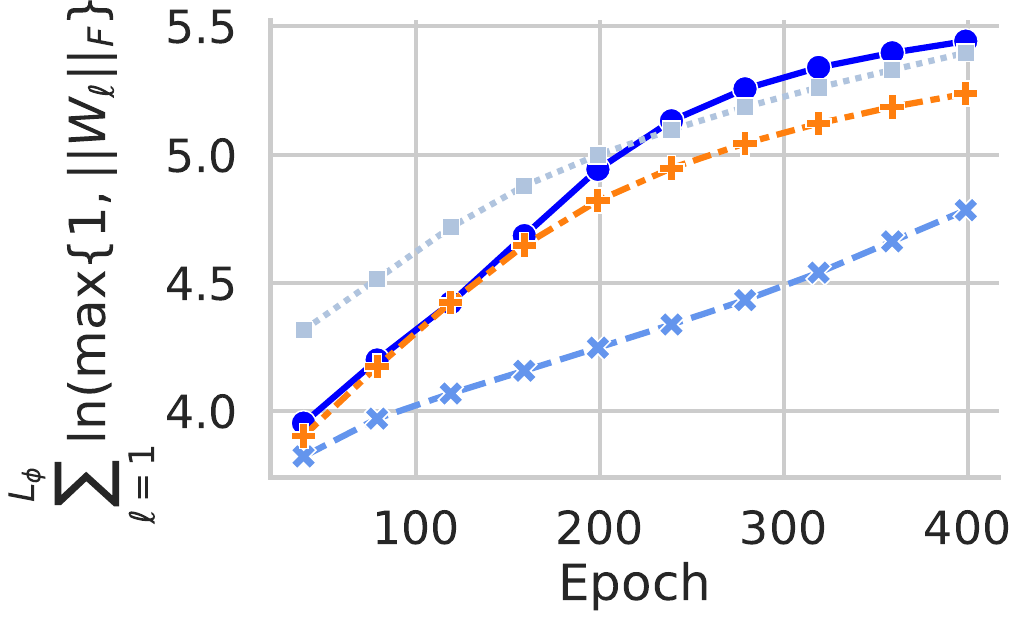}
        \caption{}
    \end{subfigure}
    \begin{subfigure}[T]{.23\textwidth}
        \centering
        \includegraphics[width=.87\textwidth]{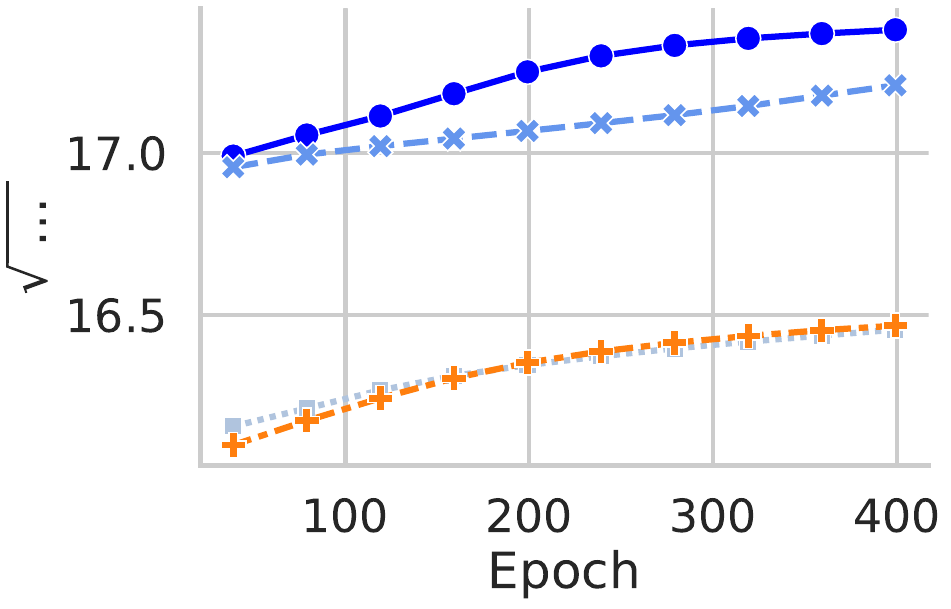}
        \caption{}
    \end{subfigure}
    \caption{
    Select terms of the generalization bound for each ML controller as it trains on its \GeneralizationFastForwardBC{} dataset.
    (a) The \GeneralizationFastForwardBC{} data bound.
    (b) The coefficient to the large square root term in the generalization bound.
    These terms are constant with respect to epochs.
    The data bound of \PupilTDAGNNTFMu{} and \PupilETDAGNN{} (mean-aggregation ML controllers) is an order of magnitude smaller compared to \PupilTDAGNN{} and \PupilTDAGNNTF{} (sum-aggregation ML controllers).
    The coefficient term of \PupilETDAGNN{} is less than half that of the other ML controllers.
    The terms in the generalization gap bound that depend on the Frobenius norm of the ML controllers' weight matrices when represented as MLPs are (c) the summation and (d) the large square root term containing that summation.
    }
    \label{fig:results:generalization-gap:variables}
\end{figure}

\begin{figure}[!ht]
    \centering
    \begin{subfigure}[T]{\textwidth}
        \centering
        \includegraphics[width=.95\textwidth]{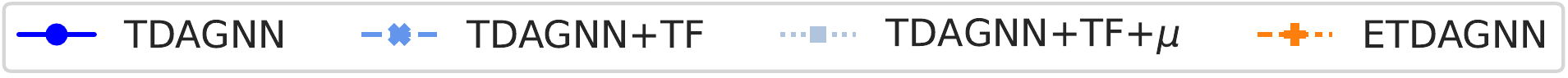}
    \end{subfigure} \\
    \begin{subfigure}[T]{.23\textwidth}
        \centering
        \hspace{-2.5em}
        \includegraphics[width=1.187\textwidth]{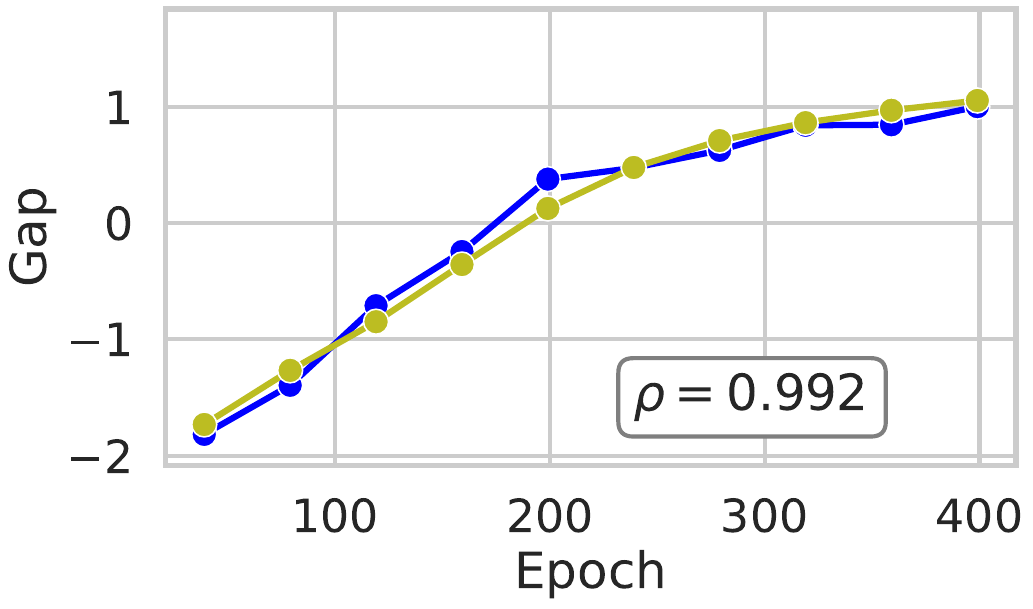}
    \end{subfigure}
    \begin{subfigure}[T]{.23\textwidth}
        \centering
        \includegraphics[width=\textwidth]{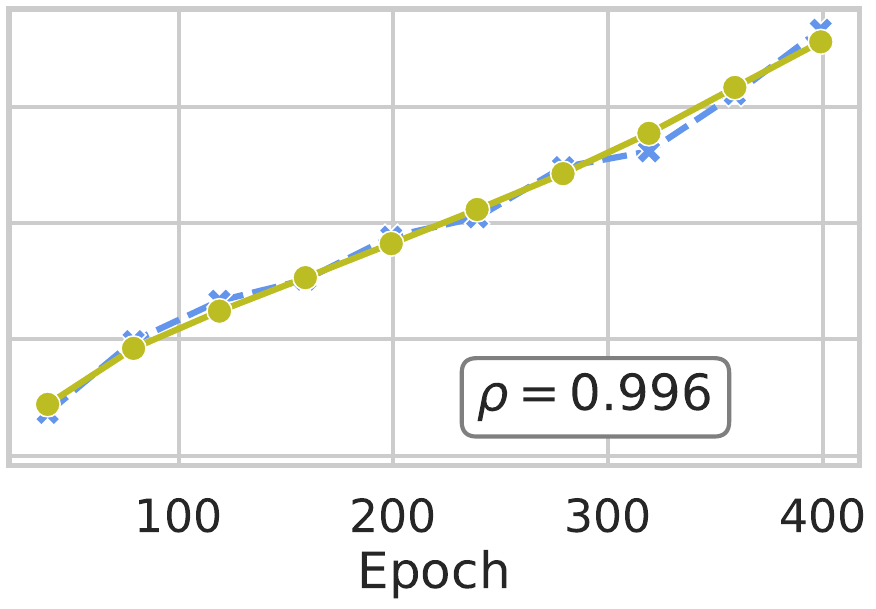}
    \end{subfigure}
    \begin{subfigure}[T]{.23\textwidth}
        \centering
        \includegraphics[width=\textwidth]{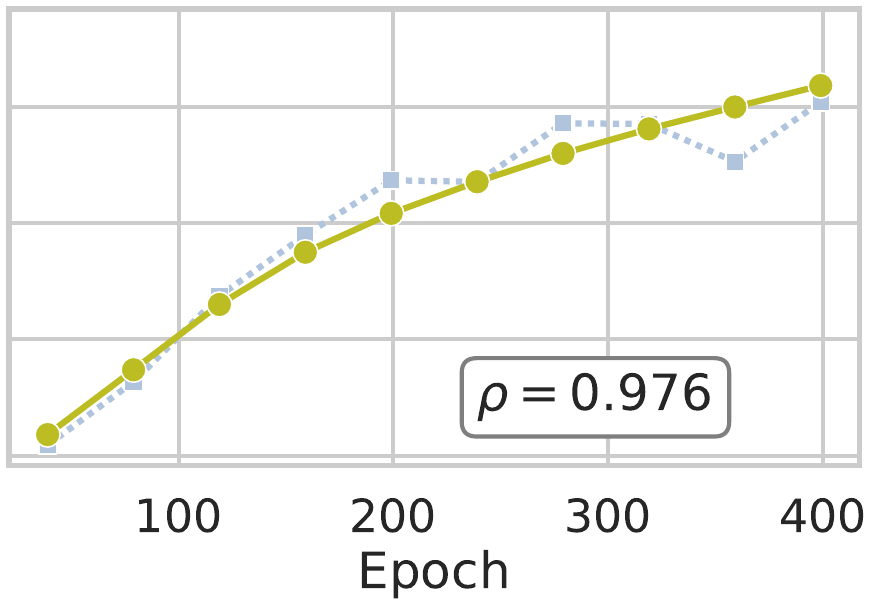}
    \end{subfigure}
    \begin{subfigure}[T]{.23\textwidth}
        \centering
        \includegraphics[width=\textwidth]{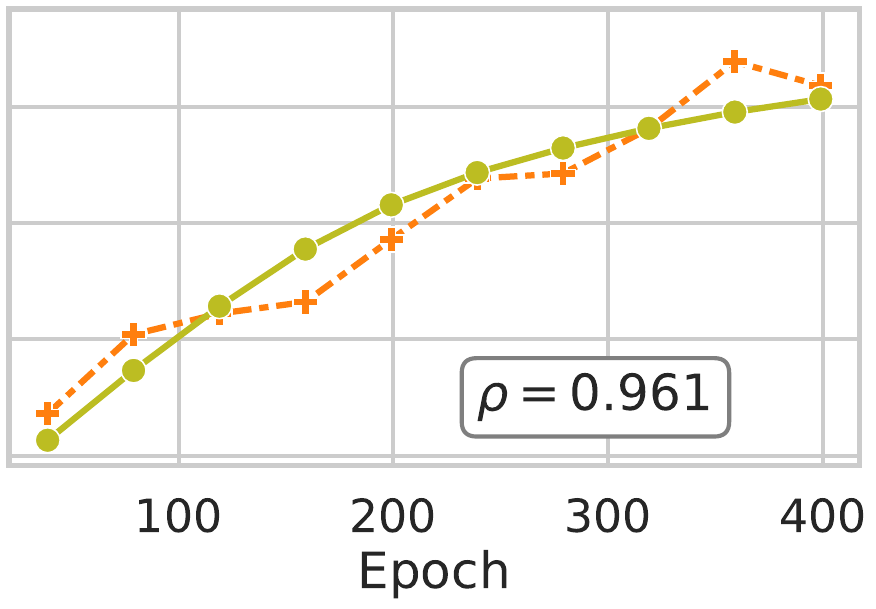}
    \end{subfigure}
    \caption{
        The correlation of generalization 
        bound and empirical generalization gap.
        We plot each of these sequences over epochs with their mean subtracted and divided by their standard deviation.
    }
    \label{fig:results:generalization-gap:corr}
\end{figure}


}

\section{Conclusion}
In response to the challenges in building decentralized flocking controllers, we presented an enhanced ML-based decentralized flocking controller that leverages a rotation equivariant and translation invariant GNN.
We demonstrated the advantages of the proposed decentralized controller over existing non-equivariant ML-based controllers and other flocking controllers using three representative case studies -- flocking, leader following, and obstacle avoidance.
We also analyzed the generalization gap of the proposed decentralized flocking controller.
Our numerical results show that the proposed decentralized controller achieves comparable performance compared to non-equivariant ML-based controllers with 70\% less training data, 75\% fewer trainable weights and a 50\% smaller generalization bound.



\section*{Declarations}


\subsection*{Data availability}
Code and animations are available at \url{github.com/Utah-Math-Data-Science/Equivariant-Decentralized-Controllers}.

\subsection*{Funding}
This material is based on research sponsored by National Science Foundation (NSF) grants DMS-2152762, DMS-2208361, DMS-2219956, and DMS-2436344, and Department of Energy grants DE-SC0023490, DE-SC0025589, and DE-SC0025801.
Taos Transue received partial financial support from the NSF under Award 2136198.

\subsection*{Competing interets}
All authors certify that they have no affiliations with or involvement in any organization or entity with any financial interest or non-financial interest in the subject matter or materials discussed in this manuscript.

\begin{appendix}

\section{Proofs}

\subsection{Preliminaries}

\begin{lemma}[Frobenius norm upper bounds spectral norm]
    \label{lemma:results:frobenius-geq-spectral}
    Let $\mA \in \bR^{m \times n}$, and then $\norm{\mA}_2 \leq \norm{\mA}_F$.
\end{lemma}
\begin{proof}
    Let $\mA = \mU\mSigma\mV^\top$ be the singular value decomposition of $\mA$ where $\mU \in \bR^{m \times m}$ and $\mV \in \bR^{n \times n}$ are orthogonal matrices.
    Let $\sigma_i$ and $\sigma_{\max}$ denote the $i$-th and the largest singular value, respectively.
    Then,
    $$
        \norm{\mA}_2 = \norm{\mU\mSigma\mV^\top}_2 = \norm{\mSigma}_2 = \sigma_{\max} \leq \sum_{i = 1}^{\min\set{m,n}} \sigma_i = \norm{\mSigma}_F = \norm{\mU\mSigma\mV^\top}_F = \norm{\mA}_F
    $$
    This bound is sharp since, for $c \in \bR^{1 \times 1}$, $\norm{c}_2 = \abs{c} = \norm{c}_F$.
\end{proof}

\subsection{Generalization gap}
\begin{lemma}
\label{lemma:results:tdagnn-bounded}
Let \PupilName{} be either \PupilTDAGNN{}, \PupilTDAGNNTF{}, or \PupilTDAGNNTFMu{}. Then,
$$
    \begin{aligned}
        \norm{\PupilName\pn{\gConvInput[1]}} =
        \norm{\PupilAsMLPName{}\pn{\gConvInputAsMLP[1]}} \leq
        \pn{\prod_{\PupilLayerIndex = 1}^{\PupilLayerCount} \PupilActivationLipshitzConstant \norm{\PupilWeightMatrix_\PupilLayerIndex}_F}\norm{\gConvInputAsMLP[1]}_F,
    \end{aligned}
$$
where \PupilAsMLPName{} is the MLP representation of \PupilName{}, $\gConvInputAsMLP[1]$ is from Definition~\ref{def:results:input-for-mlp-non-equivariant}, $\PupilActivationLipshitzConstant \geq 1$ bounds the largest Lipshitz constant of the activations used by \PupilName{}, and $\set{\PupilWeightMatrix_\PupilLayerIndex}_{\PupilLayerIndex = 1}^{\PupilLayerCount}$ are the weight matrices of \PupilAsMLPName{}.
\end{lemma}

\begin{proof}
    The activations of \PupilName{} are either $\PupilActivationName_\PupilLayerIndex(x) = \tanh(x)$ or $\PupilActivationName_\PupilLayerIndex\pn{x} = x$ for $\PupilLayerIndex \in \set{1,\dots,\PupilLayerCount}$.
    Both have Lipshitz constant $\PupilActivationLipshitzConstant = 1$ and satisfy $\PupilActivationName\pn{0} = 0$.
    By \lemmaref{lemma:results:conv1d-as-linear}, $\PupilName$ can be expressed as an MLP \PupilAsMLPName{}.
    Adapting Lemma~B.1 of \cite{pmlr-v235-karczewski24a},
$$
    \begin{aligned}
        \norm{\PupilAsMLPName_{\PupilLayerCount}\pn{\gConvInputAsMLP[1]}}_2 & =
        \norm{\PupilActivationName_{\PupilLayerCount}\pn{\PupilAsMLPName_{\PupilLayerCount - 1}\pn{\gConvInputAsMLP[1]}\PupilWeightMatrix_{\PupilLayerCount}}}_2 =
        \norm{\PupilActivationName_{\PupilLayerCount}\pn{\PupilAsMLPName_{\PupilLayerCount - 1}\pn{\gConvInputAsMLP[1]}\PupilWeightMatrix_{\PupilLayerCount}} - \PupilActivationName_{\PupilLayerCount}\pn{\mzero}}_2
        \\ & \leq
        \PupilActivationLipshitzConstant\norm{\PupilAsMLPName_{\PupilLayerCount - 1}\pn{\gConvInputAsMLP[1]}\PupilWeightMatrix_{\PupilLayerCount}}_2 \leq
        \PupilActivationLipshitzConstant\norm{\PupilWeightMatrix_{\PupilLayerCount}}_2\norm{\PupilAsMLPName_{\PupilLayerCount - 1}\pn{\gConvInputAsMLP[1]}}_2
        \\ & \leq
        \dots \leq
        \pn{\prod_{\PupilLayerIndex = 1}^{\PupilLayerCount} \PupilActivationLipshitzConstant\norm{\PupilWeightMatrix_\PupilLayerIndex}_2}\norm{\gConvInputAsMLP[1]}_2
        \\ & =
        \pn{\prod_{\PupilLayerIndex = 1}^{\PupilLayerCount} \PupilActivationLipshitzConstant\norm{\PupilWeightMatrix_\PupilLayerIndex}_2}\norm{\gConvInputAsMLP[1]}_F \quad \text{by Definition~\ref{def:results:input-for-mlp-non-equivariant}}
        \\ & \leq
        \pn{\prod_{\PupilLayerIndex = 1}^{\PupilLayerCount} \PupilActivationLipshitzConstant\norm{\PupilWeightMatrix_\PupilLayerIndex}_F}\norm{\gConvInputAsMLP[1]}_F \quad \text{by \Lemmaref{lemma:results:frobenius-geq-spectral}}
    \end{aligned}
$$
\end{proof}

\begin{lemma}
    \label{lemma:results:etdagnn-bounded}
    Let \PupilName{} be an \PupilETDAGNN{}.
    Then,
    $$
    \begin{aligned}
        \norm{\PupilName\pn{\gConvInput[1]}} =
        \norm{\PupilAsMLPName{}\pn{\gConvInputAsMLP[1]}} \leq
        \pn{\prod_{\PupilLayerIndex = 1}^{\PupilLayerCount} \PupilActivationLipshitzConstant \norm{\PupilWeightMatrix_\PupilLayerIndex}_F}\norm{\gConvInputAsMLP[1]}_F,
    \end{aligned}
    $$
    where \PupilAsMLPName{} is the MLP representation of \PupilETDAGNN{}, $\gConvInputAsMLP[1]$ is defined in Definition~\ref{def:results:input-for-mlp-equivariant}, $\PupilActivationLipshitzConstant \geq 1$ bounds the largest Lipshitz constant of the activations in \PupilETDAGNN{}, and $\set{\PupilWeightMatrix_\PupilLayerIndex}_{\PupilLayerIndex = 1}^{\PupilLayerCount}$ are the weight matrices of \PupilAsMLPName{}.
\end{lemma}

\begin{proof}
    \PupilName{} uses the \OrthogonalGroup{n} equivariant activations $\bm\PupilActivationName_\PupilLayerIndex \in \set{\PupilETDAGNNActivationScaleLog, \PupilETDAGNNActivationScaleTanh, \vx \mapsto \vx}$ for $\PupilLayerIndex \in \set{1,\dots,\PupilLayerCount}$.
    Let \PupilActivationLipshitzConstant{} be the maximum of their Lipshitz constants.
    By \Lemmaref{lemma:results:equivariant-block-matrix-func-lipshitz-multiple-output-channels}, \ref{lemma:results:etdagnn-sigma0-lipshitz}, and \ref{lemma:results:etdagnn-sigma1-lipshitz}, $\bm\PupilActivationName_\PupilLayerIndex$ is an \OrthogonalGroup{2} equivariant activation with global Lipshitz \PupilActivationLipshitzConstant{} in the Frobenius norm, and $\bm\PupilActivationName_\PupilLayerIndex\pn{\mzero} = \mzero$ by \Lemmaref{lemma:results:equivariant-zero-at-zero}.
    Next, by \Lemmaref{lemma:results:eqconv-as-linear}, \PupilName{} can be expressed as an MLP \PupilAsMLPName{}.
    Adapting Lemma~B.1 of \cite{pmlr-v235-karczewski24a},
$$
    \begin{aligned}
        \norm{\PupilAsMLPName_{\PupilLayerCount}\pn{\gConvInputAsMLP[1]}}_2 & =
        \norm{\PupilActivationName_{\PupilLayerCount}\pn{\PupilAsMLPName_{\PupilLayerCount - 1}\pn{\gConvInputAsMLP[1]}\PupilWeightMatrix_{\PupilLayerCount}}}_2 =
        \norm{\PupilActivationName_{\PupilLayerCount}\pn{\PupilAsMLPName_{\PupilLayerCount - 1}\pn{\gConvInputAsMLP[1]}\PupilWeightMatrix_{\PupilLayerCount}} - \PupilActivationName_{\PupilLayerCount}\pn{\mzero}}_2
        \\ & \leq
        \PupilActivationLipshitzConstant\norm{\PupilAsMLPName_{\PupilLayerCount - 1}\pn{\gConvInputAsMLP[1]}\PupilWeightMatrix_{\PupilLayerCount}}_2 \leq
        \PupilActivationLipshitzConstant\norm{\PupilWeightMatrix_{\PupilLayerCount}}_2\norm{\PupilAsMLPName_{\PupilLayerCount - 1}\pn{\gConvInputAsMLP[1]}}_2
        \\ & \leq
        \dots \leq
        \pn{\prod_{\PupilLayerIndex = 1}^{\PupilLayerCount} \PupilActivationLipshitzConstant\norm{\PupilWeightMatrix_\PupilLayerIndex}_2}\norm{\gConvInputAsMLP[1]}_2
        \\ & \leq
        \pn{\prod_{\PupilLayerIndex = 1}^{\PupilLayerCount} \PupilActivationLipshitzConstant\norm{\PupilWeightMatrix_\PupilLayerIndex}_F}\norm{\gConvInputAsMLP[1]}_F \quad \text{by \Lemmaref{lemma:results:frobenius-geq-spectral}}
    \end{aligned}
$$
\end{proof}

\begin{lemma}
    \label{lemma:results:scoring-model-bounded-with-bounded-input}
    Let the scoring model $\PupilGraphScorer$ be as given in \eqref{eq:generalization-gap:scoring-model} equivalently written to take an \GeneralizationFastForwardBC{} tuple as input.
    Using assumption~\ref{assumption:results:dagger-dataset-non-equivariant-bounded},
$$
    \begin{aligned}
        \PupilGraphScorer\pn{\gDAggerDatasetDatum} \leq \PupilGraphScorerBias^2 + \pn{1 + \pn{\prod_{\PupilLayerIndex = 1}^{\PupilLayerCount} \PupilActivationLipshitzConstant\norm{\PupilWeightMatrix_\PupilLayerIndex}_F}^2}\DAggerDatumBound^2.
    \end{aligned}
$$
\end{lemma}
\begin{proof}
    Using Assumption~\ref{assumption:results:dagger-dataset-non-equivariant-bounded}, and \Lemmaref{lemma:results:tdagnn-bounded} and \ref{lemma:results:etdagnn-bounded},
$$
    \begin{aligned}
        \PupilGraphScorer\pn{\gDAggerDatasetDatum} & \leq \PupilGraphScorerBias^2 + \frac{1}{\FlockAgentCount}\sum_{\FlockAgentIndex = 1}^\FlockAgentCount \norm{\gExpertName\pn{\gExpertInput}}^2 + \norm{\PupilName\pn{\gConvInput[1]}}^2 \\
        & = \PupilGraphScorerBias^2 + \frac{1}{\FlockAgentCount}\sum_{\FlockAgentIndex = 1}^\FlockAgentCount \norm{\gExpertName\pn{\gExpertInput}}^2 + \norm{\PupilAsMLPName\pn{\gConvInputAsMLP[1]}}^2 \\
        & \leq \PupilGraphScorerBias^2 + \frac{1}{\FlockAgentCount}\sum_{\FlockAgentIndex = 1}^\FlockAgentCount \norm{\gExpertName\pn{\gExpertInput}}_F^2 + \pn{\prod_{\PupilLayerIndex = 1}^{\PupilLayerCount} \PupilActivationLipshitzConstant\norm{\PupilWeightMatrix_\PupilLayerIndex}_F}^2\norm{\gConvInputAsMLP[1]}_F^2 \\
        & \leq \PupilGraphScorerBias^2 + \frac{1}{\FlockAgentCount}\sum_{\FlockAgentIndex = 1}^\FlockAgentCount \beta^2 + \pn{\prod_{\PupilLayerIndex = 1}^{\PupilLayerCount} \PupilActivationLipshitzConstant\norm{\PupilWeightMatrix_\PupilLayerIndex}_F}^2\DAggerDatumBound^2 \\
        & \leq \PupilGraphScorerBias^2 + \pn{1 + \pn{\prod_{\PupilLayerIndex = 1}^{\PupilLayerCount} \PupilActivationLipshitzConstant\norm{\PupilWeightMatrix_\PupilLayerIndex}_F}^2}\DAggerDatumBound^2.
    \end{aligned}
$$
\end{proof}

The ML controllers are parameterized by weight matrices, which have a bounded covering number (Lemma~8 of \cite{chenGeneralizationBoundsFamily2019}).
Using Lemma~G.1 of \cite{pmlr-v235-karczewski24a}, if we can show the ML controllers are Lipshitz with respect to the weight matrices, then we can bound the covering number of the ML controllers.
Next, we show that the ML controllers and their scoring function are Lipshitz. 

\begin{lemma}[Lipshitz continuity of 
TDAGNN] 
    \label{lemma:results:tdagnn-lipshitz-wrt-weights}
    Let \PupilName{} be either \PupilTDAGNN{}, \PupilTDAGNNTF{}, \PupilTDAGNNTFMu{}, or \PupilETDAGNN{} with an MLP representation $\PupilAsMLPName$.
    Take two functions $\PupilAsMLPName\pn{\gConvInputAsMLP[1]; \PupilWeightMatrixCollection}$ and $\PupilName\pn{\gConvInputAsMLP[1]; \tilde{\PupilWeightMatrixCollection}}$ with \PupilLayerCount{} layers where \PupilWeightMatrixCollection{} and $\tilde{\PupilWeightMatrixCollection}$ are collections of weight matrices. 
    Let $\PupilWeightMatrixBound_\PupilLayerIndex \geq 1$ s.t. $\max\set{\norm{\PupilWeightMatrix_\PupilLayerIndex}_F, \norm{\tilde{\PupilWeightMatrix}_\PupilLayerIndex}_F} \leq \PupilWeightMatrixBound_\PupilLayerIndex$.
    Then,
$$
    \begin{aligned}
        \norm{\PupilAsMLPName\pn{\gConvInputAsMLP[1]; \PupilWeightMatrixCollection} - \PupilAsMLPName\pn{\gConvInputAsMLP[1]; \tilde{\PupilWeightMatrixCollection}}} \leq \norm{\gConvInputAsMLP[1]}_F\pn{\prod_{\PupilLayerIndex = 1}^{\PupilLayerCount} \PupilActivationLipshitzConstant\PupilWeightMatrixBound_\PupilLayerIndex}\sum_{\PupilLayerIndex = 1}^{\PupilLayerCount} \norm{\PupilWeightMatrix_\PupilLayerIndex - \tilde{\PupilWeightMatrix}_\PupilLayerIndex}_F,
    \end{aligned}
$$
    where $\PupilActivationLipshitzConstant \geq 1$ bounds the largest Lipshitz constant of the activations used by \PupilAsMLPName{}.
\end{lemma}

\begin{proof}
    The proof follows the proof of Lemma B.2 of \cite{pmlr-v235-karczewski24a} where $\UtilNormWithPlaceholder = \UtilNormWithPlaceholder_F$.
\end{proof}

\begin{lemma}
    \label{lemma:results:scoring-function-locally-lipshitz}
    Let \PupilAsMLPName{} be the MLP representation of \PupilName{} and $\PupilGraphScorer$ be as defined in \eqref{eq:generalization-gap:scoring-model}.
    Consider two parameterizations of \PupilAsMLPName{} and $\PupilGraphScorer$: $\set{\PupilGraphScorerBias,\PupilWeightMatrixCollection},\set{\tilde\PupilWeightMatrixCollection,\tilde\PupilGraphScorerBias}$.
    Let $\PupilWeightMatrixBound_\PupilLayerIndex \geq 1$ and $\PupilWeightMatrixBound_\PupilGraphScorer \geq 1$ where $\max\set{\norm{\PupilWeightMatrix_\PupilLayerIndex}_F, \norm{\tilde{\PupilWeightMatrix}_\PupilLayerIndex}_F} \leq \PupilWeightMatrixBound_\PupilLayerIndex$ and $\max\set{\abs{\PupilGraphScorerBias}, \abs{\tilde{\PupilGraphScorerBias}}} \leq \PupilWeightMatrixBound_\PupilGraphScorer$.
    By Assumption~\ref{assumption:results:dagger-dataset-non-equivariant-bounded},
$$
    \begin{aligned}
        &\ \ \abs{\PupilGraphScorer\pn{\gDAggerDatasetDatum; \PupilWeightMatrixCollection, \PupilGraphScorerBias} - \PupilGraphScorer\pn{\gDAggerDatasetDatum; \tilde{\PupilWeightMatrixCollection}, \tilde\PupilGraphScorerBias}}\\
        & \leq 2\PupilWeightMatrixBound_\PupilGraphScorer\abs{\PupilGraphScorerBias - \tilde\PupilGraphScorerBias} + 2\pn{1 +  \PupilLipshitzConstantName}\DAggerDatumBound^2\PupilLipshitzConstantName\sum_{\PupilLayerIndex = 1}^{\PupilLayerCount} \norm{\PupilWeightMatrix_\PupilLayerIndex - \tilde{\PupilWeightMatrix}_\PupilLayerIndex}_F.
    \end{aligned}
$$
    where $\PupilLipshitzConstantName = \PupilLipshitzConstantDefn$.
\end{lemma}
\begin{proof}
$$
    \begin{aligned}
        &\ \ \abs{\PupilGraphScorer\pn{\cdot; \PupilWeightMatrixCollection} - \PupilGraphScorer\pn{\cdot; \tilde{\PupilWeightMatrixCollection}}}\\ & = \abs{\PupilGraphScorerBias^2 + \frac{1}{\FlockAgentCount}\sum_{\FlockAgentIndex = 1}^\FlockAgentCount \norm{\gExpertName\pn{\gExpertInput} - \PupilAsMLPName\pn{\gConvInputAsMLP[1]; \PupilWeightMatrixCollection}}^2 - \PupilGraphScorerBias^2 - \frac{1}{\FlockAgentCount}\sum_{\FlockAgentIndex = 1}^\FlockAgentCount\norm{\gExpertName\pn{\gExpertInput} - \PupilAsMLPName\pn{\gConvInputAsMLP[1]; \tilde{\PupilWeightMatrixCollection}}}^2} \\
        & \leq \abs{\PupilGraphScorerBias^2 - \tilde{\PupilGraphScorerBias}^2} + \frac{1}{\FlockAgentCount}\sum_{\FlockAgentIndex = 1}^\FlockAgentCount \abs{\norm{\gExpertName\pn{\gExpertInput} - \PupilAsMLPName\pn{\gConvInputAsMLP[1]; \PupilWeightMatrixCollection}}^2 - \norm{\gExpertName\pn{\gExpertInput} - \PupilAsMLPName\pn{\gConvInputAsMLP[1]; \tilde{\PupilWeightMatrixCollection}}}^2}.
    \end{aligned}
$$
    Using the fact $\abs{\norm{\vx - \va}^2 - \norm{\vx - \vb}^2} \leq \pn{2\norm{\vx} + \norm{\va} + \norm{\vb}}\norm{\va - \vb}$, $\norm{\gExpertName\pn{\gExpertInput}} \leq \DAggerDatumBound$ by Assumption~\ref{assumption:results:dagger-dataset-non-equivariant-bounded}, and $\norm{\PupilAsMLPName\pn{\gConvInputAsMLP[1]}} \leq \pn{\prod_{\PupilLayerIndex = 1}^{\PupilLayerCount} \PupilActivationLipshitzConstant\norm{\PupilWeightMatrix_\PupilLayerIndex}_F}\norm{\gConvInputAsMLP[1]}_F \leq \PupilLipshitzConstantName\norm{\gConvInputAsMLP[1]}_F$ by \Lemmaref{lemma:results:tdagnn-bounded} or \ref{lemma:results:etdagnn-bounded},
$$
    \begin{aligned}
        \abs{\PupilGraphScorer\pn{\cdot; \PupilWeightMatrixCollection} - \PupilGraphScorer\pn{\cdot; \tilde{\PupilWeightMatrixCollection}}} & \leq 2\PupilWeightMatrixBound_\PupilGraphScorer\abs{\PupilGraphScorerBias - \tilde{\PupilGraphScorerBias}} + \frac{2}{\FlockAgentCount} \sum_{\FlockAgentIndex = 1}^\FlockAgentCount \pn{\DAggerDatumBound + \PupilLipshitzConstantName\norm{\gConvInputAsMLP[1]}_F}\norm{\PupilAsMLPName\pn{\gConvInputAsMLP[1]; \PupilWeightMatrixCollection} - \PupilAsMLPName\pn{\gConvInputAsMLP[1]; \tilde{\PupilWeightMatrixCollection}}}.
    \end{aligned}
$$
    By Lemma~\ref{lemma:results:tdagnn-lipshitz-wrt-weights},
$$
    \begin{aligned}
       &\ \  \abs{\PupilGraphScorer\pn{\cdot; \PupilWeightMatrixCollection} - \PupilGraphScorer\pn{\cdot; \tilde{\PupilWeightMatrixCollection}}}\\ & \leq 2\PupilWeightMatrixBound_\PupilGraphScorer\abs{\PupilGraphScorerBias - \tilde{\PupilGraphScorerBias}} + \frac{2}{\FlockAgentCount} \sum_{\FlockAgentIndex = 1}^\FlockAgentCount \pn{\DAggerDatumBound + \PupilLipshitzConstantName\norm{\gConvInputAsMLP[1]}_F}\norm{\gConvInputAsMLP[1]}_F\PupilLipshitzConstantName\sum_{\PupilLayerIndex = 1}^{\PupilLayerCount}\norm{\PupilWeightMatrix_\PupilLayerIndex - \tilde\PupilWeightMatrix_\PupilLayerIndex}_F.
    \end{aligned}
$$
    Using that $\norm{\gConvInputAsMLP[1]}_F \leq \DAggerDatumBound$ by Assumption~\ref{assumption:results:dagger-dataset-non-equivariant-bounded},
$$
    \begin{aligned}
        \abs{\PupilGraphScorer\pn{\cdot; \PupilWeightMatrixCollection} - \PupilGraphScorer\pn{\cdot; \tilde{\PupilWeightMatrixCollection}}} & \leq 2\PupilWeightMatrixBound_\PupilGraphScorer\abs{\PupilGraphScorerBias - \tilde{\PupilGraphScorerBias}} + \frac{2}{\FlockAgentCount} \sum_{\FlockAgentIndex = 1}^\FlockAgentCount \pn{\DAggerDatumBound + \PupilLipshitzConstantName\DAggerDatumBound}\DAggerDatumBound\PupilLipshitzConstantName\sum_{\PupilLayerIndex = 1}^{\PupilLayerCount}\norm{\PupilWeightMatrix_\PupilLayerIndex - \tilde\PupilWeightMatrix_\PupilLayerIndex}_F \\
        & = 2\PupilWeightMatrixBound_\PupilGraphScorer\abs{\PupilGraphScorerBias - \tilde{\PupilGraphScorerBias}} + 2\pn{1 + \PupilLipshitzConstantName}\DAggerDatumBound^2\PupilLipshitzConstantName\sum_{\PupilLayerIndex = 1}^{\PupilLayerCount}\norm{\PupilWeightMatrix_\PupilLayerIndex - \tilde\PupilWeightMatrix_\PupilLayerIndex}_F.
    \end{aligned}
$$
\end{proof}

With the above supporting lemmas, we are ready to prove the following result:

\begin{proposition}[Generalization bound of TDAGNN]
Let $P$ be the probability distribution over tuples $\pn{\gDAggerDatasetDatum}$ induced by \GeneralizationFastForwardBC{}. 
Let $\cL\pn{y} = \min\set{1, y/\LossFunctionMSENormalizationConstant}$ for $\LossFunctionMSENormalizationConstant > 0$ be the loss function.
Let $\set{\PupilWeightMatrix_\PupilLayerIndex}_{\PupilLayerIndex = 1}^{\PupilLayerCount}$ be the weights of the MLP representation $\PupilAsMLPName$ of $\PupilName$ given by \Lemmaref{lemma:results:conv1d-as-linear} or \ref{lemma:results:eqconv-as-linear}, and let $\PupilGraphScorerBias$ of $\PupilGraphScorer$ in \eqref{eq:generalization-gap:scoring-model} such that $\PupilGraphScorerBias \in \spn{0, \sqrt{\LossFunctionMSENormalizationConstant}}$.
For any $\delta > 0$, with probability at least $1 - \delta$ over choosing a batch $\GeneralizationSampleName$ of $\GeneralizationSampleSize$ tuples sampled from $P$, the following bound holds:
$$
\begin{aligned}
   &\ \  \gGeneralizationGap\pn{\PupilGraphScorer} \leq \frac{8}{\GeneralizationSampleSize} \\
   & + \frac{48\PupilWeightMatrixLargestDim}{\sqrt{\GeneralizationSampleSize}}\sqrt{\pn{3\PupilLayerCount + 1}\ln\pn{10\PupilLayerCount\DAggerDatumBound\PupilActivationLipshitzConstant^{\PupilLayerCount}\sqrt{\PupilWeightMatrixLargestDim\GeneralizationSampleSize\LossFunctionMSENormalizationConstant}} + \pn{2\PupilLayerCount + 3}\sum_{\PupilLayerIndex = 1}^{\PupilLayerCount} \ln\pn{\max\set{1,\norm{\PupilWeightMatrix_\PupilLayerIndex}_F}}} + 3\sqrt{\frac{\ln\pn{\frac{2}{\delta}}}{2\GeneralizationSampleSize}}.
\end{aligned}
$$
\end{proposition}

\begin{proof}
    Let $\GeneralizationFunctionsLearners = \set{\PupilGraphScorer\pn{\cdot; \PupilWeightMatrixCollection, \PupilGraphScorerBias}: \PupilWeightMatrixCollection = \set{\PupilWeightMatrix_\PupilLayerIndex}_{\PupilLayerIndex = 1}^{\PupilLayerCount}, \norm{\PupilWeightMatrix_\PupilLayerIndex}_F \leq \PupilWeightMatrixBound_\PupilLayerIndex, \abs{\PupilGraphScorerBias} \leq \PupilWeightMatrixBound_\PupilGraphScorer}$, where $\PupilWeightMatrixBound_\PupilGraphScorer = \sqrt{\LossFunctionMSENormalizationConstant}$.
    Define the set of datum-to-loss functions $\GeneralizationFunctionsDatumToLoss = \set{\pn{\gDAggerDatasetDatum} \mapsto \LossFunctionName\pn{\PupilGraphScorer\pn{\gDAggerDatasetDatum}}: \PupilGraphScorer \in \GeneralizationFunctionsLearners}$.
    We follow the steps from proof of Proposition 4.1 from \cite{pmlr-v235-karczewski24a}. 
    Applying these steps requires that we find an $\GeneralizationFunctionDatumToLossHatZero \in \GeneralizationFunctionsDatumToLossHat$ where
$$
    \begin{aligned}
        \GeneralizationFunctionsDatumToLossHat = \set{\pn{\gDAggerDatasetDatum} \mapsto 1 - \LossFunctionName\pn{\PupilGraphScorer\pn{\gDAggerDatasetDatum}}: \PupilGraphScorer \in \GeneralizationFunctionsLearners},
    \end{aligned}
$$
    and $\GeneralizationFunctionDatumToLossHatZero\pn{\gDAggerDatasetDatum} = 0$ for all $\pn{\gDAggerDatasetDatum}$ satisfying Assumption~\ref{assumption:results:dagger-dataset-non-equivariant-bounded}.
    By definition of $\LossFunctionName$, we can construct $\GeneralizationFunctionDatumToLossHatZero$ by finding $\PupilGraphScorer$ s.t. $\PupilGraphScorer\pn{\gDAggerDatasetDatum} \geq \LossFunctionMSENormalizationConstant$ for all $\pn{\gDAggerDatasetDatum}$.
    This can be achieved by setting the kernel weights and bias weights (if present) of $\PupilName$ to zero, and setting $\PupilGraphScorerBias$ of $\PupilGraphScorer$ in \eqref{eq:generalization-gap:scoring-model} to $\PupilGraphScorerBias = \sqrt{\LossFunctionMSENormalizationConstant}$. Therefore, we choose $\GeneralizationFunctionDatumToLossHatZero\pn{\cdot} = 1 - \LossFunctionName\pn{\PupilGraphScorer\pn{\cdot; \set{\mzero}_{\PupilLayerIndex = 1}^{\PupilLayerCount}, \sqrt{\LossFunctionMSENormalizationConstant}}}$. With $\GeneralizationFunctionDatumToLossHatZero$ found, we follow 
    \cite{pmlr-v235-karczewski24a} to obtain
$$
    \begin{aligned}
        \GeneralizationEmpiricalRademacherComplexity\pn{\GeneralizationFunctionsDatumToLoss} \leq \frac{4}{\GeneralizationSampleSize} + \frac{24}{\GeneralizationSampleSize}\sqrt{\ln\GeneralizationCoveringNumber\pn{\GeneralizationFunctionsLearners, \frac{1}{2\sqrt{\GeneralizationSampleSize}}, \UtilNormWithPlaceholder_F}}.
    \end{aligned}
$$
    Next, we adapt the steps in \cite{pmlr-v235-karczewski24a} to bound $\ln\GeneralizationCoveringNumber\pn{\GeneralizationFunctionsLearners, \frac{1}{2\sqrt{\GeneralizationSampleSize}}, \UtilNormWithPlaceholder_\infty}$.
    Using \Lemmaref{lemma:results:scoring-function-locally-lipshitz}, we bound the supremum distance between functions in $\GeneralizationFunctionsLearners$ by a function of the distance between their weight matrices:
$$
    \begin{aligned}
        & \ \ \norm{\PupilGraphScorer\pn{\cdot; \PupilWeightMatrixCollection, \PupilGraphScorerBias} - \PupilGraphScorer\pn{\cdot; \tilde{\PupilWeightMatrixCollection}, \tilde\PupilGraphScorerBias}}_\infty\\ & = \sup_{\pn{\gDAggerDatasetDatum} \in \GeneralizationDatasetDistribution} \abs{\PupilGraphScorer\pn{\gDAggerDatasetDatum;\PupilWeightMatrixCollection} - \PupilGraphScorer\pn{\gDAggerDatasetDatum;\tilde{\PupilWeightMatrixCollection}}} \\
        & \leq 2\PupilWeightMatrixBound_\PupilGraphScorer\abs{\PupilGraphScorerBias - \tilde{\PupilGraphScorerBias}} + 2\pn{1 + \PupilLipshitzConstantName}\DAggerDatumBound^2\PupilLipshitzConstantName\sum_{\PupilLayerIndex = 1}^{\PupilLayerCount} \norm{\PupilWeightMatrix_\PupilLayerIndex - \tilde{\PupilWeightMatrix}_\PupilLayerIndex}_F
        \\
        & \leq \underbrace{2\pn{\PupilWeightMatrixBound_\PupilGraphScorer + \pn{1 + \PupilLipshitzConstantName}\DAggerDatumBound^2\PupilLipshitzConstantName}}_{:= \PupilGraphScorerLipshitzConstantName}\pn{\abs{\PupilGraphScorerBias - \tilde{\PupilGraphScorerBias}} +\sum_{\PupilLayerIndex = 1}^{\PupilLayerCount} \norm{\PupilWeightMatrix_\PupilLayerIndex - \tilde{\PupilWeightMatrix}_\PupilLayerIndex}_F}.
    \end{aligned}
$$
    Using that bound, Lemma~G.1 from \cite{pmlr-v235-karczewski24a} bounds the covering number of $\GeneralizationFunctionsLearners$ by the covering number of the weight matrices.
$$
    \begin{aligned}
        \ln\GeneralizationCoveringNumber\pn{\GeneralizationFunctionsLearners, r, \UtilNormWithPlaceholder_\infty} \leq \ln\GeneralizationCoveringNumber\pn{\PupilGraphScorerBias \in \spn{0, \sqrt{\LossFunctionMSENormalizationConstant}}, \frac{r}{\PupilLayerCount\PupilGraphScorerLipshitzConstantName}, \abs{\;\cdot\;}} + \sum_{\PupilLayerIndex = 1}^{\PupilLayerCount} \ln\GeneralizationCoveringNumber\pn{\PupilWeightMatrixCollection_\PupilLayerIndex, \frac{r}{\PupilLayerCount\PupilGraphScorerLipshitzConstantName}, \UtilNormWithPlaceholder_F},
    \end{aligned}
$$
    where $\PupilWeightMatrixCollection_\PupilLayerIndex$ is the set of possible matrices for $\PupilWeightMatrix_\PupilLayerIndex$. Using Lemma~3.2 from \cite{chenGeneralizationBoundsFamily2019},
$$
    \begin{aligned}
        \ln\GeneralizationCoveringNumber\pn{\PupilWeightMatrixCollection_\PupilLayerIndex, \frac{r}{\PupilLayerCount\PupilGraphScorerLipshitzConstantName}, \UtilNormWithPlaceholder_F} \leq \PupilWeightMatrixLargestDim^2\ln\pn{1 + 2\frac{\PupilLayerCount\PupilGraphScorerLipshitzConstantName\PupilWeightMatrixBound_\PupilLayerIndex\sqrt{\PupilWeightMatrixLargestDim}}{r}},
    \end{aligned}
$$
    where $\PupilWeightMatrixLargestDim = \max_\PupilLayerIndex \dim\pn{\PupilWeightMatrix_\PupilLayerIndex}$.
    Choosing $r = \frac{1}{2\sqrt{\GeneralizationSampleSize}}$,
$$
    \begin{aligned}
        \ln\GeneralizationCoveringNumber\pn{\GeneralizationFunctionsLearners, \frac{1}{2\sqrt{\GeneralizationSampleSize}}, \UtilNormWithPlaceholder_\infty} & \leq
        \ln\pn{1 + 4\PupilLayerCount\PupilGraphScorerLipshitzConstantName\PupilWeightMatrixBound_\PupilGraphScorer\sqrt{\GeneralizationSampleSize}} +
        \PupilWeightMatrixLargestDim^2\sum_{\PupilLayerIndex = 1}^{\PupilLayerCount} \ln\pn{1 + 4\PupilLayerCount\PupilGraphScorerLipshitzConstantName\PupilWeightMatrixBound_\PupilLayerIndex\sqrt{\PupilWeightMatrixLargestDim\GeneralizationSampleSize}}
        \\ & \leq
        \ln\pn{5\PupilLayerCount\PupilGraphScorerLipshitzConstantName\PupilWeightMatrixBound_\PupilGraphScorer\sqrt{\GeneralizationSampleSize}} +
        \PupilWeightMatrixLargestDim^2\sum_{\PupilLayerIndex = 1}^{\PupilLayerCount} \ln\pn{5\PupilLayerCount\PupilGraphScorerLipshitzConstantName\PupilWeightMatrixBound_\PupilLayerIndex\sqrt{\PupilWeightMatrixLargestDim\GeneralizationSampleSize}}
        \\ & \leq
        \PupilWeightMatrixLargestDim^2\ln\pn{5\PupilLayerCount\PupilGraphScorerLipshitzConstantName\PupilWeightMatrixBound_\PupilGraphScorer\sqrt{\GeneralizationSampleSize}} +
        \PupilWeightMatrixLargestDim^2\sum_{\PupilLayerIndex = 1}^{\PupilLayerCount} \ln\pn{5\PupilLayerCount\PupilGraphScorerLipshitzConstantName\PupilWeightMatrixBound_\PupilLayerIndex\sqrt{\PupilWeightMatrixLargestDim\GeneralizationSampleSize}}.
    \end{aligned}
$$
    Plugging into the bound for empirical Rademacher complexity,
$$
    \begin{aligned}
        \GeneralizationEmpiricalRademacherComplexity\pn{\GeneralizationFunctionsDatumToLoss} \leq \frac{4}{\GeneralizationSampleSize} + \frac{24\PupilWeightMatrixLargestDim}{\sqrt{\GeneralizationSampleSize}}\sqrt{\underbrace{\ln\pn{5\PupilLayerCount\PupilGraphScorerLipshitzConstantName\PupilWeightMatrixBound_\PupilGraphScorer\sqrt{\GeneralizationSampleSize}} + \sum_{\PupilLayerIndex = 1}^{\PupilLayerCount} \ln\pn{5\PupilLayerCount\PupilGraphScorerLipshitzConstantName\PupilWeightMatrixBound_\PupilLayerIndex\sqrt{\PupilWeightMatrixLargestDim\GeneralizationSampleSize}}}_{:= \Sigma}}.
    \end{aligned}
$$
    Next, we simplify the bound.
    Using logarithm identities,
$$
    \begin{aligned}
        \Sigma & =
        \frac{\PupilLayerCount}{2}\ln\pn{\PupilWeightMatrixLargestDim} +
        \pn{\PupilLayerCount + 1}\spn{\ln\pn{5\PupilLayerCount\sqrt{\GeneralizationSampleSize}} + \ln\pn{\PupilGraphScorerLipshitzConstantName}} + \ln\pn{\PupilWeightMatrixBound_\PupilGraphScorer} + \sum_{\PupilLayerIndex = 1}^{\PupilLayerCount} \ln\pn{\PupilWeightMatrixBound_\PupilLayerIndex}.
    \end{aligned}
$$
    Finding an upper bound for $\ln\pn{\PupilGraphScorerLipshitzConstantName}$,
$$
    \begin{aligned}
        \ln\pn{\PupilGraphScorerLipshitzConstantName} & = \ln\pn{2} + \ln\pn{\PupilWeightMatrixBound_\PupilGraphScorer + \pn{1 + \PupilLipshitzConstantName}\DAggerDatumBound^2\PupilLipshitzConstantName} \\
        & \leq \ln\pn{2} + \ln\pn{\PupilWeightMatrixBound_\PupilGraphScorer + \PupilWeightMatrixBound_\PupilGraphScorer\pn{1 + \PupilLipshitzConstantName}\DAggerDatumBound^2\PupilLipshitzConstantName} \\
        & \leq \ln\pn{2} + \ln\pn{\PupilWeightMatrixBound_\PupilGraphScorer} + \ln\pn{2\pn{1 + \PupilLipshitzConstantName}\DAggerDatumBound^2\PupilLipshitzConstantName} \\
        & = 2\ln\pn{2} + \ln\pn{\PupilWeightMatrixBound_\PupilGraphScorer} + \ln\pn{1 + \PupilLipshitzConstantName} + 2\ln\pn{\DAggerDatumBound} + \ln\pn{\PupilLipshitzConstantName} \\
        & \leq 2\ln\pn{2} + \ln\pn{\PupilWeightMatrixBound_\PupilGraphScorer} + \ln\pn{2\PupilLipshitzConstantName} + 2\ln\pn{\DAggerDatumBound} + \ln\pn{\PupilLipshitzConstantName} \\
        & = 3\ln\pn{2} + \ln\pn{\PupilWeightMatrixBound_\PupilGraphScorer} + 2\ln\pn{\PupilLipshitzConstantName} + 2\ln\pn{\DAggerDatumBound} \\
        & = 3\ln\pn{2} + \ln\pn{\PupilWeightMatrixBound_\PupilGraphScorer} + 2\ln\pn{\PupilLipshitzConstantDefn} + 2\ln\pn{\DAggerDatumBound} \\
        & = 3\ln\pn{2} + \ln\pn{\PupilWeightMatrixBound_\PupilGraphScorer} + 2\ln\pn{\DAggerDatumBound} + 2\ln\pn{\PupilActivationLipshitzConstant^{\PupilLayerCount}} + 2\sum_{\PupilLayerIndex = 1}^{\PupilLayerCount} \ln\pn{\PupilWeightMatrixBound_\PupilLayerIndex}.
    \end{aligned}
$$
    Combining these expressions, we find an upper bound for $\Sigma$:
$$
    \begin{aligned}
        \Sigma
        & \leq \frac{\PupilLayerCount}{2}\ln\pn{\PupilWeightMatrixLargestDim} + \pn{\PupilLayerCount + 1}\spn{\ln\pn{5\PupilLayerCount\sqrt{\GeneralizationSampleSize}} + 3\ln\pn{2} + 2\ln\pn{\DAggerDatumBound} + 2\ln\pn{\PupilActivationLipshitzConstant^{\PupilLayerCount}}} + \\ 
&\qquad\qquad        \pn{\PupilLayerCount + 2}\ln\pn{\PupilWeightMatrixBound_\PupilGraphScorer} + \pn{2\PupilLayerCount + 3}\sum_{\PupilLayerIndex = 1}^{\PupilLayerCount} \ln\pn{\PupilWeightMatrixBound_\PupilLayerIndex} \\
        & \leq \frac{\PupilLayerCount}{2}\ln\pn{\PupilWeightMatrixLargestDim} + 3\pn{\PupilLayerCount + 1}\spn{\ln\pn{5\PupilLayerCount\sqrt{\GeneralizationSampleSize}} + \ln\pn{2} + \ln\pn{\DAggerDatumBound} + \ln\pn{\PupilActivationLipshitzConstant^{\PupilLayerCount}}} + \\ &\qquad\qquad \pn{\PupilLayerCount + 2}\ln\pn{\PupilWeightMatrixBound_\PupilGraphScorer} + \pn{2\PupilLayerCount + 3}\sum_{\PupilLayerIndex = 1}^{\PupilLayerCount} \ln\pn{\PupilWeightMatrixBound_\PupilLayerIndex} \\
        & \leq 3\pn{\PupilLayerCount + 1}\ln\pn{10\PupilLayerCount\DAggerDatumBound\PupilActivationLipshitzConstant^{\PupilLayerCount}\PupilWeightMatrixBound_\PupilGraphScorer\sqrt{\PupilWeightMatrixLargestDim\GeneralizationSampleSize}} + \pn{2\PupilLayerCount + 3}\sum_{\PupilLayerIndex = 1}^{\PupilLayerCount} \ln\pn{\PupilWeightMatrixBound_\PupilLayerIndex}.
    \end{aligned}
$$
    Using that $\PupilWeightMatrixBound_\PupilGraphScorer = \sqrt{\LossFunctionMSENormalizationConstant}$, $\PupilWeightMatrixBound_\PupilLayerIndex \leq \max\set{1, \norm{\PupilWeightMatrix_\PupilLayerIndex}_F}$, and Theorem~\ref{theorem:erc-bounds-generalization-gap}, we attain the bound for $\gGeneralizationGap\pn{\PupilGraphScorer}$.
\end{proof}

\subsubsection{Convolutional layers as linear layers}

{
\newcommand{\zWeightBias}{{w_{\mathrm{bias}}}}
\begin{lemma}[$\ConvOneD_\PupilLayerIndex$ as a linear layer]
    \label{lemma:results:conv1d-as-linear}
    Assume that all input channels to $\ConvOneD_\PupilLayerIndex: \bR^{\gConvChannelsIn \times \gConvFeaturesIn} \to \bR^{\gConvChannelsIn[\PupilLayerIndex+1] \times \gConvFeaturesIn[\PupilLayerIndex+1]}$ are of length $\gConvFeaturesIn$, then
    $\ConvOneD_\PupilLayerIndex$ is expressible as a matrix multiplication with a fixed-dimension weight matrix.
\end{lemma}
\begin{proof}
    $\ConvOneD_\PupilLayerIndex: \bR^{\gConvChannelsIn \times \gConvFeaturesIn} \to \bR^{\gConvChannelsIn[\PupilLayerIndex+1] \times \gConvFeaturesIn[\PupilLayerIndex+1]}$ has $\gConvChannelsIn$ input channels that are row vectors of the form
$$
\begin{aligned}
    \gConvInput{} [\ConvChannelInIndex] & = \spn{
        \pn{\gConvInput}_{\ConvChannelInIndex,1}
        ,\; \dots
        ,\; \pn{\gConvInput}_{\ConvChannelInIndex,\gConvFeaturesIn}
    }.
\end{aligned}
$$
    Arranging each input channel row vector end to end and appending $\vone_{\gConvFeaturesIn[\PupilLayerIndex+1]}^\top$ to account for the bias term, an output channel $\ConvChannelOutIndex \in \set{1,\dots,\gConvChannelsIn[\PupilLayerIndex+1]}$ is the row vector
$$
    \begin{aligned}
        \gConvInput[\PupilLayerIndex+1][\ConvChannelOutIndex] = \ConvOneD_\PupilLayerIndex\pn{\gConvInput}[\ConvChannelOutIndex] & = \spn{
            \gConvInput{} [1]
            ,\; \dots
            ,\; \gConvInput{} [\gConvChannelsIn]
            ,\; \vone_{\gConvFeaturesIn[\PupilLayerIndex+1]}^\top
        }\begin{bmatrix}
            \PupilWeightMatrix\pn{1, \ConvChannelOutIndex} \\
            \vdots \\
            \PupilWeightMatrix\pn{\gConvChannelsIn, \ConvChannelOutIndex} \\
            \zWeightBias\pn{\ConvChannelOutIndex}
        \end{bmatrix}.
    \end{aligned}
$$
    If $\PupilLayerIndex < \PupilLayerCount$, then all of the output channels may be computed with one matrix multiplication by
$$
    \begin{aligned}
        & \spn{
            \gConvInput{} [1]
            ,\; \dots
            ,\; \gConvInput{} [\gConvChannelsIn]
            ,\; \vone_{\gConvFeaturesIn[\PupilLayerIndex+1]}^\top
        }\begin{bmatrix}
            \PupilWeightMatrix\pn{1, 1} & \dots & \PupilWeightMatrix\pn{1, \gConvChannelsIn[\PupilLayerIndex+1]} & \mzero \\
            \vdots & \ddots & \vdots \\
            \PupilWeightMatrix\pn{\gConvChannelsIn, 1} & \dots & \PupilWeightMatrix\pn{\gConvChannelsIn, \gConvChannelsIn[\PupilLayerIndex+1]} & \mzero \\
            \zWeightBias\pn{1}\mI & \dots & \zWeightBias\pn{\gConvChannelsIn[\PupilLayerIndex+1]}\mI & \mM
        \end{bmatrix}\\
       &\qquad = \spn{
            \ConvOneD_\PupilLayerIndex\pn{\gConvInput}[1]
            ,\; \dots
            ,\; \ConvOneD_\PupilLayerIndex\pn{\gConvInput}[\gConvChannelsIn[\PupilLayerIndex+1]]
            ,\; \vone_{\gConvFeaturesIn[\PupilLayerIndex+2]}^\top
        },
    \end{aligned}
$$
    where $\mM \in \bR^{\gConvFeaturesIn[\PupilLayerIndex+1] \times \gConvFeaturesIn[\PupilLayerIndex+2]}$ maps $\vone_{\gConvFeaturesIn[\PupilLayerIndex+1]}^\top$ to $\vone_{\gConvFeaturesIn[\PupilLayerIndex+2]}^\top$.
    If $\gConvFeaturesIn[\PupilLayerIndex + 1] = \gConvFeaturesIn[\PupilLayerIndex + 2]$, then $\mM = \mI_{\gConvFeaturesIn[\PupilLayerIndex+1] \times \gConvFeaturesIn[\PupilLayerIndex+1]}$.
    If $\gConvFeaturesIn[\PupilLayerIndex + 1] < \gConvFeaturesIn[\PupilLayerIndex + 2]$, then
$$
    \begin{aligned}
        \mM = \begin{bmatrix}
            \mI_{\pn{\gConvFeaturesIn[\PupilLayerIndex+1] - 1} \times \pn{\gConvFeaturesIn[\PupilLayerIndex+1] - 1}} &
            \vzero_{\gConvFeaturesIn[\PupilLayerIndex+1]} &
            \mzero_{\pn{\gConvFeaturesIn[\PupilLayerIndex+1] - 1} \times \pn{\gConvFeaturesIn[\PupilLayerIndex+2] - \gConvFeaturesIn[\PupilLayerIndex+1]}}
            \\
            \vzero_{\gConvFeaturesIn[\PupilLayerIndex+1] - 1}^\top &
            1 &
            \vone_{\gConvFeaturesIn[\PupilLayerIndex+2] - \gConvFeaturesIn[\PupilLayerIndex+1]}^\top
        \end{bmatrix}
    \end{aligned}.
$$
    If $\gConvFeaturesIn[\PupilLayerIndex + 1] > \gConvFeaturesIn[\PupilLayerIndex + 2]$, then
$$
    \begin{aligned}
        \mM = \begin{bmatrix}
            \mI_{\pn{\gConvFeaturesIn[\PupilLayerIndex+1] - \gConvFeaturesIn[\PupilLayerIndex+2]} \times \gConvFeaturesIn[\PupilLayerIndex+2]}
            \\
            \mzero_{\gConvFeaturesIn[\PupilLayerIndex+2] \times \gConvFeaturesIn[\PupilLayerIndex+2]}
        \end{bmatrix}
    \end{aligned}.
$$
    If $\PupilLayerIndex = \PupilLayerCount$, then all of the output channels may be computed with one matrix multiplication by
$$
    \begin{aligned}
        & \spn{
            \gConvInput{} [1]
            ,\; \dots
            ,\; \gConvInput{} [\gConvChannelsIn]
            ,\; \vone_{\gConvFeaturesIn[\PupilLayerIndex+1]}^\top
        }\begin{bmatrix}
            \PupilWeightMatrix\pn{1, 1} & \dots & \PupilWeightMatrix\pn{1, \gConvChannelsIn[\PupilLayerIndex+1]} \\
            \vdots & \ddots & \vdots \\
            \PupilWeightMatrix\pn{\gConvChannelsIn, 1} & \dots & \PupilWeightMatrix\pn{\gConvChannelsIn, \gConvChannelsIn[\PupilLayerIndex+1]} \\
            \zWeightBias\pn{1}\mI & \dots & \zWeightBias\pn{\gConvChannelsIn[\PupilLayerIndex+1]}\mI
        \end{bmatrix}\\
        &\qquad = \spn{
            \ConvOneD_\PupilLayerIndex\pn{\gConvInput}[1]
            ,\; \dots
            ,\; \ConvOneD_\PupilLayerIndex\pn{\gConvInput}[\gConvChannelsIn[\PupilLayerIndex+1]]
        }
    \end{aligned}.
$$
\end{proof}
}

\begin{lemma}[$\EqConv_\PupilLayerIndex$ as a linear layer]
    \label{lemma:results:eqconv-as-linear}
    Assume that all input channels to $\EqConv_\PupilLayerIndex: \bR^{\gConvChannelsIn \times \gConvFeaturesIn} \to \bR^{\gConvChannelsIn[\PupilLayerIndex+1] \times \gConvFeaturesIn[\PupilLayerIndex+1]}$ are of length $\gConvFeaturesIn$, then $\EqConv_\PupilLayerIndex$ is expressible as a matrix multiplication with a fixed-dimension weight matrix.
\end{lemma}
\begin{proof}
    $\EqConv_\PupilLayerIndex: \bR^{\gConvChannelsIn \times \gConvFeaturesIn} \to \bR^{\gConvChannelsIn[\PupilLayerIndex+1] \times \gConvFeaturesIn[\PupilLayerIndex+1]}$ has $\gConvChannelsIn/2$ input channels that are two-row matrices of the form
$$
    \begin{aligned}
        \gConvInput{} [\ConvChannelInIndex] & = \spn{
            \EqConvChannel_{\ConvChannelInIndex,1}
            ,\; \dots
            ,\; \EqConvChannel_{\ConvChannelInIndex,\gConvFeaturesIn}
        }
    \end{aligned}.
$$
    Arranging each input channel matrix end to end, an output channel $\ConvChannelOutIndex \in \set{1,\dots,\gConvChannelsIn[\PupilLayerIndex+1]/2}$ is 
$$
    \begin{aligned}
        \gConvInput[\PupilLayerIndex+1][\ConvChannelOutIndex] = \EqConv_\PupilLayerIndex\pn{\gConvInput}[\ConvChannelOutIndex] & = \spn{
            \gConvInput{} [1]
            ,\; \dots
            ,\; \gConvInput{} [\gConvChannelsIn/2]
        }\begin{bmatrix}
            \PupilWeightMatrix\pn{1, \ConvChannelOutIndex} \\
            \vdots \\
            \PupilWeightMatrix\pn{\gConvChannelsIn/2, \ConvChannelOutIndex} \\
        \end{bmatrix}
    \end{aligned}.
$$
    All of the output channels may be computed with one matrix multiplication by
$$
    \begin{aligned}
        & \spn{
            \gConvInput{} [1]
            ,\; \dots
            ,\; \gConvInput{} [\gConvChannelsIn/2]
        }\begin{bmatrix}
            \PupilWeightMatrix\pn{1, 1}
            & \dots
            & \PupilWeightMatrix\pn{1, \gConvChannelsIn[\PupilLayerIndex+1]/2} \\
            \vdots & \ddots & \vdots \\
            \PupilWeightMatrix\pn{\gConvChannelsIn/2, 1}
            & \dots
            & \PupilWeightMatrix\pn{\gConvChannelsIn/2, \gConvChannelsIn[\PupilLayerIndex+1]/2} \\
        \end{bmatrix}\\
        &\qquad = \spn{
            \EqConv_\PupilLayerIndex\pn{\gConvInput}[1]
            ,\; \dots
            ,\; \EqConv_\PupilLayerIndex\pn{\gConvInput}[\gConvChannelsIn[\PupilLayerIndex+1]/2]
        }.
    \end{aligned}
$$
\end{proof}

\subsubsection{Equivariant functions}

{
\newcommand{\zEquivFunc}{\vf}
\newcommand{\zEquivMatrixFunc}{\mF}
\newcommand{\zInput}{x}
\newcommand{\zvInput}{\bm{\zInput}}
\newcommand{\zmInput}{\bm{G}}
\begin{lemma}
    \label{lemma:results:equivariant-zero-at-zero}
    \OrthogonalGroup{n} equivariant functions $\zEquivMatrixFunc$ satisfy $\zEquivMatrixFunc\pn{\mzero} = \mzero$.
\end{lemma}
\begin{proof}
    By Proposition~1 in \cite{maWhySelfattentionNatural2022}, there exists a function $\PupilActivationInvariantName$ such that
    \begin{align*}
        \zEquivMatrixFunc\pn{\zmInput} = \zmInput \PupilActivationInvariantName\pn{\zmInput^\top\zmInput},
    \end{align*}
    where the output of $\PupilActivationInvariantName$ has the appropriate shape.
    Then, $\zEquivMatrixFunc\pn{\mzero} = \mzero \PupilActivationInvariantName\pn{\mzero^\top\mzero} = \mzero$.
\end{proof}

\subsubsection{Lipshitz \OrthogonalGroup{n} equivariant functions}

We derive a sufficient condition for an \OrthogonalGroup{n} equivariant function to be Lipshitz.
\begin{lemma}
\label{lemma:results:equivariant-func-lipshitz}
    Let $\zEquivFunc: \bR^n \to \bR^n$ be an \OrthogonalGroup{n} equivariant function.
    By Proposition~1 in \cite{maWhySelfattentionNatural2022}, $\zEquivFunc\pn{\zvInput} = \zvInput \PupilActivationInvariantName\pn{\norm{\zvInput}}$ where $\PupilActivationInvariantName$ is scalar-valued.
    If \PupilActivationInvariantName{} is continuously differentiable, $\lim_{\zInput \to \infty} \PupilActivationInvariantName\pn{\zInput} < \infty$, and $\lim_{\zInput \to \infty} \zInput\PupilActivationInvariantName'\pn{\zInput} < \infty$, then $\zEquivFunc\pn{\zvInput}$ is Lipshitz.
    Moreover, $\zEquivFunc\pn{\zvInput}$ has Lipshitz constant $L = \max_{\zInput \geq 0} \PupilActivationInvariantName\pn{\zInput} + \zInput\PupilActivationInvariantName'\pn{\zInput}$.
\end{lemma}
\begin{proof}
    We start by bounding the derivative of $\zEquivFunc\pn{\zvInput}$ in the spectral norm:
$$
    \begin{aligned}
        \norm*{\ddfrac{\zvInput}\zEquivFunc\pn{\zvInput}}_2 & = \norm*{\PupilActivationInvariantName\pn{\norm{\zvInput}}\mI + \zvInput\ddfrac{\zvInput}\PupilActivationInvariantName\pn{\norm{\zvInput}}}_2 \\
        & = \norm*{\PupilActivationInvariantName\pn{\norm{\zvInput}}\mI + \frac{\PupilActivationInvariantName'\pn{\norm{\zvInput}}}{\norm{\zvInput}}\zvInput\zvInput^\top}_2 \\
        & \leq \norm*{\PupilActivationInvariantName\pn{\norm{\zvInput}}\mI}_2 + \norm*{\frac{\PupilActivationInvariantName'\pn{\norm{\zvInput}}}{\norm{\zvInput}}\zvInput\zvInput^\top}_2 \\
        & \leq \PupilActivationInvariantName\pn{\norm{\zvInput}} + \frac{\PupilActivationInvariantName'\pn{\norm{\zvInput}}}{\norm{\zvInput}}\norm{\zvInput}^2 \\
        & = \PupilActivationInvariantName\pn{\norm{\zvInput}} + \norm{\zvInput}\PupilActivationInvariantName'\pn{\norm{\zvInput}}
    \end{aligned}.
$$
    By the assumptions on \PupilActivationInvariantName{},
$$
    \begin{aligned}
        \norm*{\ddfrac{\zvInput}\zEquivFunc\pn{\zvInput}}_2 & \leq \PupilActivationInvariantName\pn{\norm{\zvInput}} + \norm{\zvInput}\PupilActivationInvariantName'\pn{\norm{\zvInput}} < \infty
    \end{aligned}.
$$
    Since \PupilActivationInvariantName{} is continuously differentiable, $L = \sup_{\zInput \geq 0} \PupilActivationInvariantName\pn{\zInput} + \zInput\PupilActivationInvariantName'\pn{\zInput} < \infty$.
    Finally, $\zEquivFunc\pn{\zvInput}$ is Lipshitz with constant $L$ since, for $\va, \PupilFeatureBlockVector \in \bR^n$,
$$
    \begin{aligned}
        \norm{\zEquivFunc\pn{\va} - \zEquivFunc\pn{\PupilFeatureBlockVector}} & \leq \norm{\va - \PupilFeatureBlockVector}\sup_{\zvInput}\norm*{\frac{\mathrm{d}\zEquivFunc}{\mathrm{d}\zvInput}\pn{\zvInput}}_2 \leq L\norm{\va - \PupilFeatureBlockVector}.
    \end{aligned}
$$
\end{proof}

Next, we show that the block-wise application of a Lipshitz \OrthogonalGroup{n} equivariant function is Lipshitz.
\begin{lemma}
    \label{lemma:results:equivariant-block-matrix-func-lipshitz}
    Let $\zEquivMatrixFunc: \bR^{n \times \gConvFeaturesIn[\PupilLayerIndex + 1]} \to \bR^{n \times \gConvFeaturesIn[\PupilLayerIndex + 1]}$ be given by
$$
    \begin{aligned}
        \zEquivMatrixFunc\pn{\zmInput} = \zmInput \odot \vone_n\spn{
            \PupilActivationInvariantName\pn{\norm{\EqConvChannel_{1}}}
            ,\; \dots
            ,\; \PupilActivationInvariantName\pn{\norm{\EqConvChannel_{\gConvFeaturesIn[\PupilLayerIndex+1]}}}
        },
    \end{aligned}
$$
    where $\zEquivFunc\pn{\zvInput} = \zvInput\PupilActivationInvariantName\pn{\norm{\zvInput}}$ is \OrthogonalGroup{2} equivariant and Lipshitz with constant $L$.
    $\zEquivMatrixFunc\pn{\zmInput}$ is Lipshitz with constant $L$:
    $$
    \begin{aligned}
        \norm*{\zEquivMatrixFunc\pn{\mA} - \zEquivMatrixFunc\pn{\mB}}_F \leq L\norm*{\mA - \mB}_F.
    \end{aligned}
    $$
\end{lemma}
\begin{proof}
    This is proven using the definition of the Frobenius norm:
$$
    \begin{aligned}
        \norm*{\zEquivMatrixFunc\pn{\mA} - \zEquivMatrixFunc\pn{\mB}}_F^2 & = \sum_{\ConvFeatureIndex = 1}^{\gConvFeaturesIn[\PupilLayerIndex+1]} \norm*{\zEquivFunc\pn{\va_\ConvFeatureIndex} - \zEquivFunc\pn{\vb_\ConvFeatureIndex}}^2 \leq L^2 \sum_{\ConvFeatureIndex = 1}^{\gConvFeaturesIn[\PupilLayerIndex+1]} \norm*{\va_\ConvFeatureIndex - \vb_\ConvFeatureIndex}^2 \leq L^2\norm*{\mA - \mB}_F^2,
    \end{aligned}
$$
    Taking the square root of both sizes gives the result.
\end{proof}

\begin{lemma}
    \label{lemma:results:equivariant-block-matrix-func-lipshitz-multiple-output-channels}
    Let $\zEquivMatrixFunc: \bR^{\pn{n\gConvChannelsIn[\PupilLayerIndex+1]/n} \times \gConvFeaturesIn[\PupilLayerIndex+1]}$ such that $\gConvChannelsIn[\PupilLayerIndex+1]/n \in \bN$.
    For $\ConvChannelOutIndex \in \set{1,\dots,\gConvChannelsIn[\PupilLayerIndex+1]/n}$ and $\zmInput$ where
    $
    \begin{aligned}
        \zmInput[\ConvChannelOutIndex] = \spn{
            \EqConvChannel_{\ConvChannelOutIndex,1}
            ,\; \dots
            ,\; \EqConvChannel_{\ConvChannelOutIndex,\gConvFeaturesIn[\PupilLayerIndex+1]}
        } \in \bR^{n \times \gConvFeaturesIn[\PupilLayerIndex+1]}
    \end{aligned}
    $, suppose
$$
    \begin{aligned}
        \zEquivMatrixFunc\pn{\zmInput}[\ConvChannelOutIndex] = \zmInput[\ConvChannelOutIndex] \odot \vone_n\spn{
            \PupilActivationInvariantName\pn{\norm{\EqConvChannel_{\ConvChannelOutIndex,1}}}
            ,\; \dots
            ,\; \PupilActivationInvariantName\pn{\norm{\EqConvChannel_{\ConvChannelOutIndex,\gConvFeaturesIn[\PupilLayerIndex+1]}}}
        }
    \end{aligned}.
$$
    If $\zEquivFunc: \bR^n \to \bR^n$ is as defined in \Lemmaref{lemma:results:equivariant-block-matrix-func-lipshitz} with Lipshitz constant $L$, then
    $$
    \begin{aligned}
        \norm*{\zEquivMatrixFunc\pn{\mA} - \zEquivMatrixFunc\pn{\mB}}_F \leq L\norm*{\mA - \mB}_F.
    \end{aligned}
    $$
\end{lemma}
\begin{proof}
    By \Lemmaref{lemma:results:equivariant-block-matrix-func-lipshitz} and the definition of the Frobenius norm,
    $$
    \begin{aligned}
        \norm{\zEquivMatrixFunc\pn{\mA} - \zEquivMatrixFunc\pn{\mB}}_F^2 & = \sum_{\ConvChannelOutIndex = 1}^{\gConvChannelsIn[\PupilLayerIndex+1]/n} \norm{\zEquivMatrixFunc\pn{\mA}[\ConvChannelOutIndex] - \zEquivMatrixFunc\pn{\mB}[\ConvChannelOutIndex]}_F^2\\
        &\ \leq L^2\sum_{\ConvChannelOutIndex = 1}^{\gConvChannelsIn[\PupilLayerIndex+1]/n} \norm{\mA[\ConvChannelOutIndex] - \mB[\ConvChannelOutIndex]}_F^2 = L^2\norm{\mA - \mB}_F^2.
    \end{aligned}
    $$
    Taking the square root of both sizes gives the result.
\end{proof}

Using Lemmas~\ref{lemma:results:equivariant-func-lipshitz}, \ref{lemma:results:equivariant-block-matrix-func-lipshitz}, and \ref{lemma:results:equivariant-block-matrix-func-lipshitz-multiple-output-channels}, we show that the activations \PupilETDAGNNActivationScaleLog{} and \PupilETDAGNNActivationScaleTanh{} in \PupilEquivariantName{} are Lipshitz.
\begin{lemma}
    \label{lemma:results:etdagnn-sigma0-lipshitz}
    \OrthogonalGroup{n} equivariant activation \PupilETDAGNNActivationScaleLog{} is Lipshitz.
\end{lemma}
\begin{proof}
    Define
    $$
        \PupilETDAGNNActivationScaleLogInvariant\pn{\zInput} = \begin{cases}
            1 & \zInput = 0 \\
            \frac{\ln\pn{1 + \zInput}}{\zInput} & \zInput \neq 0
        \end{cases}
    $$
    with derivative
    $$
        \PupilETDAGNNActivationScaleLogInvariant'\pn{\zInput} = \begin{cases}
            0 & \zInput = 0 \\
            \frac{1}{\zInput\pn{1 + \zInput}} - \frac{\ln\pn{1 + \zInput}}{\zInput^2} & \zInput \neq 0
        \end{cases}.
    $$
    In order to apply \Lemmaref{lemma:results:equivariant-func-lipshitz}, we compute the following limits:
    \begin{align*}
        \lim_{\zInput \to \infty} \PupilETDAGNNActivationScaleLogInvariant\pn{\zInput} & = \lim_{\zInput \to \infty} \frac{\ln\pn{1 + \zInput}}{\zInput} \substack{\text{L'H\^opital}\\=} \lim_{\zInput \to \infty} \frac{1}{1 + \zInput} = 0 < \infty, \\
        \lim_{\zInput \to \infty} \zInput\PupilETDAGNNActivationScaleLogInvariant'\pn{\zInput} & = \lim_{\zInput \to \infty} \frac{1}{1 + \zInput} - \frac{\ln\pn{1 + \zInput}}{\zInput}\\
        &= 0 - \lim_{\zInput \to \infty} \frac{\ln\pn{1 + \zInput}}{\zInput} \substack{\text{L'H\^opital}\\=} 0 - \lim_{\zInput \to \infty} \frac{1}{1 + \zInput} = 0 - 0 = 0.
    \end{align*}
    By \Lemmaref{lemma:results:equivariant-func-lipshitz}, $\zEquivFunc\pn{\zvInput} = \zvInput\PupilETDAGNNActivationScaleLogInvariant\pn{\norm{\zvInput}}$ is Lipshitz with constant $L = \max_{\zInput \geq 0} \PupilETDAGNNActivationScaleLogInvariant\pn{\zInput} + \zInput\PupilETDAGNNActivationScaleLogInvariant\pn{\zInput}$.
    Applying Lemmas~\ref{lemma:results:equivariant-block-matrix-func-lipshitz} and \ref{lemma:results:equivariant-block-matrix-func-lipshitz-multiple-output-channels}, $\PupilETDAGNNActivationScaleLog\pn{\zmInput}$ is Lipshitz with constant $L$.
\end{proof}

\begin{lemma}
    \label{lemma:results:etdagnn-sigma1-lipshitz}
    \OrthogonalGroup{n} equivariant activation \PupilETDAGNNActivationScaleTanh{} is Lipshitz.
\end{lemma}
\begin{proof}
    Define $\PupilETDAGNNActivationScaleTanhInvariant\pn{\zInput} = \tanh\pn{\zInput}$, and then $\PupilETDAGNNActivationScaleTanhInvariant'\pn{\zInput} = 1 - \tanh^2\pn{\zInput}$.
    In order to apply \Lemmaref{lemma:results:equivariant-func-lipshitz}, we compute the following limits:
    $$
    \begin{aligned}
        \lim_{\zInput \to \infty} \PupilETDAGNNActivationScaleTanhInvariant\pn{\zInput} & = \lim_{\zInput \to \infty} \tanh\pn{\zInput} = 1 < \infty, \\
        \lim_{\zInput \to \infty} \zInput\PupilETDAGNNActivationScaleTanhInvariant'\pn{\zInput} & = \lim_{\zInput \to \infty} \zInput\pn{1 - \tanh^2\pn{\zInput}} \substack{\text{ L'H\^opital}\\=} \lim_{\zInput \to \infty} \frac{1}{\frac{2\tanh\pn{\zInput}}{1 - \tanh^2\pn{\zInput}}} = \lim_{\zInput \to \infty} \frac{1 - \tanh^2\pn{\zInput}}{2\tanh\pn{\zInput}} = 0.
    \end{aligned}
    $$
    By \Lemmaref{lemma:results:equivariant-func-lipshitz}, $\zEquivFunc\pn{\zvInput} = \zvInput\PupilETDAGNNActivationScaleTanhInvariant\pn{\norm{\zvInput}}$ is Lipshitz with constant $L = \max_{\zInput \geq 0} \PupilETDAGNNActivationScaleTanhInvariant\pn{\zInput} + \zInput\PupilETDAGNNActivationScaleTanhInvariant\pn{\zInput}$.
    Applying Lemmas~\ref{lemma:results:equivariant-block-matrix-func-lipshitz} and \ref{lemma:results:equivariant-block-matrix-func-lipshitz-multiple-output-channels}, $\PupilETDAGNNActivationScaleTanh\pn{\zmInput}$ is Lipshitz with constant $L$.
\end{proof}
}

\subsection{Obstacle avoidance}

{
\newcommand{\zRotationMatrix}{{\UtilRotationMatrix}}
\newcommand{\zAngle}{{\theta}}
\newcommand{\zLinearDiscriminantName}{{\gamma}}
\newcommand{\zFlockAgentVelocity}{\vv}
\newcommand{\zVelocityNeighborAverage}{{\overline{\zFlockAgentVelocity}}}
\newcommand{\zRelativeVelocityMagnitude}{{\alpha_1}}
\newcommand{\zRelativeVelocityGravityBreak}{{\alpha_2}}
\newcommand{\zRelativeVelocityObstacleDodgeAngle}{{\alpha_\zAngle}}

\begin{lemma}
    $\zLinearDiscriminantName\pn{\FlockAgentRelativePos[], \zFlockAgentVelocity, \zAngle}$ from \eqref{eq:results:obstacle-avoidance-linear-discriminant} is invariant with respect to \SpecialOrthogonalGroup{2}.
\end{lemma}
\begin{proof}
    Let $\zRotationMatrix \in $ \SpecialOrthogonalGroup{2}.
    $$
    \begin{aligned}
        \zLinearDiscriminantName\pn{\zRotationMatrix\FlockAgentRelativePos[], \zRotationMatrix\zFlockAgentVelocity, \zAngle} & = \frac{\pn{\zRotationMatrix\FlockAgentRelativePos[]}^\top}{\norm{\zRotationMatrix\FlockAgentRelativePos[]}}\zRotationMatrix\pn{\zAngle}\frac{\zRotationMatrix\zFlockAgentVelocity}{\norm{\zRotationMatrix\zFlockAgentVelocity}} = \frac{\FlockAgentRelativePos[]^\top\zRotationMatrix^\top\zRotationMatrix\pn{\zAngle}\zRotationMatrix\zFlockAgentVelocity}{\norm{\FlockAgentRelativePos[]}\norm{\zFlockAgentVelocity}} \\
        & = \frac{\FlockAgentRelativePos[]^\top\zRotationMatrix^\top\zRotationMatrix\zRotationMatrix\pn{\zAngle}\zFlockAgentVelocity}{\norm{\FlockAgentRelativePos[]}\norm{\zFlockAgentVelocity}} \quad \text{since matrices of \SpecialOrthogonalGroup{2} commute} \\
        & = \frac{\FlockAgentRelativePos[]^\top\mI\zRotationMatrix\pn{\zAngle}\zFlockAgentVelocity}{\norm{\FlockAgentRelativePos[]}\norm{\zFlockAgentVelocity}} = \zLinearDiscriminantName\pn{\FlockAgentRelativePos[], \zFlockAgentVelocity, \zAngle}.
    \end{aligned}
    $$
    Therefore, $\zLinearDiscriminantName$ is invariant with respect to \SpecialOrthogonalGroup{2}.
\end{proof}

\begin{proposition}
    The relative velocity in \eqref{eq:results:obstacle-avoidance:relative-velocity} is equivariant with respect to \SpecialOrthogonalGroup{2}.
\end{proposition}
\begin{proof}
    Let $\zRotationMatrix \in$ \SpecialOrthogonalGroup{2}.
    First, $\zVelocityNeighborAverage_\FlockAgentIndex$ is equivariant with respect to \SpecialOrthogonalGroup{2} since
    \begin{align*}
        \mean\set{\zRotationMatrix\FlockAgentVel_j\pn{t} : j \in \FlockAgentNeighborhood_\FlockAgentIndex\pn{t}} = \zRotationMatrix\mean\set{\FlockAgentVel_j\pn{t} : j \in \FlockAgentNeighborhood_\FlockAgentIndex\pn{t}} = \zRotationMatrix\zVelocityNeighborAverage_\FlockAgentIndex\pn{t}.
    \end{align*}
    We check each case of the relative velocity formula.
    When $-\zRelativeVelocityGravityBreak \leq \zLinearDiscriminantName\pn{\FlockAgentRelativePos[_{\FlockAgentIndexTwo\FlockAgentIndex}], \zVelocityNeighborAverage_\FlockAgentIndex, 0} \leq 0$,
    $$
    \begin{aligned}
        \zRelativeVelocityMagnitude\pn{\norm{\zRotationMatrix\FlockAgentRelativePos[_{\FlockAgentIndexTwo\FlockAgentIndex}]}, \norm{\zRotationMatrix\zVelocityNeighborAverage_\FlockAgentIndex}}\frac{\zRotationMatrix\FlockAgentRelativePos[_{\FlockAgentIndexTwo\FlockAgentIndex}]}{\norm{\zRotationMatrix\FlockAgentRelativePos[_{\FlockAgentIndexTwo\FlockAgentIndex}]}} & = \zRelativeVelocityMagnitude\pn{\norm{\FlockAgentRelativePos[_{\FlockAgentIndexTwo\FlockAgentIndex}]}, \norm{\zVelocityNeighborAverage_\FlockAgentIndex}}\frac{\zRotationMatrix\FlockAgentRelativePos[_{\FlockAgentIndexTwo\FlockAgentIndex}]}{\norm{\FlockAgentRelativePos[_{\FlockAgentIndexTwo\FlockAgentIndex}]}} \\
        & = \zRotationMatrix\zRelativeVelocityMagnitude\pn{\norm{\FlockAgentRelativePos[_{\FlockAgentIndexTwo\FlockAgentIndex}]}, \norm{\zVelocityNeighborAverage_\FlockAgentIndex}}\frac{\FlockAgentRelativePos[_{\FlockAgentIndexTwo\FlockAgentIndex}]}{\norm{\FlockAgentRelativePos[_{\FlockAgentIndexTwo\FlockAgentIndex}]}} = \zRotationMatrix\pn{-\FlockAgentRelativeVel\pn{t}}.
    \end{aligned}
    $$
    For the other case,
    $$
    \begin{aligned}
        & \zRelativeVelocityMagnitude\pn{\norm{\zRotationMatrix\FlockAgentRelativePos[_{\FlockAgentIndexTwo\FlockAgentIndex}]},\norm{\zRotationMatrix\zVelocityNeighborAverage_\FlockAgentIndex}}\pn{-\mathrm{sgn}\spn{\zLinearDiscriminantName\pn{\zRotationMatrix\FlockAgentRelativePos[_{\FlockAgentIndexTwo\FlockAgentIndex}], \zRotationMatrix\zVelocityNeighborAverage_\FlockAgentIndex, \frac{\pi}{2}}}\zRotationMatrix\pn{\zRelativeVelocityObstacleDodgeAngle}}\frac{\zRotationMatrix\FlockAgentRelativePos[_{\FlockAgentIndexTwo\FlockAgentIndex}]}{\norm{\zRotationMatrix\FlockAgentRelativePos[_{\FlockAgentIndexTwo\FlockAgentIndex}]}} \\
        = & \zRelativeVelocityMagnitude\pn{\norm{\FlockAgentRelativePos[_{\FlockAgentIndexTwo\FlockAgentIndex}]},\norm{\zVelocityNeighborAverage_\FlockAgentIndex}}\pn{-\mathrm{sgn}\spn{\zLinearDiscriminantName\pn{\FlockAgentRelativePos[_{\FlockAgentIndexTwo\FlockAgentIndex}], \zVelocityNeighborAverage_\FlockAgentIndex, \frac{\pi}{2}}}\zRotationMatrix\pn{\zRelativeVelocityObstacleDodgeAngle}}\frac{\zRotationMatrix\FlockAgentRelativePos[_{\FlockAgentIndexTwo\FlockAgentIndex}]}{\norm{\FlockAgentRelativePos[_{\FlockAgentIndexTwo\FlockAgentIndex}]}} \\
        = & \zRotationMatrix\zRelativeVelocityMagnitude\pn{\norm{\FlockAgentRelativePos[_{\FlockAgentIndexTwo\FlockAgentIndex}]},\norm{\zVelocityNeighborAverage_\FlockAgentIndex}}\pn{-\mathrm{sgn}\spn{\zLinearDiscriminantName\pn{\FlockAgentRelativePos[_{\FlockAgentIndexTwo\FlockAgentIndex}], \zVelocityNeighborAverage_\FlockAgentIndex, \frac{\pi}{2}}}\zRotationMatrix\pn{\zRelativeVelocityObstacleDodgeAngle}}\frac{\FlockAgentRelativePos[_{\FlockAgentIndexTwo\FlockAgentIndex}]}{\norm{\FlockAgentRelativePos[_{\FlockAgentIndexTwo\FlockAgentIndex}]}} \quad \text{since matrices of \SpecialOrthogonalGroup{2} commute} \\
        = & \zRotationMatrix\pn{-\FlockAgentRelativeVel\pn{t}}.
    \end{aligned}
    $$
    Therefore, the relative velocity is \SpecialOrthogonalGroup{2} equivariant.
\end{proof}
}

\end{appendix}

\bibliographystyle{spmpsci}      
\bibliography{references}   

\begin{thebibliography}{10}
\providecommand{\url}[1]{{#1}}
\providecommand{\urlprefix}{URL }
\expandafter\ifx\csname urlstyle\endcsname\relax
  \providecommand{\doi}[1]{DOI~\discretionary{}{}{}#1}\else
  \providecommand{\doi}{DOI~\discretionary{}{}{}\begingroup \urlstyle{rm}\Url}\fi

\bibitem{akyildiz2002survey}
Akyildiz, I.F., Su, W., Sankarasubramaniam, Y., Cayirci, E.: A survey on sensor networks.
\newblock IEEE Communications magazine \textbf{40}(8), 102--114 (2002).
\newblock \doi{https://doi.org/10.1109/MCOM.2002.1024422}

\bibitem{baker2024an}
Baker, J., Wang, S.H., de~Fernex, T., Wang, B.: An explicit frame construction for normalizing 3d point clouds.
\newblock In: Forty-first International Conference on Machine Learning (2024).
\newblock \urlprefix\url{https://openreview.net/forum?id=SZ0JnRxi0x}

\bibitem{bartlettSpectrallynormalizedMarginBounds2017}
Bartlett, P.L., Foster, D.J., Telgarsky, M.J.: Spectrally-normalized margin bounds for neural networks.
\newblock In: I.~Guyon, U.V. Luxburg, S.~Bengio, H.~Wallach, R.~Fergus, S.~Vishwanathan, R.~Garnett (eds.) Advances in Neural Information Processing Systems, vol.~30. Curran Associates, Inc. (2017).
\newblock \urlprefix\url{https://proceedings.neurips.cc/paper_files/paper/2017/file/b22b257ad0519d4500539da3c8bcf4dd-Paper.pdf}

\bibitem{brandstetter2022geometric}
Brandstetter, J., Hesselink, R., van~der Pol, E., Bekkers, E.J., Welling, M.: Geometric and physical quantities improve {E(3)} equivariant message passing.
\newblock In: International Conference on Learning Representations (2022).
\newblock \urlprefix\url{https://openreview.net/forum?id=_xwr8gOBeV1}

\bibitem{camazineSelforganizationBiologicalSystems2003}
Camazine, S. (ed.): Self-Organization in Biological Systems, 2. print., and 1. paperback print edn.
\newblock Princeton Studies in Complexity. Princeton Univ. Press.
\newblock \doi{https://doi.org/10.2307/j.ctvzxx9tx}

\bibitem{chenGeneralizationBoundsFamily2019}
Chen, M., Li, X., Zhao, T.: On generalization bounds of a family of recurrent neural networks (2019).
\newblock \urlprefix\url{https://openreview.net/forum?id=Skf-oo0qt7}

\bibitem{choiEmergentDynamicsCuckerSmale2016}
Choi, Y.P., Ha, S.Y., Li, Z.: Emergent Dynamics of the Cucker--Smale Flocking Model and Its Variants, pp. 299--331.
\newblock Springer International Publishing, Cham (2017).
\newblock \doi{https://doi.org/10.1007/978-3-319-49996-3_8}

\bibitem{choiCollisionlessSingularCuckerSmale2019}
Choi, Y.P., Kalise, D., Peszek, J., Peters, A.A.: A collisionless singular cucker--smale model with decentralized formation control.
\newblock SIAM Journal on Applied Dynamical Systems \textbf{18}(4), 1954--1981 (2019).
\newblock \doi{https://doi.org/10.1137/19M1241799}

\bibitem{choiControlledPatternFormation2022}
Choi, Y.P., Oh, D., Tse, O.: Controlled pattern formation of stochastic cucker–smale systems with network structures.
\newblock Communications in Nonlinear Science and Numerical Simulation \textbf{111}, 106474 (2022).
\newblock \doi{https://doi.org/10.1016/j.cnsns.2022.106474}

\bibitem{cohen2017steerable}
Cohen, T.S., Welling, M.: Steerable {CNN}s.
\newblock In: International Conference on Learning Representations (2017).
\newblock \urlprefix\url{https://openreview.net/forum?id=rJQKYt5ll}

\bibitem{cuckerAvoidingCollisionsFlocks2010}
Cucker, F., Dong, J.G.: Avoiding {{Collisions}} in {{Flocks}} \textbf{55}(5), 1238--1243.
\newblock \doi{https://doi.org/10.1109/TAC.2010.2042355}.
\newblock \urlprefix\url{http://ieeexplore.ieee.org/document/5406110/}

\bibitem{cuckerEmergentBehaviorFlocks2007}
Cucker, F., Smale, S.: Emergent {{Behavior}} in {{Flocks}} \textbf{52}(5), 852--862.
\newblock \doi{https://doi.org/10.1109/TAC.2007.895842}

\bibitem{fuchs2020se}
Fuchs, F., Worrall, D., Fischer, V., Welling, M.: Se (3)-transformers: 3d roto-translation equivariant attention networks.
\newblock Advances in neural information processing systems \textbf{33}, 1970--1981 (2020).
\newblock \urlprefix\url{https://papers.neurips.cc/paper/2020/file/15231a7ce4ba789d13b722cc5c955834-Paper.pdf}

\bibitem{garg2020generalization}
Garg, V., Jegelka, S., Jaakkola, T.: Generalization and representational limits of graph neural networks.
\newblock In: International Conference on Machine Learning, pp. 3419--3430. PMLR (2020).
\newblock \urlprefix\url{https://proceedings.mlr.press/v119/garg20c/garg20c.pdf}

\bibitem{guLeaderFollowerFlocking2009}
Gu, D., Wang, Z.: Leader–{{Follower Flocking}}: {{Algorithms}} and {{Experiments}} \textbf{17}(5), 1211--1219.
\newblock \doi{https://doi.org/10.1109/TCST.2008.2009461}.
\newblock \urlprefix\url{http://ieeexplore.ieee.org/document/4895844/}

\bibitem{doi:10.1137/21M1465081}
Hua, Y., Miller, K., Bertozzi, A.L., Qian, C., Wang, B.: Efficient and reliable overlay networks for decentralized federated learning.
\newblock SIAM Journal on Applied Mathematics \textbf{82}(4), 1558--1586 (2022).
\newblock \doi{https://doi.org/10.1137/21M1465081}

\bibitem{hua2024towards}
Hua, Y., Pang, J., Zhang, X., Liu, Y., Shi, X., Wang, B., Liu, Y., Qian, C.: Towards practical overlay networks for decentralized federated learning.
\newblock arXiv preprint arXiv:2409.05331  (2024).
\newblock \doi{https://doi.org/10.48550/arXiv.2409.05331}

\bibitem{pmlr-v235-karczewski24a}
Karczewski, R., Souza, A.H., Garg, V.: On the generalization of equivariant graph neural networks.
\newblock In: R.~Salakhutdinov, Z.~Kolter, K.~Heller, A.~Weller, N.~Oliver, J.~Scarlett, F.~Berkenkamp (eds.) Proceedings of the 41st International Conference on Machine Learning, \emph{Proceedings of Machine Learning Research}, vol. 235, pp. 23159--23186. PMLR.
\newblock \urlprefix\url{https://proceedings.mlr.press/v235/karczewski24a.html}

\bibitem{pmlr-v80-kondor18a}
Kondor, R., Trivedi, S.: On the generalization of equivariance and convolution in neural networks to the action of compact groups.
\newblock In: J.~Dy, A.~Krause (eds.) Proceedings of the 35th International Conference on Machine Learning, \emph{Proceedings of Machine Learning Research}, vol.~80, pp. 2747--2755. PMLR (2018).
\newblock \urlprefix\url{https://proceedings.mlr.press/v80/kondor18a.html}

\bibitem{pmlr-v162-lawrence22a}
Lawrence, H., Georgiev, K., Dienes, A., Kiani, B.T.: Implicit bias of linear equivariant networks.
\newblock In: K.~Chaudhuri, S.~Jegelka, L.~Song, C.~Szepesvari, G.~Niu, S.~Sabato (eds.) Proceedings of the 39th International Conference on Machine Learning, \emph{Proceedings of Machine Learning Research}, vol. 162, pp. 12096--12125. PMLR (2022).
\newblock \urlprefix\url{https://proceedings.mlr.press/v162/lawrence22a/lawrence22a.pdf}

\bibitem{maWhySelfattentionNatural2022}
Ma, C., Ying, L.: Why self attention is natural for sequence-to-sequence problems? a perspective from symmetries (2023).
\newblock \urlprefix\url{https://openreview.net/forum?id=dNdOnKy9YNs}

\bibitem{minakowskiSingularCuckerSmaleDynamics2019}
Minakowski, P., Mucha, P.B., Peszek, J., Zatorska, E.: Singular Cucker--Smale Dynamics, pp. 201--243.
\newblock Springer International Publishing, Cham (2019).
\newblock \doi{https://doi.org/10.1007/978-3-030-20297-2_7}

\bibitem{mohriFoundationsMachineLearning2012}
Mohri, M., Rostamizadeh, A., Talwalkar, A.: Foundations of Machine Learning.
\newblock Adaptive Computation and Machine Learning. The MIT Press.
\newblock \urlprefix\url{https://mitpress.mit.edu/9780262039406/foundations-of-machine-learning/}

\bibitem{ohFlockingBehaviorStochastic2021}
Oh, D.: Flocking {{Behavior}} in {{Stochastic Cucker-Smale Model}} with {{Formation Control}} on {{Symmetric Digraphs}}.
\newblock \urlprefix\url{http://www.riss.kr/link?id=T15771814}

\bibitem{omotuyiLearningDecentralizedControllers2022}
Omotuyi, O., Kumar, M.: Learning {{Decentralized Controllers}} for {{Segregation}} of {{Heterogeneous Robot Swarms}} with {{Graph Neural Networks}}.
\newblock In: 2022 {{International Conference}} on {{Manipulation}}, {{Automation}} and {{Robotics}} at {{Small Scales}} ({{MARSS}}), pp. 1--6. IEEE.
\newblock \doi{https://doi.org/10.1109/MARSS55884.2022.9870482}

\bibitem{parkCuckerSmaleFlockingInterParticle2010}
Park, J., Kim, H.J., Ha, S.Y.: Cucker-{{Smale Flocking With Inter-Particle Bonding Forces}} \textbf{55}(11), 2617--2623.
\newblock \doi{https://doi.org/10.1109/TAC.2010.2061070}

\bibitem{parrishSelfOrganizedFishSchools2002}
Parrish, J.K., Viscido, S.V., Grünbaum, D.: Self-{{Organized Fish Schools}}: {{An Examination}} of {{Emergent Properties}} \textbf{202}(3), 296--305.
\newblock \doi{https://doi.org/10.2307/1543482}.
\newblock \urlprefix\url{https://www.journals.uchicago.edu/doi/10.2307/1543482}

\bibitem{perea2009extension}
Perea, L., G{\'o}mez, G., Elosegui, P.: Extension of the cucker-smale control law to space flight formations.
\newblock Journal of guidance, control, and dynamics \textbf{32}(2), 527--537 (2009).
\newblock \doi{https://doi.org/10.2514/1.36269}

\bibitem{van2020mdp}
van~der Pol, E., Worrall, D., van Hoof, H., Oliehoek, F., Welling, M.: {MDP} homomorphic networks: Group symmetries in reinforcement learning.
\newblock Advances in Neural Information Processing Systems \textbf{33}, 4199--4210 (2020).
\newblock \urlprefix\url{https://proceedings.neurips.cc/paper/2020/file/2be5f9c2e3620eb73c2972d7552b6cb5-Paper.pdf}

\bibitem{reynoldsFlocksHerdsSchools1987}
Reynolds, C.W.: Flocks, herds and schools: {{A}} distributed behavioral model \textbf{21}(4), 25--34.
\newblock \doi{https://doi.org/10.1145/37402.37406}.
\newblock \urlprefix\url{https://dl.acm.org/doi/10.1145/37402.37406}

\bibitem{rossReductionImitationLearning2011}
Ross, S., Gordon, G., Bagnell, D.: A reduction of imitation learning and structured prediction to no-regret online learning.
\newblock In: G.~Gordon, D.~Dunson, M.~Dudík (eds.) Proceedings of the Fourteenth International Conference on Artificial Intelligence and Statistics, \emph{Proceedings of Machine Learning Research}, vol.~15, pp. 627--635. PMLR, Fort Lauderdale, FL, USA (2011).
\newblock \urlprefix\url{https://proceedings.mlr.press/v15/ross11a.html}

\bibitem{santosSegregationMultipleHeterogeneous2014}
Santos, V.G., Pimenta, L.C.A., Chaimowicz, L.: Segregation of multiple heterogeneous units in a robotic swarm.
\newblock In: 2014 {{IEEE International Conference}} on {{Robotics}} and {{Automation}} ({{ICRA}}), pp. 1112--1117. IEEE.
\newblock \doi{https://doi.org/10.1109/ICRA.2014.6906993}

\bibitem{satorras2021n}
Satorras, V.G., Hoogeboom, E., Welling, M.: {E(n)} equivariant graph neural networks.
\newblock In: International conference on machine learning, pp. 9323--9332. PMLR (2021).
\newblock \urlprefix\url{https://proceedings.mlr.press/v139/satorras21a/satorras21a.pdf}

\bibitem{9850408}
Sun, T., Li, D., Wang, B.: Decentralized federated averaging.
\newblock IEEE Transactions on Pattern Analysis and Machine Intelligence \textbf{45}(4), 4289--4301 (2023).
\newblock \doi{https://doi.org/10.1109/TPAMI.2022.3196503}

\bibitem{tannerStableFlockingMobile2003a}
Tanner, H., Jadbabaie, A., Pappas, G.: Stable flocking of mobile agents. {{I}}. {{Fixed}} topology.
\newblock In: 42nd {{IEEE International Conference}} on {{Decision}} and {{Control}} ({{IEEE Cat}}. {{No}}.{{03CH37475}}), vol.~2, pp. 2010--2015. IEEE.
\newblock \doi{https://doi.org/10.1109/CDC.2003.1272910}

\bibitem{tannerStableFlockingMobile2003}
Tanner, H., Jadbabaie, A., Pappas, G.: Stable flocking of mobile agents. {{II}}. {{Dynamic}} topology.
\newblock In: 42nd {{IEEE International Conference}} on {{Decision}} and {{Control}} ({{IEEE Cat}}. {{No}}.{{03CH37475}}), vol.~2, pp. 2016--2021. IEEE.
\newblock \doi{https://doi.org/10.1109/CDC.2003.1272911}

\bibitem{tolstayaLearningDecentralizedControllers2017b}
Tolstaya, E., Gama, F., Paulos, J., Pappas, G., Kumar, V., Ribeiro, A.: Learning decentralized controllers for robot swarms with graph neural networks.
\newblock In: Conference on robot learning, pp. 671--682. PMLR (2020).
\newblock \urlprefix\url{https://proceedings.mlr.press/v100/tolstaya20a.html}

\bibitem{wang2024rethinking}
Wang, S.H., Hsu, Y.C., Baker, J., Bertozzi, A.L., Xin, J., Wang, B.: Rethinking the benefits of steerable features in 3d equivariant graph neural networks.
\newblock In: The Twelfth International Conference on Learning Representations (2024).
\newblock \urlprefix\url{https://openreview.net/forum?id=mGHJAyR8w0}

\bibitem{weiler2019general}
Weiler, M., Cesa, G.: General e(2)-equivariant steerable cnns.
\newblock Advances in neural information processing systems \textbf{32} (2019).
\newblock \urlprefix\url{https://proceedings.neurips.cc/paper_files/paper/2019/file/45d6637b718d0f24a237069fe41b0db4-Paper.pdf}

\bibitem{weimerskirchEnergySavingFlight2001}
Weimerskirch, H., Martin, J., Clerquin, Y., Alexandre, P., Jiraskova, S.: Energy saving in flight formation \textbf{413}(6857), 697--698.
\newblock \doi{https://doi.org/10.1038/35099670}.
\newblock \urlprefix\url{https://www.nature.com/articles/35099670}

\bibitem{yarotsky2022universal}
Yarotsky, D.: Universal approximations of invariant maps by neural networks.
\newblock Constructive Approximation \textbf{55}(1), 407--474 (2022).
\newblock \doi{https://doi.org/10.1007/s00365-021-09546-1}

\end{thebibliography}

\end{document}